\documentclass[journal]{IEEEtran}

\IEEEoverridecommandlockouts                              

\usepackage{moreverb,url}
\usepackage[colorlinks,bookmarksopen,bookmarksnumbered,citecolor=red,urlcolor=red]{hyperref}
\usepackage{subfigure}
\usepackage{graphics} 
\usepackage{epsfig} 
\usepackage{mathptmx} 
\usepackage{times} 
\usepackage{amsmath} 
\usepackage{amssymb}  
\usepackage{amsthm}
\usepackage{bm}
\usepackage{mathrsfs}
\usepackage{xcolor}
\usepackage{threeparttable}
\usepackage{multirow}
\usepackage{bigdelim}
\usepackage{algorithm}
\usepackage{algorithmicx}
\usepackage{algpseudocode}
\usepackage{graphicx}
\usepackage{comment}
\usepackage{galois}

\newtheorem{definition}{Definition}
\newtheorem{lemma}[definition]{Lemma}
\newtheorem{remark}[definition]{Remark}
\newtheorem{corollary}[definition]{Corollary}

\newtheorem{problem}[definition]{Problem}

\newtheorem{proposition}[definition]{Proposition}
\newtheorem{theorem}[definition]{Theorem}

\floatname{algorithm}{Algorithm}

\title{
Efficient Search of the $k$ Shortest Non-homotopic Paths by Eliminating Non-$k$-Optimal Topologies
}

\author{Tong Yang, Li Huang, Yue Wang$^*$ and Rong Xiong
\thanks{Tong Yang, Yue Wang and Rong Xiong are with the State Key Laboratory of Industrial Control and Technology, Zhejiang University, P.R. China. Yue Wang is the corresponding author {\tt\small wangyue@iipc.zju.edu.cn}.}
\thanks{Li Huang is with the Institute of Advanced Digital Technologies and Instrumentation, Zhejiang University, P.R. China. }
}

\begin{document}

\maketitle
\thispagestyle{empty}
\pagestyle{empty}

\begin{abstract}
An efficient algorithm to solve the $k$ shortest non-homotopic path planning ($k$-SNPP) problem in a 2D environment is proposed in this paper. 
Motivated by accelerating the inefficient exploration of the homotopy-augmented space of the 2D environment, our fundamental idea is to identify the non-$k$-optimal path topologies as early as possible and terminate the pathfinding along them. 
This is a non-trivial practice because it has to be done at an intermediate state of the path planning process when locally shortest paths have not been fully constructed. 
In other words, the paths to be compared have not rendezvoused at the goal location, which makes the homotopy theory, modelling the spatial relationship among the paths having the same endpoint, not applicable. 

This paper is the first work that develops a systematic distance-based topology simplification mechanism to solve the $k$-SNPP task, whose core contribution is to assert the distance-based order of non-homotopic locally shortest paths before constructing them. 
If the order can be predicted, then those path topologies having more than $k$ better topologies are proven free of the desired $k$ paths and thus can be safely discarded during the path planning process. 
To this end, a hierarchical topological tree is proposed as an implementation of the mechanism, whose nodes are proven to expand in non-homotopic directions and edges (collision-free path segments) are proven locally shortest. 
With efficient criteria that observe the order relations between partly constructed locally shortest paths being imparted into the tree, the tree nodes that expand in non-$k$-optimal topologies will not be expanded. 
As a result, the computational time for solving the $k$-SNPP problem is reduced by near two orders of magnitude. 
\end{abstract}

\section{Introduction}\label{section:introduction}

Given the starting location and the goal location in a fully known 2D environment, the $k$ shortest non-homotopic path planning ($k$-SNPP) problem aims to collect not only the globally shortest path but $(k-1)$ alternative optimal paths, where the ``alternative" property translates to path non-homotopy, and the ``optimal" property indicates that every resultant path is the locally shortest one in its own homotopy class of paths. 
Requiring both the geometry awareness for distance optimality and topology awareness for path non-homotopy, the $k$-SNPP problem is essentially a combination of the shortest path planning (SPP) problem~\cite{Lavalle2006Planning} and the topological path planning (TPP) problem and is crucial~\cite{Bhattacharya2012Topological} as a fundamental sub-module for other robot tasks, such as minimising the conflicts between multiple robots~\cite{Kimmel2013Minimizing}, finding alternative routes~\cite{Werner2014Homotopy}, providing an initial value for non-homotopic trajectory generation~\cite{Rosmann2017Integrated}, multi-robot exploration~\cite{Kim2013Topological}. 

\begin{figure}[t]
\centering
\includegraphics[width=0.44\textwidth]{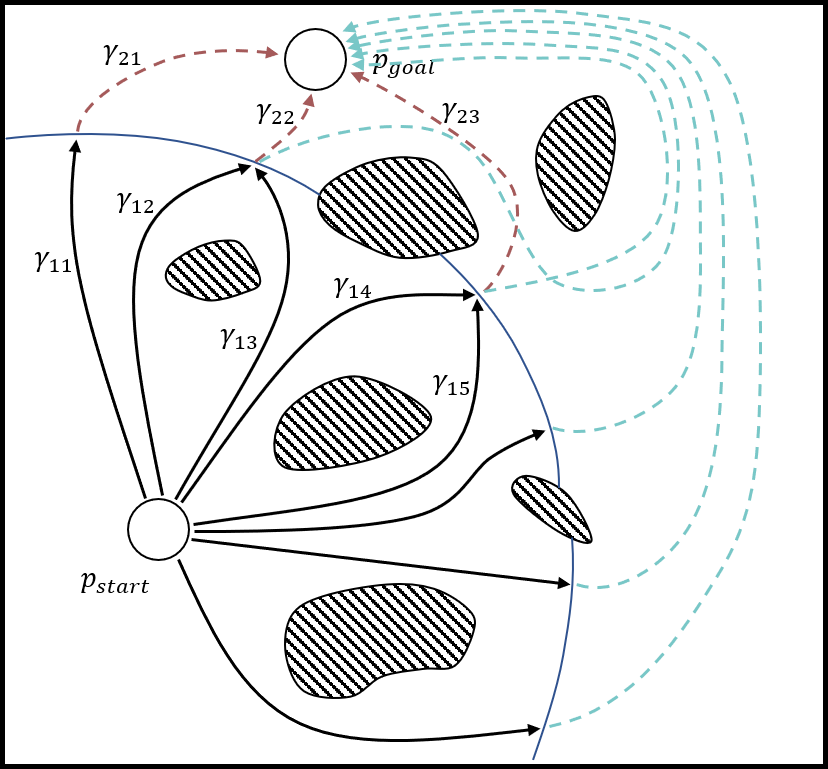}
\caption{Illustration of an intermediate state of a homotopy-aware pathfinding process for a $k$-SNPP task. 
The envelope of the explored part of the environment is visually depicted by the blue curve. 
Black arrows represent the part of paths that have been constructed by the planner, whilst dotted arrows are the path segments that remain to be constructed in the subsequent pathfinding process. 
On the one hand, it is noteworthy that the homotopy between partly constructed paths, such as $\gamma_{11}*\gamma_{21}$ and $\gamma_{12}*\gamma_{22}$, cannot be observed at the current planning state. 
Only after both $\gamma_{21}$ and $\gamma_{22}$ have been constructed can we identify their homotopy and discard one of them. 
On the other hand, the concrete implementation to stop the pathfinding along a homotopy, such as continuing the pathfinding along $\gamma_{12}$ whilst stopping it along $\gamma_{13}$, is unknown, because they may share the same path segment $\gamma_{22}$. 
Moreover, large amounts of path topologies, intuitively depicted in cyan dotted arrows for illustration, would not contain one of the $k$-SNPP resultant paths. 
However, there has not existed any method that instructs (eliminates) the topological directions for path searching from a distance-based perspective. 
}\label{fig:halfway}
\end{figure}

Existing works~\cite{Kim2014Path}~\cite{Bhattacharya2012Topological}~\cite{Pokorny2016High} solved the $k$-SNPP problem by dealing with topology awareness and geometry awareness separately. 
The \textit{homotopy-augmented graph} is required to fully characterise the topology information of the collision-free environment, whereby non-homotopic paths in the 2D environment, after being mapped into the homotopy-augmented graph, are proven to have distinct goals. 
Then, the geometry awareness is handled by performing distance-based pathfinding processes in the homotopy-augmented graph, such as the Dijkstra~\cite{Dijkstra1959Note} algorithm. 
In this regard, the $k$-SNPP problem in the 2D environment is equivalently solved as an infinite-goal shortest pathfinding in the homotopy-augmented graph, from the image of the 2D starting location to all of the images of the 2D goal location, and the pathfinding finishes after $k$ nearest goals are reached. 
Note that the pathfinding process could not know in prior which $k$ goals are the $k$ nearest, hence it inevitably looks for the resultant paths in all directions and would be extremely inefficient. 

One natural idea that motivates a significant improvement in algorithmic efficiency is to develop a mechanism that couples geometric information and topological information, leveraging the topological information to identify non-$k$-optimal path topologies whilst the pathfinding process is performing. 
As such, the pathfinding along non-$k$-optimal topologies can be safely terminated. 
A simple example could be easily appreciated to reflect the efficiency of a possible distance-based topology simplification mechanism: 
Let the A* algorithm be applied to find the globally shortest path in a symmetric ring-like collision-free environment, it will expand nodes symmetrically in all (two) directions, both clockwise (CW) and counter-clockwise (CCW), because which topology contains the globally shortest path is unknown.
If any assertion of the topological optimality exists, such as ``the globally shortest path bypasses the obstacle in CCW direction", then the A* searching space would be reduced to half (the CCW direction only) without loss of global optimality of the resultant path. 
Unfortunately, there has not existed any work that predicts the distance-based optimality of path topologies. 

Save from the toy example discussed above, a distance-based topology simplification mechanism is in general a module that \textit{predicts} the distance-based order among non-homotopic locally shortest paths. 
For any two path topologies, the one having a shorter locally shortest path is better than the other one. 
The pathfinding process along the $k$-best topologies must be necessary to collect the $k$ desirable resultant paths, whilst that of any other topologies would be unnecessary. 
Developing a topology simplification mechanism is however a non-trivial practice because of the following three reasons (see Fig.~\ref{fig:halfway} for illustration): 

\begin{enumerate}
\item The mechanism functions at an intermediate state of the homotopy-aware pathfinding process, when all locally shortest paths are still partly constructed and thus have not rendezvoused at the goal location. 
This means that the term ``homotopy" among partly constructed paths is undefined. 
Topological definitions such as $H$-signature~\cite{Bhattacharya2012Topological} and winding number~\cite{Pokorny2016High} may be legit in distinguishing different partly constructed paths which have the same endpoint, but they have not been developed to perform topology simplification. 
\item Since ``homotopy" is undefined, the implementation of terminating the pathfinding along ``a specific homotopy" is also undefined. 
\item Since locally shortest paths are only partly-constructed, their length is unknown, which makes comparing their length an unachievable task. 
This becomes further impossible because the length of locally shortest paths, closely related to the goal locations and the shape of obstacles in the environment, 
cannot be precisely predicted. 
\end{enumerate}

In this paper, a practical distance-based topology simplification mechanism to efficiently solve the $k$-SNPP problem in 2D environments is deeply investigated. 
Throughout this paper, we assume that the planner knows no prior information about the environments, neither topological nor geometrical. 
For example, when the planner detects two disconnected parts of the boundary of a single obstacle, it is unable to distinguish whether they belong to a single obstacle or there are two obstacles with a traversable corridor in between. 
In particular, the number of internal obstacles is unknown which was usually required for calculating topological invariants such as the $H$-signature. 
The main contributions of this paper are summarised as follows: 
\begin{enumerate}
\item We generalise the classic definition of ``homotopy" to a new definition ``distinguished homotopy" so that the intermediate state of the homotopy-aware path planning process can be analysed in depth. 
\item A goal location relaxation strategy is presented and shown necessary for proposing any distance-based topology simplification mechanism. (Section~\ref{section:goal_relaxation})
\item A hierarchical topological tree is proposed as an implementation of the mechanism, whose nodes are collision-free robot locations, edges are collision-free paths connecting nodes, and each node is assigned a sub-region in the environment. 
Topologically, the tree explores the environment in a complete manner, exhaustively in all topologies.  
Geometrically, the concatenation of edges is proven locally shortest, which is in contrast to all existing hierarchical topological structures~\cite{Ge2011Simultaneous}~\cite{Aldahak2013Frontier}.
The tree is essentially an exhaustive solver for $k$-SNPP without any topology simplification. (Section~\ref{section_node} and~\ref{section_theoretical_analysis})
\item Efficient criteria that compare the partly constructed locally shortest paths are imparted to the topological tree, which act as a tree branch pruning mechanism so that the nodes expanded in non-$k$-optimal topologies are discarded. 
As a result, the $k$-SNPP problem is solved in a reasonable computational time. (Section~\ref{section_related_optimality})
\end{enumerate}

The following sections are organised as follows: 
Section~\ref{section_problem_modelling} formally defines the $k$-SNPP problem. 
Section~\ref{section:goal_relaxation} deeply analyses the necessity of a goal location relaxation strategy for distance-based topology simplification. 
Section~\ref{section_node} presents the construction of a hierarchical topological tree. 
Section~\ref{section_theoretical_analysis} proves the completeness of pathfinding in tree node sub-regions and the local optimality of the concatenation of tree edges. 
Section~\ref{section_related_optimality} constructs the topology simplification mechanism upon the tree structure. 
Experiments are collected in Section~\ref{section_experiment}. 
Relation to existing works are summarised in Section~\ref{section_relatedworks}, with final concluding remarks gathered in Section~\ref{section_conclusion}. 

\section{Problem Modelling: $k$-SNPP}\label{section_problem_modelling}
In this section, we provide formal definitions and notations for the $k$-SNPP problem. 
\begin{definition}
(Homotopy~\cite{Rotman2013Introduction}) Given two collision-free paths $\gamma_1$ and $\gamma_2$ connecting the start point $p_{start}$ and the goal point $p_{goal}$ of the planning task, they are homotopic, denoted by $\gamma_1\simeq \gamma_2$, if one can be continuously deformed into the other within the collision-free part of the environment, with endpoints fixed. Each homotopy class is denoted by $[\cdot]$. 
By writing $[\gamma]$ we take path $\gamma$ as a representative of its homotopy class. 
\end{definition}

\begin{remark}
We formally denote $\varphi$ as the quotient mapping from a set of paths to the homotopy classes. 
Here the set of paths is given in the context. 
Then the subset of paths that belongs to $[\gamma]$ is denoted as $\varphi^{-1}([\gamma])$. 
\end{remark}

\begin{definition}
(Locally Shortest Path) Given one homotopy class of paths, the locally shortest path is the path with minimum length, 
\begin{equation}
\gamma^* = \mathop{\rm argmin}\limits_{\tilde{\gamma}\in\varphi^{-1}([\gamma])}g(\tilde{\gamma})
\end{equation}
where $g(\cdot)$ returns the length of path. 
In 2D planning scenarios, the locally shortest path in a single homotopy class is unique. 
\end{definition}

The $k$-\textit{SNPP} problem is to find the first $k$ locally shortest paths, formally defined as:

\begin{problem}\label{prob:ksnpp}
($k$ Shortest Non-homotopic Path Planning) 
Given the start point $p_{start}$, the goal $p_{goal}$ in a 2D environment, the $k$ shortest non-homotopic path planning ($k$-SNPP) problem is to find the resulting paths $\gamma_1^*, \cdots, \gamma_k^*$ from the set of all paths $\mathscr{P}$, such that: 
\begin{enumerate}
\item $\gamma_1^*$ is the shortest locally optimal path among all homotopy classes of paths, i.e., the globally optimal path. 
If there are multiple such paths, $\gamma_1^*$ is one of them. 
\item $\gamma_i^*, 2\leq i\leq k$ is the shortest locally optimal path in the left-out path space $\mathscr{P}\backslash \left(\varphi^{-1}([\gamma_1^*])\cup\cdots\cup\varphi^{-1}([\gamma_{i-1}^*])\right)$. 
If there are multiple such paths, $\gamma_i^*$ is one of them. 
\end{enumerate}
\end{problem}

\section{Goal Location Relaxation for Distance-based Topology Simplification}\label{section:goal_relaxation}
In this section, we first define the terminology \textit{distinguished homotopy} which is used to analyse partly constructed locally shortest paths. 
Then, a goal location relaxation strategy is shown necessary for a distance-based topology simplification mechanism. 

Note that in this section we assume that the planner discussed is a generic exploration process of the configuration space from the starting configuration to the goal configuration, such as A*~\cite{hart1968formal} and RRT*~\cite{Karaman2011Sampling} for solving SPP, and homotopy-aware Dijkstra~\cite{Bhattacharya2012Topological} and A*~\cite{Kim2015Path} for solving $k$-SNPP. 
The discussion is not applicable if either the topological information or the geometric information of the whole environment has been pre-abstracted prior to the path planner being adopted. 
Cases belonging to this category may be (a) the Voronoi graph~\cite{Aurenhammer1991Voronoi} is pre-calculated wherein different paths must be non-homotopic, and (b) the environment has been modelled into a polygonal environment and the Visibility graph~\cite{Lozano1979Algorithm}  is pre-calculated wherein all locally shortest path segments have been collected. 

\subsection{Distinguished Homotopy}

\begin{definition}
(Distinguished Homotopy) Let $\gamma$ be a collision-free path whose endpoints are $p_{start}$ and a generic point $p$. 
The distinguished homotopy class of $\gamma$ represents the set of paths such that 
\begin{enumerate}
\item The paths start from $p_{start}$ and visit $p$. 
\item The locally shortest path homotopic to $\gamma$ has been known by the planner. 
\end{enumerate}
Without abuse of symbols, we also use $[\gamma]$ as the notation of the distinguished homotopy represented by $\gamma$. 
And the set of such paths is still represented by $\varphi^{-1}([\gamma])$ for consistency. 
When the path from $p_{start}$ to $p$ is well-defined in the context, we also use $[p]$ to refer to the distinguished homotopy, and $\varphi^{-1}([p])$ to refer to the set of paths in $[p]$. 
\end{definition}

The distinguished homotopy is a generalisation of the homotopy because for a path whose endpoints are $p_{start}$ and $p_{goal}$, its distinguished homotopy is the same as its homotopy. 
And terminating the pathfinding along a distinguished homotopy $[\gamma]$ could be clearly defined as discarding all the paths in $\varphi^{-1}([\gamma])$ during the planning process. 

\subsection{Distance-based Topology Simplification}
The distance-based topology simplification is literally to remove the path topologies whose locally shortest path is too long to be one of the $k$-SNPP resultant paths.  
To this end, whether a locally shortest path is one of the $k$-SNPP resultant paths depends on its distance-based priority against all other locally shortest paths, i.e., the number of other locally optimal paths that have a shorter length. 
This can be also interpreted as the comparison between homotopies, referred to as \textit{relative optimality}, defined as follows: 
\begin{definition}\label{def:partial_order}
(Relative Optimality of Homotopies) Given the starting location $p_{start}$ and the goal location $p_{goal}$, the distance-based relative optimality between homotopy classes is defined by comparing the length of their locally shortest path, 
\begin{equation}\label{equ:def:homo_por}
[\gamma_1]\leq [\gamma_2]\Leftrightarrow \min\limits_{\gamma\in \varphi^{-1}([\gamma_1])} g(\gamma) \leq \min\limits_{\gamma\in \varphi^{-1}([\gamma_2])}g(\gamma)
\end{equation}
where $g(\cdot)$ returns the length of path. 
When strict inequality $[\gamma_1]< [\gamma_2]$ is obtained, we say $[\gamma_1]$ is relatively optimal (compared to $[\gamma_2]$), and $[\gamma_2]$ is relatively non-optimal (compared to $[\gamma_1]$). 
\end{definition}

The relative optimality of homotopies is relevant to all the distance-based planning problems, indicating the unnecessity of non-$k$-optimal path homotopies for planning. 
Its significance in solving SPP and $k$-SNPP are presented in the following corollary: 

\begin{corollary}
(Unnecessity of Homotopy)
For two homotopy classes of paths $[\gamma_1]$ and $[\gamma_2]$, if $[\gamma_1]<[\gamma_2]$, then 
\begin{enumerate}
\item The globally shortest path will not lie in $\varphi^{-1}([\gamma_2])$. 
\item If there have existed $(k-1)$ smaller elements of $[\gamma_2]$, i.e.,  $[\gamma_{2_1}]<[\gamma_2], \cdots, [\gamma_{2_{k-1}}]< [\gamma_2]$, then any path in $\varphi^{-1}([\gamma_2])$ will not be one of the paths of the $k$-SNPP solution. 
\item If $[\gamma_2]$ has been unnecessary, then any homotopy $[\gamma_3]$ that satisfies $[\gamma_2]\leq [\gamma_3]$ is also unnecessary. 
\end{enumerate}
\end{corollary}

However, one logistic contradiction in the above definitions is that, if the length of the two locally shortest paths are known, then we have already constructed them. 
This means that even if the relatively non-optimal homotopy is identified, the computational load for the pathfinding along it has been paid. 
Therefore, although the above discussions are correct as definitions, they cannot be directly applied to path planner designing. 
In contrast, similar definitions based on distinguished homotopy are motivated. 

\begin{definition}\label{def:relative_optimality_of_distinguished_homotopy}
(Relative Optimality of Distinguished Homotopies) Given the starting location $p_{start}$, the distance-based relative optimality of distinguished homotopies is defined by 
\begin{equation}\label{eqn:relative_optimality}
[\gamma_1]\prec [\gamma_2]\Leftrightarrow \min\limits_{\gamma\in \varphi^{-1}([\gamma_1])} g(\gamma) < \min\limits_{\gamma\in \varphi^{-1}([\gamma_2])}g(\gamma)
\end{equation}
When $[\gamma_1]\prec [\gamma_2]$, we say $[\gamma_1]$ is relatively optimal (compared to $[\gamma_2]$), and $[\gamma_2]$ is relatively non-optimal (compared to $[\gamma_1]$). 
Here the partial order notation is adopted to indicate that not all distinguished homotopies are comparable. 
And we ignore the coincidental cases when two locally shortest paths are of the same length. 
\end{definition}

The main difference in \textbf{Definition~\ref{def:relative_optimality_of_distinguished_homotopy}} and \textbf{Definition~\ref{def:partial_order}} is that paths are unnecessary to be fully known (i.e., in \textbf{Definition~\ref{def:relative_optimality_of_distinguished_homotopy}} we may have $\gamma_1(1) \neq \gamma_2(1)$, $\gamma_1(1)\neq p_{goal}$, and $\gamma_2(1)\neq p_{goal}$. 
And the paths from $\gamma_1(1)$ and $\gamma_2(1)$ to $p_{goal}$ could be left unconstructed). 

The utilisation of the relative optimality of distinguished homotopies is presented as the following corollary: 

\begin{corollary}\label{coro:partial_order}
(Unnecessity of Distinguished Homotopy)
During the planning process, for two distinguished homotopies $[\gamma_1]$ and $[\gamma_2]$, once we observe $[\gamma_1]\prec[\gamma_2]$, then 
\begin{enumerate}
\item For SPP, the pathfinding along $[\gamma_2]$ can be safely terminated. 
\item For $k$-SNPP, if there have existed $(k-1)$ smaller elements of $[\gamma_2]$, i.e.,  $[\gamma_{2_1}]\prec[\gamma_2], \cdots, [\gamma_{2_{k-1}}]\prec [\gamma_2]$, then the pathfinding along $[\gamma_2]$ can be safely terminated. 
\item If a distinguished homotopy $[\gamma_2]$ has been observed unnecessary, then any distinguished homotopy $[\gamma_3]$ that satisfies $[\gamma_2]\prec[\gamma_3]$ can also be safely terminated. 
\end{enumerate}
\end{corollary}

The only remaining problem is how to construct the relative optimality among distinguished homotopies whilst the planner is performing. 

\subsection{Goal Location Relaxation}
The length of locally shortest paths is theoretically sensitive to the goal location because locally shortest paths are always the concatenation of straight path segments and wall-following path segments. 
As a result, a small variation of the goal location may influence the distance-based priority of locally shortest paths thoroughly.
A simple example is that, say the goal location is in the vicinity of the saddle curve~\cite{Shnaps2014Online} formed by two path topologies, then the comparison of the two topologies fully depends on the side of the saddle curve that the goal lies in. 
See Fig.~\ref{fig:saddle_curve} for illustration. 
In this regard, the precise length of locally shortest paths cannot be predicted. 
In order to remove the goal location dependency when constructing the relative optimality, we are inspired to consider relaxing the goal location from a single point to a region in the environment.

\begin{theorem}\label{thm:goal_location_relaxation}
(Goal Location Relaxation)
Given two distinguished homotopies $[\gamma_1]$ and $[\gamma_2]$, if there exists a collision-free connected region $\Omega$ such that: 
\begin{enumerate}
\item $p_{start}\notin \Omega$ and  $p_{goal} \in \Omega$. 
\item Regarding any collision-free location $q$ on the boundary of the $\Omega$ region as a fake goal location for planning, the relative optimality between distinguished homotopies $[\gamma_1]\prec [\gamma_2]$ still holds. 
This means that, denoting the set of paths that connect $p_{start}$ and $q$ as $\mathscr{P}_q$, and letting $\varphi_q$ be the quotient mapping from $\mathscr{P}_q$ to its homotopy equivalence, we have 
\begin{equation}\label{equ:sufficient_condition}
\min\limits_{\gamma\in \varphi_q^{-1}([\gamma_1])} g(\gamma) < \min\limits_{\gamma\in \varphi_q^{-1}([\gamma_2])}g(\gamma)
\end{equation}
where $g(\cdot)$ returns the length of path. 
\end{enumerate}
then we observe $[\gamma_1]\prec [\gamma_2]$. 
\end{theorem}

\begin{proof}
Denote the shortest path in $\varphi^{-1}([\gamma_2])$ as $\gamma_2^*$. 
Since $p_{start}\notin \Omega$ and $p_{goal}\in \Omega$, there must be an intersection of $\gamma_2^*$ and the boundary of $\Omega$, denoted as $q$. 
Since $\gamma_2^*$ is collision-free, $q$ is also collision-free. 

Then, we can find the locally shortest path in $\varphi_q^{-1}([\gamma_1])$. 
We concatenate it with the truncated part of $\gamma_2^*$ that starts at $q$ and ends at $p_{goal}$ and denote the concatenated path as $\gamma_{1q}^*$. 
By condition 2), $g(\gamma_{1q}^*) < g(\gamma_2^*)$. 
And notice that $\gamma_{1q}^*$ is essentially the locally shortest path that not only belongs to $\varphi^{-1}([\gamma_1])$ but also visits $q$. 
Hence we have 
\begin{equation}
\min\limits_{\gamma\in \varphi^{-1}([\gamma_1])} g(\gamma) \leq g(\gamma_{1q}^*) < g(\gamma_2^*) = \min\limits_{\gamma\in \varphi^{-1}([\gamma_2])}g(\gamma)
\end{equation}
which yields $[\gamma_1]\prec [\gamma_2]$.
\end{proof}

\begin{figure}[t]
\centering
\includegraphics[width=0.48\textwidth]{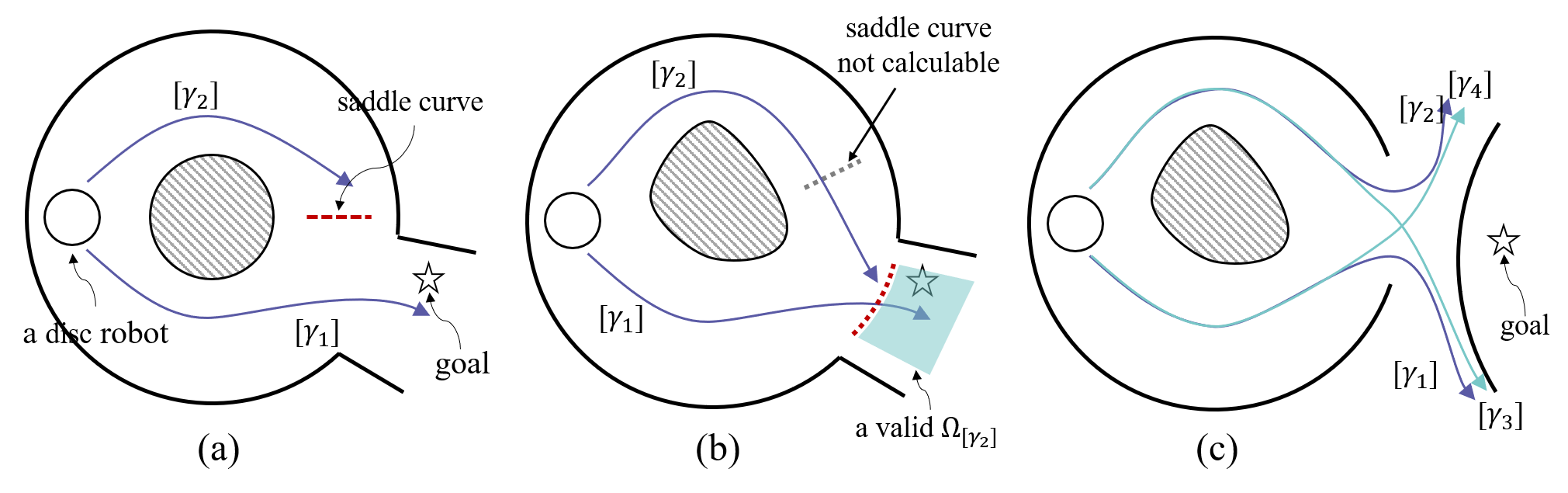}
\caption{
Illustration of the relative optimality of topologies in an SPP task. 
Topological directions are visualised as arrowed paths, represented by $[\gamma_1]\sim [\gamma_4]$. 
Theoretically, the points towards which the locally shortest path in $[\gamma_2]$ will be shorter than the one in $[\gamma_1]$ are conditioned by their saddle curve. 
In (a), under the assumption that both the robot and the internal obstacle are circular, the saddle curve can be predicted as a straight line segment. 
If the goal is given in the bottom half of the ring-like region, then ideally the pathfinding process should not be performed along topology $[\gamma_2]$, because even if a locally shortest path is constructed, it is not the SPP solution. 
The key problem is how to observe that $[\gamma_2]$ is relatively non-optimal than $[\gamma_1]$ because as long as obstacles are not ideal, such as the case in (b), the shape of the saddle curve is unknown. 
In a slightly more generalised environment it can be further seen that the ``topologies" will bifurcate with the path planner running, as shown in (c), where $[\gamma_1]$ bifurcates to $[\gamma_1]$ and $[\gamma_4]$, and $[\gamma_2]$ bifurcates to $[\gamma_2]$ and $[\gamma_3]$. 
Intuitively, the desired topology simplification should be ``preserving $[\gamma_1]$ and $[\gamma_2]$, terminating $[\gamma_3]$ and $[\gamma_4]$". 
However, there does not exist an algorithm that makes assertions about the distance-based optimality of topology. 
Solution to this case study will be presented in detail in Section~\ref{subsection:fig_3d_solution}. 
}\label{fig:saddle_curve}
\end{figure}

\textbf{Theorem~\ref{thm:goal_location_relaxation}} is a sufficient and unnecessary proposition to observe relative optimality between distinguished homotopies. 
This can be easily revealed by letting the $\Omega$ region be a single-point region, $\{p_{goal}\}$, where then the conditions become the same as Eqn.~(\ref{eqn:relative_optimality}) ($p_{goal}\in \{p_{goal}\}$ is obvious, and the boundary of $\Omega$ is a single point $p_{goal}$ which makes Eqn.~(\ref{eqn:relative_optimality}) and Eqn.~(\ref{equ:sufficient_condition}) exactly the same formula). 
It is noteworthy that although in generic an unnecessary proposition makes the result more difficult to be concluded, the merit of \textbf{Theorem~\ref{thm:goal_location_relaxation}} is that we can freely design the shape of $\Omega$ to verify the conditions in \textbf{Theorem~\ref{thm:goal_location_relaxation}} with high efficiency.

In summary, a $k$-SNPP planner with a distance-based topology simplification mechanism should perform as follows:  
Whilst the planner is constructing paths in different topologies, it keeps looking for possible construction of the above-mentioned region $\Omega$ between two distinguished homotopies say $[\gamma_1]$ and $[\gamma_2]$. 
Once $\Omega$ is verified to satisfy the conditions in \textbf{Theorem~\ref{thm:goal_location_relaxation}}, the relative optimality between $[\gamma_1]$ and $[\gamma_2]$ is observed, say $[\gamma_1]\prec [\gamma_2]$. 
Then, the unnecessity of $[\gamma_2]$ is checked as per \textbf{Corollary~\ref{coro:partial_order}}. 
If $[\gamma_2]$ is unnecessary, the pathfinding along $[\gamma_2]$ will be terminated. 
In this paper, the $\Omega$ region construction will be presented in \textbf{Theorem~\ref{thm:omega}} which is based on the geometric structures of the hierarchical topological tree (to be proposed in Section~\ref{section_node} and Section~\ref{section_theoretical_analysis}), whose verification of the conditions in \textbf{Theorem~\ref{thm:goal_location_relaxation}} would be extremely efficient, detailed in Section~\ref{section_related_optimality}. 

\begin{table}[t]
\centering
\caption{Nomenclature}\label{tab:nomenclature}
\begin{tabular}{lp{6cm}l}
\hline
\textbf{Symbols}& \textbf{Meanings} \\
\hline
$i, j, k, l$ & Generic indices\\
$\gamma$ & A path\\
$[\cdot]$ & Homotopy class of paths\\
$\varphi$ & The mapping from a set of paths to its homotopy classes. The set of paths depends on the context. \\
$\mathscr{P}_*$ & A set of paths. \\
$e(\theta_j)$ & The endpoint of the $j$-th ray of the node\\
$M_{\rm free}$ & The collision-free part of the environment\\
$p^i$ & The \textit{source point} of the $i$-th node\\
$p_{start}$ & The start point for $k$-SNPP\\
$p_{goal}$ & The goal point for $k$-SNPP\\
$Q^i$ & The sub-region of the $i$-th node\\
$R$ & The lethal radius of the robot\\
\multicolumn{2}{c}{\textbf{(Below are for the $k$-th corridor of the $i$-th node)}}\\
$\alpha_k^i$ & An edge of the topological tree\\
$\Delta_k^i$ & The \textit{gap sweeper}\\
$c_k^i$ & The \textit{critical point}\\
$D_k^i$ & The \textit{corridor}\\
$e_{{\rm f}k}^i$ & The obstacle that obstructs gap sweeper $\Delta_k^i$\\
$e_{{\rm n}k}^i$ & The obstacle that hit by the near ray forming the gap\\
$o_k^i$ & The endpoint of the \textit{gap sweeper}\\
\hline
\end{tabular}
\end{table}

\section{Construction of Topological Tree}\label{section_node}
In this section, the construction of a hierarchical topological tree is proposed. 
The procedure of constructing a node of the tree is alternately running three modules, 
\textit{\textbf{sparse raycasting}}, \textit{\textbf{planning in corridor}}, and \textit{\textbf{sweeping the gap}}. 
We assume that the robot is circular whose radius is $R$. 
Without loss of generality, we also assume the environment is represented by grids. 
Frequently used variables are listed in Table~\ref{tab:nomenclature}.

\subsection{Overview}
The reader is referred to the algorithm pseudocode given in \textbf{Algorithm}~\ref{alg:raystar}. 
The core role that the topological tree plays is to construct all locally shortest paths segment-by-segment instead of point-by-point. 
The tree edges will be proven a part of locally shortest paths. 
And in the opposite, each locally shortest path will be proven incrementally constructed by a branch of tree nodes. 
As such, each node represents a distinguished homotopy, and the branching of nodes represents the bifurcation of the corresponding distinguished homotopy.

\begin{figure}[t]
\centering
\includegraphics[width=0.48\textwidth]{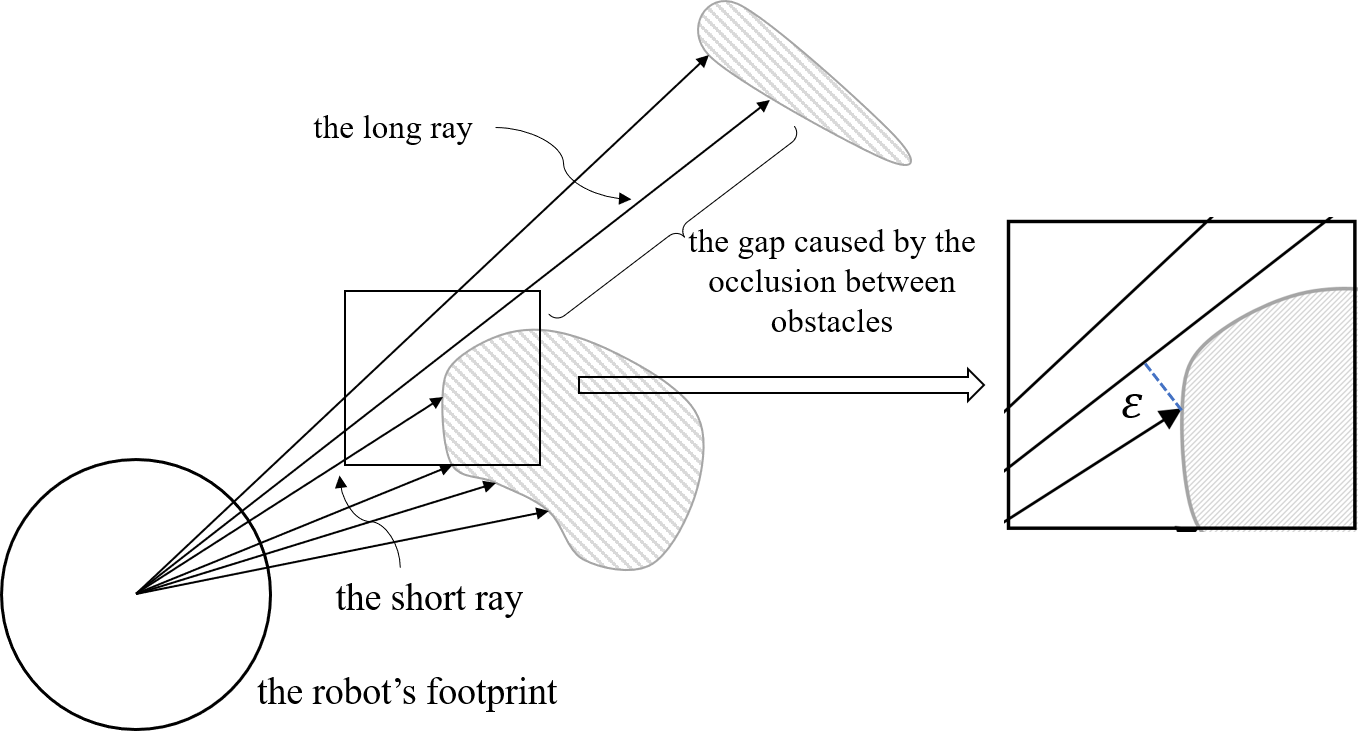}
\caption{Stopping criterion of inserting new rays.}\label{fig:epsilon}
\end{figure}

\subsection{Sparse Raycasting}\label{section:sparse_raycasting}
This module aims to find a list of obstacles that implicitly form connected C-space obstacles in a reasonable computational time. 
Given the radius of the robot's footprint as $R$, we need to find a sequence of obstacle points with clearance less than $2R$. 

Given a source point $p$ and an orientation $\theta_j$, rays depart from $p$ and extend until hitting an obstacle denoted by $e(\theta_j)$. 
We also use $p$ to represent ``the node whose source point is $p$" in short. 
The orientation of rays are such that
\begin{equation}\label{equ:range}
\theta_j \in \left\{
\begin{aligned}
&[0, 2\pi],\ &\mbox{ for the root node}\\
&[\theta_{\min}, \theta_{\max}],\ &\mbox{ for other nodes}
\end{aligned}
\right.
\end{equation}
where $\theta_{\min}$ and $\theta_{\max}$ are given by its parent node which will be presented at the end of this section.  
The length of the ray is denoted by $\Phi_p(\theta_j)$. 
There have been efficient implementations for raycasting, such as Bresenham's line algorithm~\cite{Bresenham1965Algorithm}, so we omit its details. 
For the root node, we initially create $5$ rays orienting to $0, \frac{\pi}{2}, \pi, \frac{3\pi}{2}, 2\pi$, whilst for other nodes, we create $2$ rays orienting to $\theta_{\min}$ and $\theta_{\max}$, as such the angle bisector of consecutive rays are meaningful. 
Storing all rays in angle increasing order as shown in Eqn.~(\ref{equ_raylists}), new rays are inserted into the bisector of consecutive rays.
\begin{equation}\label{equ_raylists}
\{\theta_j, e(\theta_j), \Phi_p(\theta_j)\}, j = 1, \cdots, n, 0\leq \theta_1 < \cdots < \theta_n \leq 2\pi
\end{equation}
There are clear stopping criteria for inserting rays between the $j$-th and $(j+1)$-th rays: 
\begin{enumerate}
\item If
\begin{equation}\label{equ_insert}
\parallel e(\theta_j) - e(\theta_{j+1})\parallel_2 < 2R
\end{equation}
then we stop inserting new rays between the $j$-th ray and the $(j+1)$-th ray, because a robot with lethal radius $R$ cannot go through the middle of $e(\theta_j)$ and $e(\theta_{j+1})$.
\item When the clearance between occluded obstacles is larger than $2R$, then Eqn.~(\ref{equ_insert}) cannot be satisfied. 
\footnote{The cases given $\Phi_p(\theta_j) > \Phi_p(\theta_{j+1})$ or $\Phi_p(\theta_j) < \Phi_p(\theta_{j+1})$ are totally symmetric, so hereafter we assume that $\Phi_p(\theta_{j+1}) - \Phi_p(\theta_{j}) > 2R$.}
If $e(\theta_j)$ is near the $(j+1)$-th ray enough, constrained by $\epsilon$ whose physical meaning is the minimal admissible distance between the robot's footprint and the obstacle, 
\begin{equation}\label{equ_epsilon}
\left(\Phi_p(\theta_{j})\right)^2 - \left(e(\theta_{j})\cdot\frac{ e(\theta_{j+1})}{\Phi_p(\theta_{j+1})}\right)^2 < \epsilon^2
\end{equation}
we also stop inserting new rays. 
See Fig.~\ref{fig:epsilon} for the physical meaning of $\epsilon$. 
In a practical setting $\epsilon$ as $0.1d$ is enough for grid-based applications, where $d$ is the grid size of the map. 
\end{enumerate}

Two consecutive rays violating Eqn.~(\ref{equ_insert}) but satisfying Eqn.~(\ref{equ_epsilon}) form a depth discontinuity of the node's surrounding region, which was also referred to as a \textit{gap}~\cite{Tovar2007Distance}. (line~\ref{alg:raycasting} and line~\ref{alg:raycasting_2} in \textbf{Algorithm}~\ref{alg:raystar})

\begin{figure}[t]
\centering
\includegraphics[width=0.3\textwidth]{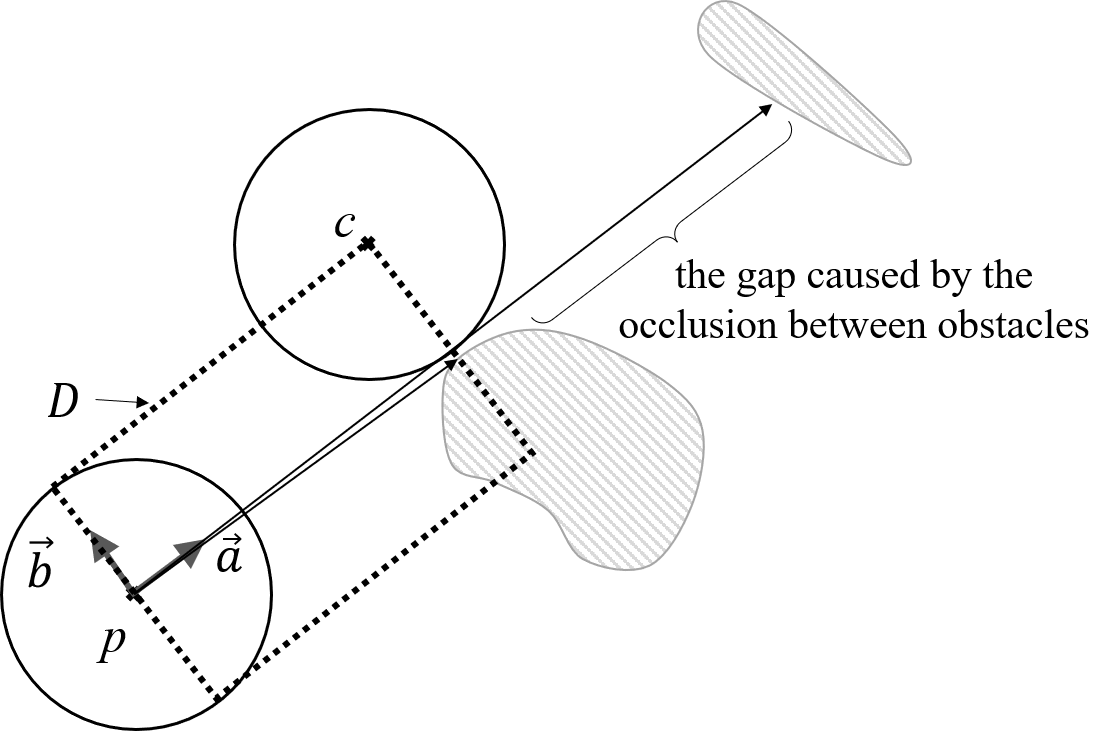}
\caption{Variables in the corridor corresponding to a gap. }\label{fig:corridor}
\end{figure}

\subsection{Planning in Corridor}
This module defines a corridor (note that we do not assume that corridors are obstacle-free as in~\cite{Ge2011Simultaneous}) for each gap, and path planning is carried out within the corridor. 
The corridor is a rectangle staying symmetrically on the long ray that forms the gap. 
See Fig.~\ref{fig:corridor} for visualisation. 
Let the $j$-th ray and the $(j+1)$-th ray satisfy Eqn.~(\ref{equ_epsilon}), then 
\begin{equation}
\Phi_{\rm near} \triangleq \Phi_p(\theta_j),\ \Phi_{\rm far} \triangleq \Phi_p(\theta_{j+1})
\end{equation}
Let $\vec{a}$ be the unit vector parallel to the long ray,
\begin{equation}
\vec{a} = \left(\cos\theta_{j+1}, \sin\theta_{j+1}\right)\label{equ_a}
\end{equation}
$\vec{b}$ be perpendicular to $\vec{a}$, 
\begin{equation}
\vec{b} = \left(\cos\left(\theta_{j+1}+ \frac{\pi}{2}\right), \sin\left(\theta_{j+1} + \frac{\pi}{2}\right)\right)
\end{equation}
$(\vec{a}, \vec{b})$ forms an orthonormal coordinate. 
Then the corridor $D$ is the collision-free part of a rectangular region spanned by $\vec{a}$ and $\vec{b}$, which can be parameterised as 
\begin{equation}\label{equ_D}
D \subset \{p + \phi\vec{a} + \varphi\vec{b}\}, \phi\in [0, \Phi_{\rm near}], \varphi\in [-R, R]
\end{equation}
After Eqn.~(\ref{equ_D}) we can represent every point in the modified frame by $(\phi, \varphi)$, such as the source point $p$ is $(0, 0)$.

Finally, the \textit{critical point} $c$ is defined as a corner point of $D$, 
\begin{equation}\label{equ:c}
c= p + \Phi_{\rm near}\vec{a} + R\vec{b}
\end{equation}

\begin{algorithm}[H]
    \caption{$k$-SNPP Solver}\label{alg:raystar}
    \begin{algorithmic}[1]
        \Require start point $p_{start}$, goal point $p_{goal}$ 
        \Ensure $k$ shortest non-homotopic paths $\alpha_1, \cdots, \alpha_k$
            \State Initialise the priority queue, $open = \{p_{start}\}$
			\State Initialise the edge list $edge = \varnothing$
			\State Initialise the gap sweeper list $gs = \varnothing$
            \While{in the $i$-th cycle}  
							\If{$open == \varnothing$}
								\State \Return{all collected result paths (\textbf{Thm~\ref{thm:completeness}})}
							\EndIf
							\State $p^i$ = currently best candidate in the $open$ list
							\State // Expand node $i$
							\State list of gaps $gaps = \varnothing$
							\State $gaps \leftarrow$ \textit{\textbf{Sparse Raycasting}} from $p^i$\label{alg:raycasting}
							\While{for each gap in $gaps$ (say the $l$-th)}
								\State list of new obstacle $newobs = \varnothing$
								\State Define corridor $D^i_l$
								\State $[newobs, \alpha_l^i] \leftarrow$ \textit{\textbf{Planning in Corridor}}\label{alg:planning_in_corridor}
								\If{$\alpha_l^i$ is found}
									\State $[newobs, \Delta_l^i] \leftarrow$ \textit{\textbf{Sweeping the Gap}}\label{alg:sweeping}
								\Else
									\State Removing this gap by removing rays
								\EndIf
								\State Inserting rays towards all obstacles in $newobs$\label{alg:insert_rays}
								\State $gaps \leftarrow$ \textit{\textbf{Sparse Raycasting}} from $p^i$\label{alg:raycasting_2}
                    \State // new gaps may be collected\label{alg:iter}
							\EndWhile
							\State // Looking for result paths
							\If{$p_{goal} \in Q^i$}
								\If{Found path in $Q^i$ from $p^i$ to $p_{goal}$}
									\State Collect the result path
								\Else
									\State \Return {no path (\textbf{Thm~\ref{thm:subregion_completeness}})}
								\EndIf
							\EndIf
							\State // Assume there are $L$ critical points
							\State // Looking for partial order relations
							\For{each gap (say the $l$-th)}
							\For {each gap sweeper in $gs$ (say $
\Delta_m^j$)}
								\State Compare $\Delta_l^i$ to $\Delta_m^j$ (\textbf{Lem}~\ref{lem:sweeper_to_sweeper}, \textbf{Prop}~\ref{prop:goal_dependent})
								\State Compare $\alpha_l^i$ to $\Delta_m^j$ (\textbf{Lem}~\ref{lem:sweeper_to_path}, \textbf{Cor}~\ref{coro:sweeper_to_path})
							\EndFor
							\For {each edge in $edge$ (say $\alpha_m^j$)}
								\State Compare $\Delta_l^i$ to $\alpha_m^j$ (\textbf{Lem}~\ref{lem:sweeper_to_path}, \textbf{Cor}~\ref{coro:sweeper_to_path})
								\State Compare $\alpha_l^i$ to $\alpha_m^j$ (\textbf{Lem}~\ref{lem:path_to_path})
							\EndFor
								\State Store $\alpha_l^i$ into $edge$
								\State Store $\Delta_l^i$ into $gs$
							\EndFor
							\State Inherit partial order relation to child nodes (\textbf{Thm}~\ref{thm:inherit})
                \State Estimate the cost $g(c_l^i) + h(c_l^i), l = 1, \cdots, L$
                \State Push $c_l^i, l = 1, \cdots, L$ into the $open$ list
						\State Update $open$ list (\textbf{Thm}~\ref{thm:to_result_path})
            \EndWhile  
    \end{algorithmic}  
\end{algorithm}
\noindent

\noindent
Regarding $p$ as the start point and $c$ as the goal point in $D$, we use the A*~\cite{hart1968formal} planner to find the shortest path connecting $p$ and $c$ within $D$, denoted by $\alpha$. 

If the path $\alpha$ cannot be found, then $D$ is not C-space connected, and we have detected the obstacles that obstruct the robot. 
In this case, we do the following steps to exclude missing obstacles, remove the rays that have been inserted to the unreachable area, discard the invalid gap, and find missing gaps (line~\ref{alg:planning_in_corridor} $\sim$ line~\ref{alg:iter} in \textbf{Algorithm}~\ref{alg:raystar}): 

\begin{enumerate}
\item We create rays pointing at all the newly detected obstacles and insert them into the ray list. 
\item In the updated ray list, if the distance between the endpoint of two rays is less than $2R$, then We remove all the rays between them. 
\item The $D$ does not admit the path $\alpha$, so the gap is invalid and is removed. 
Note that the removal of a gap is simply avoiding creating a child node (will be presented in Section~\ref{subsection_iterative}). 
\item Since the new rays might violate the stopping criteria of sparse raycasting (Eqn.~(\ref{equ_insert}) and Eqn.~(\ref{equ_epsilon})), we may need to insert new rays at the angle bisector of a ray inserted in step 1) and a ray created initially.  \label{step:re-refine}
\item New gaps may be detected by step 4). 
If so, we insert them into the gap list. 
Note that the gap list need not be ordered when creating the node, so the gap insertion is simply appending them at the rear.
\end{enumerate}

\noindent
Illustrations of the above steps are provided in Fig.~\ref{fig:failure}. 

Note that the more rays are created, the rarer that missing obstacle appears, with the above steps being seldom executed. 
But these steps would still exist for the completeness of the algorithm because the creation of a large number of sparse rays is still different from dense raycasting which is computationally unaffordable in practice. 

\begin{figure}[t]
\centering
\subfigure[The path connecting $p$ and $c$ within $D$ does not exist, because there is a missing obstacle which is overlooked by rays and detected in the planning phase. 
New rays are created, a new corridor structure is constructed, and the old corridor is discarded. ]{
\includegraphics[width = 0.48\textwidth]{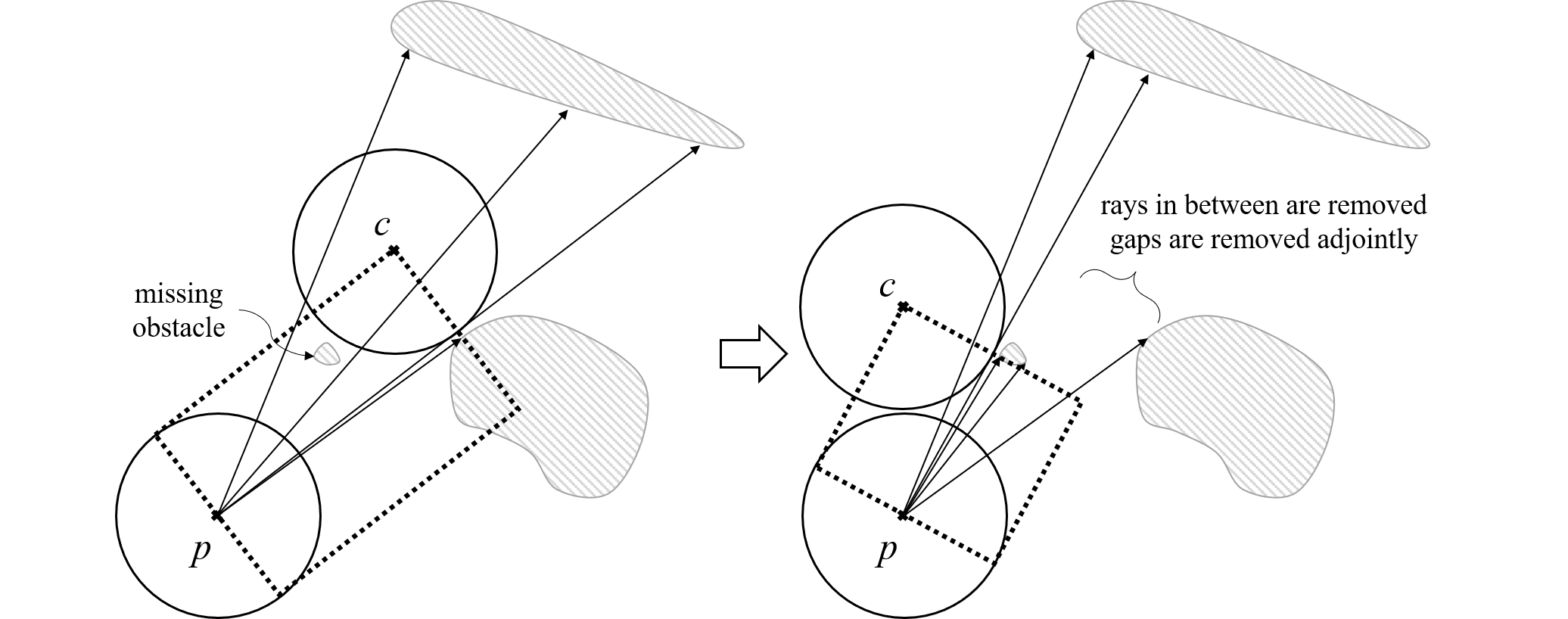}\label{fig:failure:a}
}
\subfigure[Two corridors are obstructed mutually. When the intermediate rays are removed, both corridors are removed. ]{
\includegraphics[width = 0.48\textwidth]{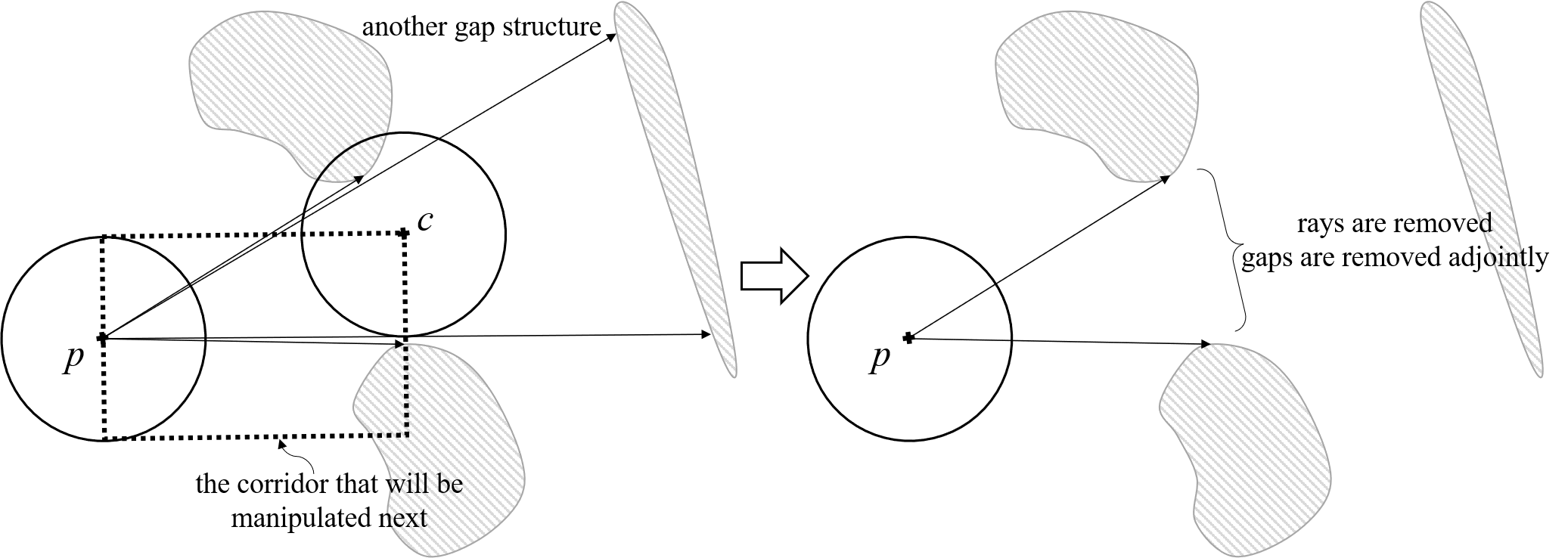}\label{fig:failure:b}
}
\caption{The cases where the planning in corridor fails. In all the cases shown above, new rays will be inserted and the gap structures will be updated. When the node expansion finishes, all the corridors admit a collision-free path. 
}\label{fig:failure}
\end{figure}

\begin{figure*}[t]
\centering
\subfigure[]{
\includegraphics[width=0.23\textwidth]{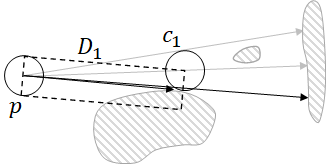}
}
\subfigure[]{
\includegraphics[width=0.23\textwidth]{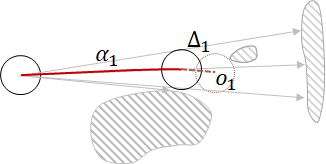}
}
\subfigure[]{
\includegraphics[width=0.23\textwidth]{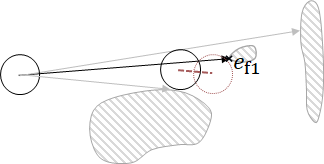}
}
\subfigure[]{
\includegraphics[width=0.23\textwidth]{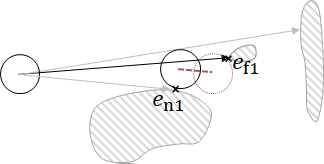}
}
\subfigure[]{
\includegraphics[width=0.23\textwidth]{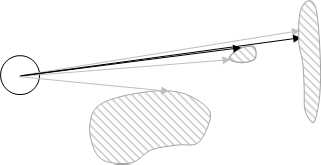}
}
\subfigure[]{
\includegraphics[width=0.23\textwidth]{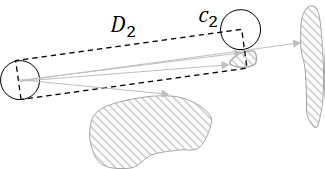}
}
\subfigure[]{
\includegraphics[width=0.23\textwidth]{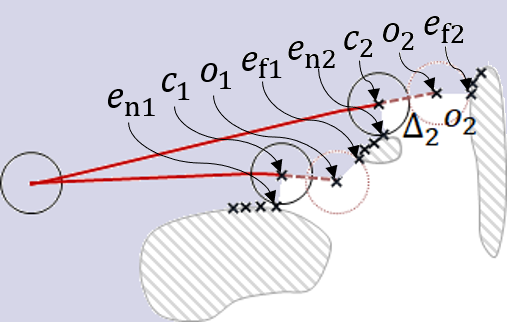}
}
\subfigure[]{
\includegraphics[width=0.23\textwidth]{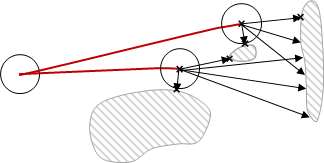}
}
\caption{Illustration of the node expansion. (a) Using sparse rays, one corridor $D_1$ is defined whilst a small obstacle is overlooked. (b) After $\alpha_1$ is found in $D_1$, the gap sweeper $\Delta_1$ (represented by the red dashed line) is constructed and stops at $o_1$ (where the robot is about to hit the missing obstacle). (c) A new ray is inserted pointing at the obstacle that $o_1$ hits, $e_{\rm f1}$. (d) The rays with orientation between $e_{\rm n1}$ and $e_{\rm f1}$ are removed. (e) With new rays being inserted, a new gap is detected. (f) A new corridor $D_2$ is defined, together with $c_2$. (g) After the gap sweeper $\Delta_2$ is constructed, all points marked by $\times$ together with $\Delta_1$ and $\Delta_2$ are well-defined and form a C-space connected boundary of $Q$ (painted in light purple). (h) Illustration of the orientation of sparse raycasting of child nodes. 
}\label{fig:insert_new_rays}
\end{figure*}

\subsection{Sweeping the Gap}

After the path has been constructed in the corridor, we do collision-checking starting from $c$ with orientation $\vec{a}$ (the long ray) until a position $o$ where the robot hits an obstacle. 
The checked line segment is referred to as a \textit{gap sweeper}, denoted by $\Delta$ and parameterised as
\begin{equation}
\Delta \subset \{p+\phi\vec{a}+R\vec{b}\}, \phi\in [\Phi_{\rm near}, \Phi_{\rm far}]
\end{equation}
The obstacle that stops the gap sweeping process is denoted by $e_{\rm f}$. 
(line~\ref{alg:sweeping} in \textbf{Algorithm}~\ref{alg:raystar})

This process is very simple but is a finishing touch to the completeness of planning in the tree node sub-region (to be formally defined in the next subsection): All points in $\Delta$ are reachable by the robot since the edge $\alpha$ has been constructed and $\Delta$ is collision-free. 
In fact, the gap sweepers of a node guarantees that all the collision-free points on the boundary of the sub-region are reachable by travelling along the tree. 
An example showing the necessity of the gap sweeper will be given in Section~\ref{subsection_counterexample}. 

\subsection{Sub-Region of a Node}\label{subsection_complete_region}
After running the above-mentioned modules, rays and gaps have been created. 
For each gap, the corridor has been defined, the shortest path within the corridor that connects the source point and the critical point has been found, and the gap sweeper has been constructed. 
See Fig.~\ref{fig:insert_new_rays} for a conceptual illustration of the whole process. 
Let the corresponding sub-region of node $p$ be denoted by $Q$. 
A visualisation of the boundary of $Q$ is illustrated in Fig.~\ref{fig:insert_new_rays}(g). 
Concretely, denoting the endpoint of the short ray forming the gap as $e_{\rm n}\triangleq e(\theta_j)$, the point $e_{\rm n}, c, o$, and $e_{\rm f}$ form a C-space connected boundary. 
So the combination of such points of all gaps $\{e_{{\rm n}k}, c_k, o_k, e_{{\rm f}k}\}$ and the endpoint of all rays $\{e(\theta_*)\}$ form a C-space closed contour, which is the boundary of $Q$.  

Finally, $\theta_{\rm f}$ and $\theta_{\rm n}$ are the related orientation of $e_{\rm f}$ and $e_{\rm n}$ to $c$, which also specifies the orientation range (in Eqn.~(\ref{equ:range})) of the sparse raycasting of the child node. 
To list the rays of the child node in an orientation-increasing order, $\theta_{\min}, \theta_{\max}$ are given as follows
\begin{equation}
(\theta_{\min}, \theta_{\max}) = 
\left\{
\begin{aligned}
&(\theta_{\rm f}, \theta_{\rm n}), \Phi_p(\theta_j) > \Phi_p(\theta_{j+1})\\
&(\theta_{\rm n}, \theta_{\rm f}), \Phi_p(\theta_j) < \Phi_p(\theta_{j+1})
\end{aligned}
\right.
\end{equation}
They are shifted by $2\pi$ so that 
\begin{equation}
0\leq \theta_{\min} < 2\pi,\ \theta_{\min}< \theta_{\max} < \theta_{\min} + 2\pi
\end{equation}

\subsection{Iterative Expansion of Nodes}\label{subsection_iterative} 
The hierarchical topological tree is constructed by iteratively creating new nodes, where the critical point will be the source point of the child node. 
The edges of the tree are the path segments constructed in the corridors. 
For simplicity, we define the terminology for easy reference to the path indicated by the tree. 
\begin{definition}\label{def:tree_path}
(Tree Path) Given a point $x$, the tree path is the path connecting $p_{start}$ and $x$ following the ``parent-child" structures of the tree. Concretely, 
\begin{enumerate}
\item When $x$ is a source point or a critical point, the tree path is the concatenation of edges. 
\item When $x$ is a point on the edge, the tree path is the concatenation of edges truncated at $x$. 
\item When $x$ is a point on the gap sweeper, the tree path concatenates not only the edges but also the gap sweeper truncated at $x$. 
\item When $x$ is a generic point in the sub-region of a node, say $Q^i$, the tree path the concatenation of the tree path of $p^i$ and the locally shortest path from $p^i$ to $x$ in $Q^i$.
\end{enumerate}
If $x$ is in multiple structures, then the tree path depends on the context. 
The length of the tree path is denoted by $g(x)$. 
\end{definition}

As a final remark, the expansion order of nodes makes no difference to the shape of their sub-regions, because the construction of a new node only requires the position of $p$ and the value of $\theta_{\min}$ and $\theta_{\max}$, which depend only on its parent node. 
Hence we may follow the basic idea of the priority queue (the same one as A*~\cite{hart1968formal}): Pushing all the unexpanded critical points into a queue whereby the cost is the sum of the cost-to-come and a heuristic cost-to-go, 
\begin{equation}
Fcost \triangleq Gcost+Hcost = g(c) + \parallel c - p_{goal}\parallel_2 
\end{equation}
And the critical point with the least cost will be chosen as the seed point for the next iteration.

\section{Completeness and Local Optimality}\label{section_theoretical_analysis}
In this section, the completeness of planning in sub-regions and the local optimality of tree edges are proven. 
All notations are adopted, with super-script denoting the index of the node, and sub-script denoting the index of the corridor. 

\subsection{Completeness}\label{subsection_completeness}
\begin{proposition}\label{prop:completeness}
Assume $p_{goal}\in Q^j$. 
If there exists a path from $p_{start}$ to $p_{goal}$ in the whole environment, then there must be a path from $p^j$ to $p_{goal}$ in $Q^j$. 
\end{proposition}
\begin{proof}
The proof is constructive. 
If there exists a path from $p_{start}$ to $p_{goal}$ which does not enter $Q^j$ at $p^j$, then it must have an intersection with one of the gap sweepers, say $\Delta_l^j$. 
Denote the intersection as $q$, we can  construct a collision-free path from $p^j$ to $p_{goal}$ in $Q^j$: starting from $p^j$, passing through $c_l^j$ to $q$, and taking over the following segment of the path to $p_{goal}$. 
\end{proof}

\begin{theorem}\label{thm:subregion_completeness}
(Sub-Region Completeness) Assume $p_{goal}\in Q^j$. If we cannot find a collision-free path in $Q^j$ from $p^j$ to $p_{goal}$, then there does not exist a resultant path in the whole environment. 
In other words, the planning algorithm can terminate immediately with proven no path. 
\end{theorem}
\begin{proof}
The contrapositive of \textbf{Proposition~\ref{prop:completeness}}. 
\end{proof}

\begin{theorem}\label{thm:completeness}
(Completeness) When there is no node to expand, if $p_{goal} \notin Q^i, \forall i$, then there does not exist a path connecting $p_{start}$ and $p_{goal}$ in the whole environment.
\end{theorem}
\begin{proof}
Denote the union of the sub-regions of all nodes as 
\begin{equation}
Q_{\rm union} = \bigcup\limits_{i}Q^i, \forall i
\end{equation}
The boundary of $Q_{\rm union}$ is formed by the obstacle points detected by rays, where the clearance between consecutive obstacle points is less than $2R$. 
Since the robot's radius is $R$, it cannot move from inside $Q_{\rm union}$ to outside $Q_{\rm union}$. 
Since $p_{start} = p^1\in Q^1 \subset Q_{\rm union}$ and  $p_{goal}\notin Q_{\rm union}$, there is no resultant path. 
\end{proof}

\subsection{An Example of Non-completeness without Gap Sweepers}\label{subsection_counterexample}

\begin{figure}[t]
\centering
\subfigure[]{
\includegraphics[width = 0.22\textwidth]{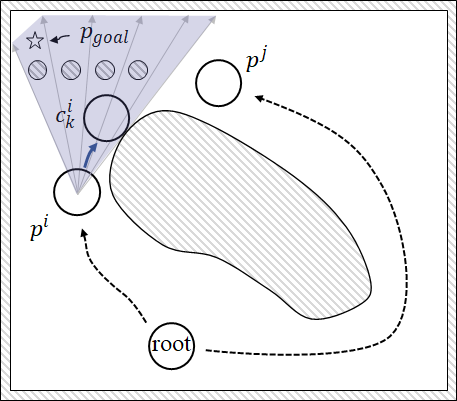}
}
\subfigure[]{
\includegraphics[width = 0.22\textwidth]{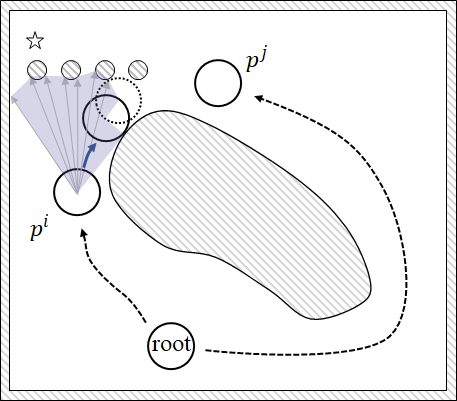}
}
\caption{Illustration of a counterexample of node expansion where gap sweepers are not generated. 
(a) The non-existence of a path connecting $p^i$ and the goal in $Q^i$ is not equivalent to the non-existence of a path connecting the root and the goal in the whole map. 
(b) The desired construction of $Q^i$ with gap sweeper involved. 
}\label{fig:counterexample}
\end{figure}

To help the reader understand the necessity of gap sweepers, we give a counterexample: If rays are sparse and we do not generate gap sweepers, then the region $Q$ loses completeness. 
See Fig.~\ref{fig:counterexample}, four small obstacles are missing by the rays emitting from $p^i$, and they are not detected during the pathfinding from $p^i$ to $c_k^i$. 
So the expansion of nodes finishes normally. 
Given the goal shown in the figure (within $Q^i$), there is no path in $Q^i$ connecting $p^i$ and the goal. 
However, a resultant path exists that visits $p^j$. 
Hence, the local non-existence of paths towards the goal in $Q^i$ is not equivalent to the non-existence of the resultant path in the whole environment, i.e., no sub-region completeness. 

\subsection{Local Optimality}\label{subsection_optimality} 

In this subsection, we prove that locally shortest paths must be a concatenation of tree edges, hence are being constructed segment-by-segment along with the node expansion.
\begin{proposition}\label{prop:deformation}
Denote $\mathscr{P}_{D_k^i}$ as the set of paths that connect $p^i$ and $c_k^i$ in $D_k^i$. 
Denote $\mathscr{P}_{Q^i}$ as the set of paths that connect $p^i$ and $c_k^i$ in $Q^i$. 
Assume that there is no internal obstacle in $Q^i$ overlooked by the sparse rays, then 
\begin{equation}
\mathop{\rm argmin}\limits_{\alpha\in \mathscr{P}_{D_k^i}}g(\alpha) = \mathop{\rm argmin}\limits_{\alpha\in \mathscr{P}_{Q^i}}g(\alpha)
\end{equation}
\end{proposition}
\begin{proof}
We prove that $\forall \alpha\in \mathscr{P}_{Q^i}$, $\alpha\notin \mathscr{P}_{D_k^i}$, $\exists \tilde{\alpha}\in \mathscr{P}_{D_k^i}$ such that $g(\tilde{\alpha}) < g(\alpha)$. 
Since $c_k^i, p^i\in D_k^i$, $\alpha$ must have two intersections with the boundary of $D_k^i$. 
We replace the part of the path in $\alpha$ that is outside $D_k^i$ by the corresponding straight path segments on the boundary of $D_k^i$, and then $\tilde{\alpha}$ is constructed. 
The straight path segment is always shorter than the curved one in $\alpha$, hence $g(\tilde{\alpha}) < g(\alpha)$. 
\end{proof}
\begin{theorem}\label{thm:optimality_single_subregion}
Denote $\mathscr{P}_{D_k^i}$ as the set of paths that connect $p^i$ and $c_k^i$ in $D_k^i$. 
Denote $\mathscr{P}_{Q^i}$ as the set of paths that connect $p^i$ and $c_k^i$ in $Q^i$. 
Then 
\begin{equation}
(\alpha_k^i =)\mathop{\rm argmin}\limits_{\alpha\in \mathscr{P}_{D_k^i}}g(\alpha) = \mathop{\rm argmin}\limits_{\alpha\in \mathscr{P}_{Q^i}}g(\alpha)
\end{equation}
\end{theorem}
\begin{proof}
If there is no missing obstacle in $Q^i$, then it is \textbf{Proposition~\ref{prop:deformation}}. 
So we assume that some missing obstacles $\{O_j\}$ ($j$ is a generic index) were overlooked by sparse rays. Denote 
\begin{equation}
\bar{Q}^i = Q^i\cup \{O_j\}
\end{equation}
(the missing obstacles are purposely overlooked to construct the simply-connected region $\bar{Q}^i$)
and denote $\mathscr{P}_{\bar{Q}^i}$ as the set of paths that connect $p^i$ and $c_k^i$ in $\bar{Q}^i$. 
By \textbf{Proposition~\ref{prop:deformation}}, 
\begin{equation}\label{eqn:22}
\begin{aligned}
\mathop{\rm min}\limits_{\alpha\in \mathscr{P}_{\bar{Q}^i}}g(\alpha) &= \mathop{\rm min}\limits_{\alpha\in \mathscr{P}_{D_k^i}}g(\alpha)\\
\mathop{\rm argmin}\limits_{\alpha\in \mathscr{P}_{\bar{Q}^i}}g(\alpha) &\subset D_k^i
\end{aligned}
\end{equation}
Since $D_k^i\subset Q^i$, 
\begin{equation}\label{eqn:23}
\min\limits_{\alpha\in \mathscr{P}_{D_k^i}}g(\alpha) \geq \min\limits_{\alpha\in \mathscr{P}_{Q^i}}g(\alpha)
\end{equation}
Since $Q^i\subset \bar{Q}^i$, 
\begin{equation}\label{eqn:24}
\min\limits_{\alpha\in \mathscr{P}_{Q^i}}g(\alpha) \geq \min\limits_{\alpha\in \mathscr{P}_{\bar{Q}^i}}g(\alpha)
\end{equation}
Summarising the result from Eqn.~(\ref{eqn:22}), Eqn.~(\ref{eqn:23}), and Eqn.~(\ref{eqn:24}), 
\begin{equation}
\begin{aligned}
&\min\limits_{\alpha\in \mathscr{P}_{D_k^i}}g(\alpha) \geq \min\limits_{\alpha\in \mathscr{P}_{Q^i}}g(\alpha) \geq
\min\limits_{\alpha\in \mathscr{P}_{\bar{Q}^i}}g(\alpha) = 
\min\limits_{\alpha\in \mathscr{P}_{D_k^i}}g(\alpha)\\
\Rightarrow&\min\limits_{\alpha\in \mathscr{P}_{D_k^i}}g(\alpha) = \min\limits_{\alpha\in \mathscr{P}_{Q^i}}g(\alpha) =
\min\limits_{\alpha\in \mathscr{P}_{\bar{Q}^i}}g(\alpha)
\end{aligned}
\end{equation}
Recall $D_k^i\subset Q^i$, so $\mathop{\rm argmin}\limits_{\alpha\in \mathscr{P}_{D_k^i}}g(\alpha)\subset Q^i$, hence we have 
\begin{equation}
(\alpha_k^i =)\mathop{\rm argmin}\limits_{\alpha\in \mathscr{P}_{D_k^i}}g(\alpha) = \mathop{\rm argmin}\limits_{\alpha\in \mathscr{P}_{Q^i}}g(\alpha)
\end{equation}
\end{proof}

As a simple corollary of \textbf{Theorem}~\ref{thm:optimality_single_subregion}, to find the locally shortest paths connecting $p^i$ and $c_k^i$ in $Q^i$, we only need to run A* in the corridor $D_k^i$ instead of in the whole $Q^i$ region. 

Generally, the concatenation of locally shortest paths is no longer locally shortest. 
However, we prove that the local optimality is preserved when concatenating the edges of the proposed topological tree. 
\begin{proposition}\label{prop:optimality_for_two_subregion}
For $\forall s\in \Delta_k^i$, define $\mathscr{P}_{Q^i}$ as the set of paths that connect $p^i$ and $s$ in $Q^i$. Then 
\begin{equation}
c_k^i \in \mathop{\rm argmin}\limits_{\alpha\in \mathscr{P}_{Q^i}} g(\alpha)
\end{equation}
\end{proposition}
\begin{proof}
Let $\parallel s - c_k^i\parallel = d$. 
Define ${D'}_k^i$ as the collision-free part in a prolonged corridor: 
\begin{equation}
{D'}_k^i \subset \{p_k^i + \phi\vec{a} + \varphi\vec{b}|\phi\in [0, \Phi_{\rm near}+d], \varphi\in [-R, R]\}
\end{equation}
Then $s$ is the corner point of ${D'}_k^i$. 
Following the same discussion in \textbf{Theorem~\ref{thm:optimality_single_subregion}}, 
\begin{equation}
\mathop{\rm argmin}\limits_{\alpha\in \mathscr{P}_{{D'}_k^i}}g(\alpha) = \mathop{\rm argmin}\limits_{\alpha\in \mathscr{P}_{Q^i}}g(\alpha)
\end{equation}
where $\mathscr{P}_{{D'}_k^i}$ is the set of paths that connect $p^i$ and $s$ within ${D'}_k^i$. 
Note that when $\phi = \Phi_{\rm near}$, $c_k^i$ is the only collision-free point in the section of ${D'}_k^i$, 
\begin{equation}
c_k^i\in \mathop{\rm argmin}\limits_{\alpha\in \mathscr{P}_{{D'}_k^i}}g(\alpha)
\end{equation}
\end{proof}

\begin{theorem}\label{thm:optimality}
Assume $p_{goal}\in Q^j$. 
Any locally shortest path of $p_{goal}$ is a tree path. 
\end{theorem}
\begin{proof}
The proof is inductive. 
If $j = 1$, by the definition of tree path, the claim is correct. 

Assume $j = 2$ and $p^2 = c_k^1$. 
Let $\gamma$ be the locally shortest path from $p_{start}(=p^1)$ to $p_{goal}$. 
Since $\gamma$ must intersect with $\Delta_k^1$, the intersecting point is denoted as $s$ which separates $\gamma$ into $\gamma_1$ (from $p_{start}$ to $s$) and $\gamma_2$ (from $s$ to $p_{goal}$). 
Since $\gamma$ is a locally shortest path, both $\gamma_1$ and $\gamma_2$ are locally shortest paths. 
And by \textbf{Proposition~\ref{prop:optimality_for_two_subregion}},
\begin{equation}
c_k^1\in \gamma_1 = \mathop{\rm argmin}\limits_{\alpha\in \mathscr{P}_{Q^1}} g(\alpha)
\end{equation}
where $\mathscr{P}_{Q^1}$ is the set of paths that connect $p^1$ and $s$ in $Q^1$. 
hence $\alpha_k^1\subset \gamma_1\subset \gamma$, i.e., the claim is correct. 

When $j > 2$, let node $j$ be the $k$-th child of node $i$, i.e., $p^j = c_k^i$, and $p_{goal}\in Q^j$. 
Let $\gamma$ be the locally shortest path from $p_{start}(=p^1)$ to $p_{goal}$. 
Since $\gamma$ must intersect with $\Delta_k^i$, the intersected point is denoted as $s$ which separates $\gamma$ into $\gamma_1$ (from $p_{start}$ to $s$) and $\gamma_2$ (from $s$ to $p_{goal}$). 
By induction, assume the statement is true for the parent of node $i$, then 
\begin{equation}
s\in Q^i \Rightarrow p^i\in \gamma_1 
\end{equation}
Then, applying the above discussion for another time, 
\begin{equation}
p_{goal}\in Q^j\Rightarrow p^j \in \gamma
\end{equation}
Hence $\gamma$ is the concatenation of the tree path of $p^j$ and the final path segment from $p^j$ to $p_{goal}$ found in $Q^j$. 
\end{proof}

Finally, its opposite is also correct: 
\begin{theorem}
Assume $p_{goal}\in Q^j$. 
Any tree path of $p_{goal}$ is a locally shortest path. 
\end{theorem}
\begin{proof}
Proof by contradiction. 
If a tree path is not the locally shortest path, then we denote the locally shortest path (homotopic to the tree path) as $\gamma^*$. 
By \textbf{Theorem~\ref{thm:optimality}}, $\gamma^*$ is a tree path. 
Hence the two paths are the same. 
\end{proof}

\section{Construction of Relative Optimality}\label{section_related_optimality}

\begin{figure*}[t]
\centering
\subfigure[$g_1+g_2 > g_3 + g_4 + g_5 + \parallel c_l^j - q\parallel_2 + \parallel q - c_k^i\parallel_2$]{
\includegraphics[width=0.48\textwidth]{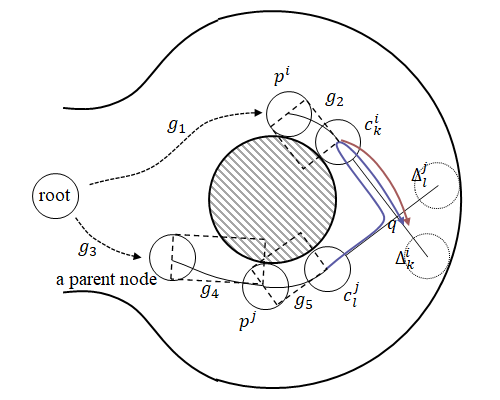}
}
\subfigure[$g_1 + g_2 + g_6 > g_3 + g_4 + \parallel p^j - q'\parallel_2 + \parallel q' - c_{k'}^{i'}\parallel_2$]{
\includegraphics[width=0.48\textwidth]{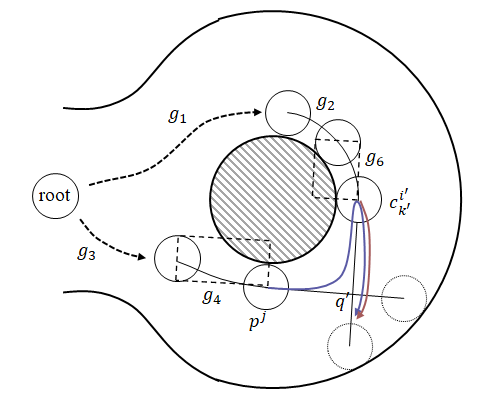}
}
\caption{(a) Illustration of the comparison between $[c_k^i]$ and $[c_l^j]$ when $\Delta_k^i, \Delta_l^j$ intersect at $q$. 
(b) As $c_k^i$ expands further, the movement cost increases from $g_1 + g_2$ to $g_1 + g_2 + g_6$, while the cost to be compared reduces. 
Sooner or later, the relative optimality can be observed. 
}\label{fig:figureloseoptimality}
\end{figure*}

In this section, the relative optimality between distinguished homotopies is constructed. 
This is achieved by constructing the $\Omega$ region of a certain distinguished homotopy, and proving that it satisfies the constraint in \textbf{Theorem~\ref{thm:goal_location_relaxation}}. 
The set of criteria proposed in this section can be seen as a baseline of the distance-based topology simplification strategy built upon the proposed hierarchical topological tree, and more criteria might be created in the future. 
For notation, the $k$-th child of node $i$ and the $l$-th child of node $j$ are again referred: Let $c_k^i$ and $c_l^j$ be the critical points to be discussed. 
The source point $p^i$, $p^j$, the path in corridors $\alpha_k^i$, $\alpha_l^j$, and the gap sweepers $\Delta_k^i$, $\Delta_l^j$ have been constructed. 
The distinguished homotopy represented by the tree path of $p^i$, $p^j$, $c_k^i$, and $c_l^j$ are denoted as $[p^i]$, $[p^j]$, $[c_k^i]$, and $[c_l^j]$, respectively. 
We formally introduce the notation \textit{leaf node} for easy discussion. 
\begin{remark}
(Leaf node) Leaf nodes are unexpanded critical points. 
\end{remark}

\subsection{Observation of Relative Optimality}\label{subsection_omega}
As an introduction, we first present a sufficient condition of relative optimality provided by a newly found resultant path.
\begin{theorem}\label{thm:to_result_path}
(Comparison to Resulting Paths)
Denote $c_l^j$ as the last critical point expanded before a locally shortest path is found with length $L$ (i.e., the newly found resultant path passes $c_l^j$). 
For any critical point $c_k^i$ that is still in the priority queue with cost $g(c_k^i) + h(c_k^i)$, if $g(c_k^i) + h(c_k^i) > L$, then we obtain $[c_l^j]\prec [c_k^i]$. 
\end{theorem}
\begin{proof}
The length of the result path visiting $c_k^i$ must be longer than the estimated cost $g(c_k^i) + h(c_k^i)$, so
\begin{equation}
\begin{aligned}
\min\limits_{\varphi^{-1}([c_k^i])} g(p_{goal}) &= g(c_k^i) + g(p_{goal})|_{{\rm visiting}\ c_k^i}\\
&> g(c_k^i) + h(c_k^i)\\
&> L = \min\limits_{\varphi^{-1}([c_l^j])} g(p_{goal})
\end{aligned}
\end{equation}
Hence $[c_l^j]\prec [c_k^i]$. 
\end{proof}

In the sequel, we assume $p_{goal}$ is still out of reach by sub-regions of node $i$, $j$ and their predecessors. 
We define the uncovered region of the pathfinding along a distinguished homotopy as $U$. 
The region $\Omega$ will be constructed based on $U$ under some conditions. 

\begin{definition}\label{def:U}
(Uncovered region of a distinguished homotopy) 
Given a distinguished homotopy $[c_k^i]$, the predecessors (i.e., its parent, and the parent of the parent, $\cdots$) of node $i$ are well-defined whose indices are recorded in $P_i$. 
The region $U_{[c_k^i]}$ is all the collision-free area that is uncovered by the sub-region of node $i$ and its predecessors. 
Formally, it is 
\begin{equation}
U_{[c_k^i]} = M_{\rm free}\backslash \bigcup\limits_{m\in P_i\cup\{i\}} Q^m
\end{equation}
\end{definition}

Then, the most important symbol of this paper, the $\Omega_{[c_k^i]}$ region for the distinguished homotopy $[c_k^i]$, is constructed as follows: 

\begin{theorem}\label{thm:omega}
(Construction of $\Omega$ region) 
For a distinguished homotopy $[c_k^i]$, if we can find another distinguished homotopy $[c_l^j]$ satisfying that, 
\begin{equation}\label{equ:suff_condition}
\min\limits_{\varphi^{-1}([c_l^j])}g(m) < \min\limits_{\varphi^{-1}([c_k^i])}g(m),\ \forall m\in \Delta_k^i
\end{equation}
then $U_{[c_k^i]}$ is our construction of $\Omega_{[c_k^i]}$, and we obtain $[c_l^j]\prec[c_k^i]$. 
\end{theorem}

\begin{proof}
Note that the boundary of $U_{[c_k^i]}$ consists of $\Delta_k^i$, its sibling gap sweepers ($\Delta_1^i, \cdots, \Delta_{k-1}^i, \Delta_{k+1}^i, \cdots$), and the sibling gap sweepers of its predecessors. 
And since we assumed that $p_{goal}$ has not been covered by $Q^m, m\in P_i\cup \{i\}$ (or else the locally shortest path has been fully constructed), $p_{goal}\in U_{[c_k^i]}$ is guaranteed. 
When the path enters $U_{[c_k^i]}$ by intersecting with the gap sweepers except $\Delta_k^i$, say $\Delta_{k-1}^i$, the path actually belongs to another distinguished homotopy $[c_{k-1}^i]$ but not $[c_k^i]$. 
Hence the points that remain to be verified in \textbf{Theorem~\ref{thm:goal_location_relaxation}} is reduced to $\Delta_k^i$, i.e., 
\begin{equation}
\min\limits_{\varphi^{-1}([c_l^j])}g(m) < \min\limits_{\varphi^{-1}([c_k^i])}g(m),\ \forall m\in \Delta_k^i
\end{equation}
\end{proof}

The following \textbf{Lemma}~\ref{lem:sweeper_to_sweeper}, \textbf{Lemma}~\ref{lem:sweeper_to_path}, \textbf{Corollary}~\ref{coro:sweeper_to_path}, and \textbf{Lemma}~\ref{lem:path_to_path} are sufficient propositions when Eqn.~(\ref{equ:suff_condition}) is guaranteed.

\begin{lemma}(Comparison between Gap Sweepers)\label{lem:sweeper_to_sweeper}
Let $\Delta_k^i, \Delta_l^j$ intersect and the intersection be $q$. 
If
\begin{equation}\label{equ:sweeper_to_sweeper}
g(c_k^i) > g(c_l^j)+\parallel c_l^j-q\parallel_2 + \parallel q - c_k^i\parallel_2
\end{equation}
then 
\begin{equation}
\min\limits_{\varphi^{-1}([c_l^j])} g(m) < \min\limits_{\varphi^{-1}([c_k^i])} g(m), \forall m\in \Delta_k^i
\end{equation}
\end{lemma}

\begin{proof}
See illustration in Fig.~\ref{fig:figureloseoptimality}(a). 
For a point $m$ on $\Delta_k^i$, by the local optimality of the tree path,
\begin{equation}
\min\limits_{\varphi^{-1}([c_k^i])}g(m) = g(c_k^i) + \parallel c_k^i-m\parallel_2 
\end{equation}
If Eqn.~(\ref{equ:sweeper_to_sweeper}) is satisfied, then we can find an alternative path towards $m$: The robot first visits $c_l^j$ following the tree path of $c_l^j$, then visits $q$ along $\Delta_l^j$, and finally reaches $m$ along $\Delta_k^i$. 
It is actually a path belonging to the distinguished homotopy $[c_l^j]$, and we have 
\begin{equation}
\begin{aligned}
g(m)|_{{\rm visiting } c_l^j} &= g(c_l^j)+\parallel c_l^j-q\parallel_2 + \parallel q-m\parallel_2 \\
&< g(c_l^j)+\parallel c_l^j-q\parallel_2\\
&~~~~~~~~+ \parallel q-c_k^i\parallel_2 + \parallel c_k^i-m\parallel_2\\
&< g(c_k^i)+\parallel c_k^i-m\parallel_2 \\
&= \min\limits_{\varphi^{-1}([c_k^i])} g(m)
\end{aligned} 
\end{equation}
Since $m$ is arbitrarily chosen, and the length of the alternative path is an upper bound of the length of the locally shortest path in $[c_l^j]$, we have
\begin{equation}\label{equ:gap_gap}
\min\limits_{\varphi^{-1}([c_l^j])} g(m) < \min\limits_{\varphi^{-1}([c_k^i])} g(m), \forall m\in \Delta_k^i
\end{equation}
\end{proof}

It can be noticed that even if the inequality Eqn.~(\ref{equ:sweeper_to_sweeper}) is currently not satisfied, see Fig.~\ref{fig:figureloseoptimality}(b),  with the leaf node $c_k^i$ being expanded, the movement cost along $[c_k^i]$ is increasing, whilst the cost to be compared is decreasing. 
Hence the comparison between path homotopies will sooner or later be observed with the tree growing. 

\begin{figure}[t]
\centering
\includegraphics[width = 0.38\textwidth]{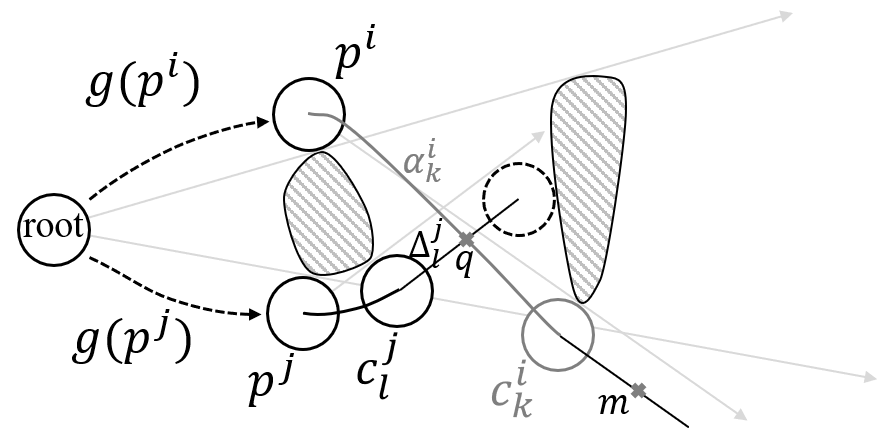}
\caption{Illustration of the intersection between the edge $\alpha_k^i$ and the gap sweeper $\Delta_l^j$. 
}\label{fig:path_sweeper_intersect}
\end{figure}

If a gap sweeper intersects with an edge, then we may also have a lemma similar to \textbf{Lemma}~\ref{lem:sweeper_to_sweeper}. 
\begin{lemma}\label{lem:sweeper_to_path}
(Comparison between Gap Sweeper and Edge)
Let the gap sweeper $\Delta_l^j$ and the edge $\alpha_k^i$ intersect at $q$. 
Separating $\alpha_k^i$ at $q$, the length of the truncated parts from $p^i$ to $q$ and from $q$ to $c_k^i$ are denoted as $g(q)|_{{\rm from}\ p^i}$ and $g(c_k^i)|_{{\rm from}\ q}$, respectively. 
If 
\begin{equation}\label{equ:path_to_sweeper}
g(p^i) + g(q)|_{{\rm from}\ p^i} > g(c_l^j)+ \parallel c_l^j - q\parallel_2 
\end{equation}
then 
\begin{equation}
\min\limits_{\varphi^{-1}([c_l^j])} g(m) < \min\limits_{\varphi^{-1}([c_k^i])} g(m), \forall m\in \Delta_k^i
\end{equation}
\end{lemma}
\begin{proof}
See Fig.~\ref{fig:path_sweeper_intersect} for illustration. 
For a point $m$ on $\Delta_k^i$, its tree path is proven the locally shortest path, thus
\begin{equation}
\begin{aligned}
\min\limits_{\varphi^{-1}([c_k^i])}g(m) &= g(c_k^i) + \parallel c_k^i - m\parallel_2\\
&= g(p^i) + g(q)|_{{\rm from}\ p^i}\\
 &~~~~~~~~ + g(c_k^i)|_{{\rm from}\ q}+ \parallel c_k^i - m\parallel_2
\end{aligned}
\end{equation}
If Eqn.~(\ref{equ:path_to_sweeper}) is satisfied, a shorter path towards $q$ is observed, which indicates an alternative path towards $m$ following path homotopy $[c_l^j]$: The path first goes to $c_l^j$ following the tree path, then reaches $q$ along $\Delta_l^j$, and finally reaches $m$ along the remaining part of $\alpha_k^i$. The length of the alternative path is 
\begin{equation}
\begin{aligned}
g(m)|_{{\rm visiting}\ c_l^j} &= g(c_l^j) + \parallel c_l^j - q\parallel_2\\
&~~~~~~~~+ g(c_k^i)|_{{\rm from}\ q} + \parallel c_k^i - m\parallel_2\\
& < g(p^i) + g(q)|_{{\rm from}\ p^i}\\
&~~~~~~~~+ g(c_k^i)|_{{\rm from}\ q} + \parallel c_k^i - m\parallel_2\\
&= g(c_k^i) + \parallel c_k^i - m\parallel_2\\
&= \min\limits_{\varphi^{-1}([c_k^i])} g(m)
\end{aligned}
\end{equation}
Since $m$ is arbitrarily chosen, and the length of the constructed alternative path is only an upper bound of the local minimum, 
\begin{equation}
\min\limits_{\varphi^{-1}([c_l^j])} g(m) < \min\limits_{\varphi^{-1}([c_k^i])} g(m), \forall m\in \Delta_k^i
\end{equation}
\end{proof}

Similarly, the edge being relatively optimal than the gap sweeper may be observed. 
\begin{corollary}\label{coro:sweeper_to_path}
Let the gap sweeper $\Delta_l^j$ and the edge $\alpha_k^i$ intersect at $q$. 
If 
\begin{equation}\label{equ:sweeper_to_path}
g(c_l^j) > g(p^i) + g(q)|_{{\rm from}\ p^i} +  \parallel q - c_l^j\parallel_2
\end{equation}
then 
\begin{equation}
\min\limits_{\varphi^{-1}([p^i])}g(m) < \min\limits_{\varphi^{-1}([c_l^j])}g(m), \forall m\in \Delta_l^j
\end{equation}
\end{corollary}
\begin{proof}
For a point $m$ on $\Delta_l^j$, 
\begin{equation}
\min\limits_{\varphi^{-1}([c_l^j])} g(m) = g(c_l^j) + \parallel c_l^j - m\parallel_2
\end{equation}
If Eqn.~(\ref{equ:sweeper_to_path}) is satisfied, then we obtain an alternative path visiting $p^i, q, c_l^j, m$ in order, 
\begin{equation}
\begin{aligned}
g(m)|_{{\rm visiting}\ p^i} &= g(p^i) + g(q)|_{{\rm from}\ p^i}\\
&~~~~~~~~+ \parallel q - c_l^j\parallel_2 + \parallel c_l^j - m\parallel_2\\
&< g(c_l^j) + \parallel c_l^j - m\parallel_2\\
&= \min\limits_{\varphi^{-1}([c_k^i])}g(m)
\end{aligned}
\end{equation} 
Hence 
\begin{equation}
\min\limits_{\varphi^{-1}([p^i])}g(m) < \min\limits_{\varphi^{-1}([c_l^j])}g(m), \forall m\in \Delta_l^j
\end{equation}
\end{proof}

\begin{figure}[t]
\centering
\includegraphics[width = 0.35\textwidth]{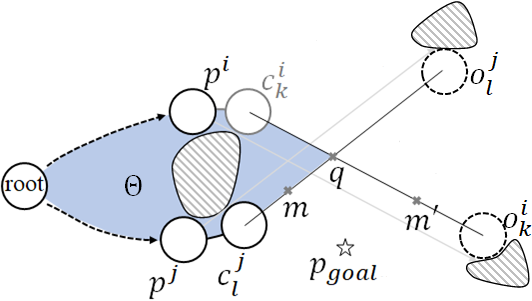}
\caption{Illustration of the optimality criterion when two \textit{gap sweeper}s intersect, with $p_{goal}\notin\Theta$ (painted in blue).}\label{fig:sweeper_sweeper_no_goal}
\end{figure}

\begin{lemma}\label{lem:path_to_path}
(Comparison between Edges)
When two edges $\alpha_k^i$ and $\alpha_l^j$ intersect at $q$. 
Separating $\alpha_k^i$ and $\alpha_l^j$ at $q$, the length of the truncated part of paths from $p^i$ to $q$ and from $p^j$ to $q$ are denoted as $g(q)|_{{\rm from}\ p^i}$ and $g(q)|_{{\rm from}\ p^j}$. Then 
\begin{eqnarray}
&g(p^i) + g(q)|_{{\rm from}\ p^i} > g(p^j) + g(q)|_{{\rm from}\ p^j}\label{equ:new_path_to_path} \\
\Rightarrow & \min\limits_{\varphi^{-1}([c_l^j])} g(m) < \min\limits_{\varphi^{-1}([c_k^i])} g(m), \forall m\in \Delta_k^i
\\
&g(p^j) + g(q)|_{{\rm from}\ p^j} > g(p^i) + g(q)|_{{\rm from}\ p^i} \\
\Rightarrow & \min\limits_{\varphi^{-1}([c_k^i])} g(m) < \min\limits_{\varphi^{-1}([c_l^j])} g(m), \forall m\in \Delta_l^j
\end{eqnarray}
\end{lemma}
\begin{proof}
Easy, based on the classic shortcut mechanism. 
\end{proof}

One final result is that the $\Omega$ region of the child node is included in that of the parent node, so the following inheritance of relative non-optimality is guaranteed. 

\begin{theorem}\label{thm:inherit}
(Inheritance of Relative Non-optimality along Tree Expansion) 
\begin{equation}
[c_l^j]\prec [p^i] \Rightarrow [c_l^j] \prec [c_k^i], \forall k
\end{equation}
\end{theorem}
\begin{proof}
By \textbf{Theorem}~\ref{thm:optimality}, 
\begin{equation}
p^i\in \mathop{\rm argmin}\limits_{\alpha\in \varphi^{-1}([c_k^i])}g(\alpha),\ \forall k
\end{equation}
hence 
\begin{equation}
\varphi^{-1}([p^i]) = \bigcup\limits_{k} \varphi^{-1}([c_k^i])
\end{equation}
Then, 
\begin{equation}
\begin{aligned}
&[c_l^j]\prec [p^i]\\
\Rightarrow & \min\limits_{\alpha\in \varphi^{-1}([c_l^j])}g(\alpha) < \min\limits_{\alpha\in \varphi^{-1}([p^i])}g(\alpha)\\
\Rightarrow & \min\limits_{\alpha\in \varphi^{-1}([c_l^j])}g(\alpha) < \min\limits_{\alpha\in \varphi^{-1}([c_k^i])}g(\alpha),\ \forall k \\ 
\Rightarrow & [c_l^j]\prec [c_k^i],\ \forall k
\end{aligned}
\end{equation}
\end{proof}

Based on \textbf{Theorem}~\ref{thm:inherit}, all nodes, including not only leaf nodes but also expanded nodes and unnecessary nodes (those having been observed as relatively non-optimal), are meaningful for constructing the relative optimality between distinguished homotopies. 

\subsection{More Efficient Comparison Concerning the Goal Location}\label{section:more_subtle_constraint}

In all the discussions above, intuitively speaking, observing $[c_l^j]\prec [c_k^i]$ by \textbf{Theorem~\ref{thm:goal_location_relaxation}} requires the distinguished homotopy $[c_l^j]$ to be better ``enough" to be observed: For any goal location in $\Omega_{[c_k^i]}$ the locally shortest path in $\varphi^{-1}([c_l^j])$ is always shorter than the locally shortest path in $\varphi^{-1}([c_k^i])$. 
Noticing that the goal for the planning tasks is known prior, constructing a smaller region $\Omega_{[c_l^j]}$ will make the sufficient conditions ``weaker", which means that the relative optimality can be observed more easily. 
In this regard, the following \textbf{Proposition}~\ref{prop:goal_dependent} shows as an example that there indeed exists more delicate constructions of the criteria of relative optimality.

\begin{proposition}\label{prop:goal_dependent}
(Relative Optimality Depending on the Goal Location)
Let $\Delta_k^i$ and $\Delta_l^j$ intersect at $q$. 
The tree path of $q$ along the distinguished homotopy $[c_k^i]$  and $[c_l^j]$ form a closed boundary of an internal region, denoted by $\Theta$. 
See Fig. \ref{fig:sweeper_sweeper_no_goal} for illustration. 
If
\begin{equation}\label{equ:two_sweepers_no_goal}
\left\{
\begin{aligned}
&p_{goal}\notin\Theta\\
&\mbox{Both leaf nodes will expand ``inside" $\Theta$}\\
&g(c_k^i) + \parallel c_k^i - q\parallel_2 > g(c_l^j) + \parallel c_l^j - q\parallel_2
\end{aligned}
\right.
\end{equation}
then $[c_l^j]\prec [c_k^i]$. 
\end{proposition}
\begin{proof}
Based on the location of $p_{goal}$, the theorem is proven by enumerating all possible point $m$ and $m'$ on $\Delta_l^j$ and $\Delta_k^i$ that the path may visit to reach $p_{goal}$ in $\Omega_{[c_k^i]}\backslash (\Omega_{[c_k^i]}\cap \Theta)$. 
See Fig.~\ref{fig:sweeper_sweeper_no_goal} for illustration. 

Denote $m$ as a point in $\Delta_l^j$, between $c_l^j$ and $q$. 
Note that we have the triangular inequality for the minimal length of paths towards $m$ in homotopy $[c_k^i]$, because $c_k^i$, $q$, and $m$ are not colinear, 
\begin{equation}
\begin{aligned}
\min\limits_{\varphi^{-1}([c_k^i])} g(m) &> \min\limits_{\varphi^{-1}([c_k^i])}g(q) - \parallel q - m \parallel_2\\ 
&= g(c_k^i) + \parallel c_k^i - q\parallel_2 - \parallel q - m\parallel_2
\end{aligned}
\end{equation}
And also notice that $c_l^j$, $m$, and $q$ are colinear, 
\begin{equation}
\begin{aligned}
\min\limits_{\varphi^{-1}([c_l^j])} g(m) &= g(c_l^j) + \parallel c_l^j - m\parallel_2\\
&= g(c_l^j) + \parallel c_l^j - q\parallel_2 - \parallel q - m\parallel_2
\end{aligned}
\end{equation}
If the inequality in Eqn.~(\ref{equ:two_sweepers_no_goal}) is satisfied, then an alternative path along the distinguished homotopy $[c_l^j]$ will be always shorter than the one in $[c_k^i]$, 
\begin{equation}
\begin{aligned}
\min\limits_{\varphi^{-1}([c_l^j])} g(m) &= g(c_l^j) + \parallel c_l^j - q\parallel_2 - \parallel q - m\parallel_2\\
&< g(c_k^i) + \parallel c_k^i - q\parallel_2 - \parallel q - m\parallel_2\\
&< \min\limits_{\varphi^{-1}([c_k^i])} g(m)
\end{aligned}
\end{equation}

Similarly, denoted $m'$ as a point in $\Delta_k^i$, between $q$ and $o_k^i$. 
Since $c_k^i, q$ and $m'$ are colinear, we have 
\begin{equation}
\begin{aligned}
\min\limits_{\varphi^{-1}([c_k^i])}g(m') &= g(c_k^i) + \parallel c_k^i - m' \parallel_2\\
& = g(c_k^i) + \parallel c_k^i - q\parallel_2 + \parallel q - m'\parallel_2
\end{aligned}
\end{equation}
And by the triangular inequality, 
\begin{equation}
\begin{aligned}
\min\limits_{\varphi^{-1}([c_l^j])}g(m') &< \min\limits_{\varphi^{-1}([c_l^j])}g(q) + \parallel q - m'\parallel_2\\
&= g(c_l^j) + \parallel c_l^j - q\parallel_2 + \parallel q - m'\parallel_2
\end{aligned}
\end{equation}
Hence 
\begin{equation}
\min\limits_{\varphi^{-1}([c_l^j])}g(m') < \min\limits_{\varphi^{-1}([c_k^i])}g(m')
\end{equation}
Since $m$ and $m'$ are arbitrarily chosen, we prove that $[c_l^j]\prec [c_k^i]$. 
\end{proof}

Note that Eqn.~(\ref{equ:two_sweepers_no_goal}) is a sufficient condition of Eqn.~(\ref{equ:sweeper_to_sweeper}), which supports our claim at the beginning of this subsection that, a smaller $\Omega$ region generally leads to a more delicate construction of the relative optimality. 
We believe that more efficient criteria exist and will be exploited by the community. 
In particular, a potential direction for further improvements, finding the minimal $\Omega$ region may be taken in the future. 

\subsection{Solution to SPP of the Fig.~\ref{fig:saddle_curve}(c) Case}\label{subsection:fig_3d_solution}
To finalise the discussion of this section, we show step-by-step the solution to the unsolved problem, and how unnecessary path topologies are discarded in Fig.~\ref{fig:saddle_curve}(c). 
See Fig.~\ref{fig:solution_fig3} for illustration. 
For clarity, the path topology that bypasses the internal obstacle from upwards is drawn in red, and the one from downwards is drawn in blue. 
The leaf nodes that have been proven relatively non-optimal will be removed for clarity even if they have been constructed. 
Before the topological tree grows as shown in  Fig.~\ref{fig:solution_fig3}(a), no relative optimality can be constructed. 
See Fig.~\ref{fig:solution_fig3}(b), when $p^4( = c_1^2)$ is expanded, 
\begin{enumerate}
\item $\Delta_3^4$ and $\Delta_2^3$ intersect at $q_1$, where by \textbf{Proposition}~\ref{prop:goal_dependent} we obtain $[c_2^3]\prec[c_3^4]$. So leaf node $c_3^4$ is removed from the priority queue. 
\item $\alpha_2^4$ and $\Delta_2^3$ intersect at $q_2$, where by \textbf{Lemma}~\ref{lem:sweeper_to_path} we obtain $[c_2^3]\prec [c_2^4]$. So leaf node $c_2^4$ is removed from the priority queue.
\end{enumerate}
After $c_1^3$ is expanded as the source point of node $5$, $c_1^3 = p^5$, see Fig.~\ref{fig:solution_fig3}(c), 
\begin{enumerate}
\item $\Delta_1^5$ and $\Delta_1^2$ intersect at $q_3$, where by \textbf{Proposition}~\ref{prop:goal_dependent} we obtain $[c_1^2]\prec [c_1^5]$. So leaf node $c_1^5$ is removed from the priority queue. 
\item $\alpha_2^5$ and $\Delta_1^4$ intersect at $q_4$, where by \textbf{Corollary}	~\ref{coro:sweeper_to_path} we obtain $[p^5]\prec [c_1^4]$. So leaf node $c_1^4$ is removed from the priority queue. 
\end{enumerate}
Note that here all leaf nodes drawn in red have been distinguished to be non-optimal, which physically means that any path that bypasses the internal obstacle from its upwards will not be the globally shortest path. 
This is in contrast to our initial guessing about the resultant solution in Section~\ref{section_problem_modelling}. 
Finally, see Fig.~\ref{fig:solution_fig3}(e), after node $6$-$9$ have been constructed and the critical point $c_1^9$ has been expanded as the $10$-th node, a locally shortest resulting path along the distinguished homotopy $[c_1^9]$ is constructed. 
Then, all the remaining critical points in the priority queue, $c_1^7$, $c_2^7$, and $c_2^8$, have a heuristic cost greater than the length of the currently shortest resulting path. 
Hence the priority queue is wiped out and the algorithm terminates.

\begin{figure}[t]
\includegraphics[width = 0.48\textwidth]{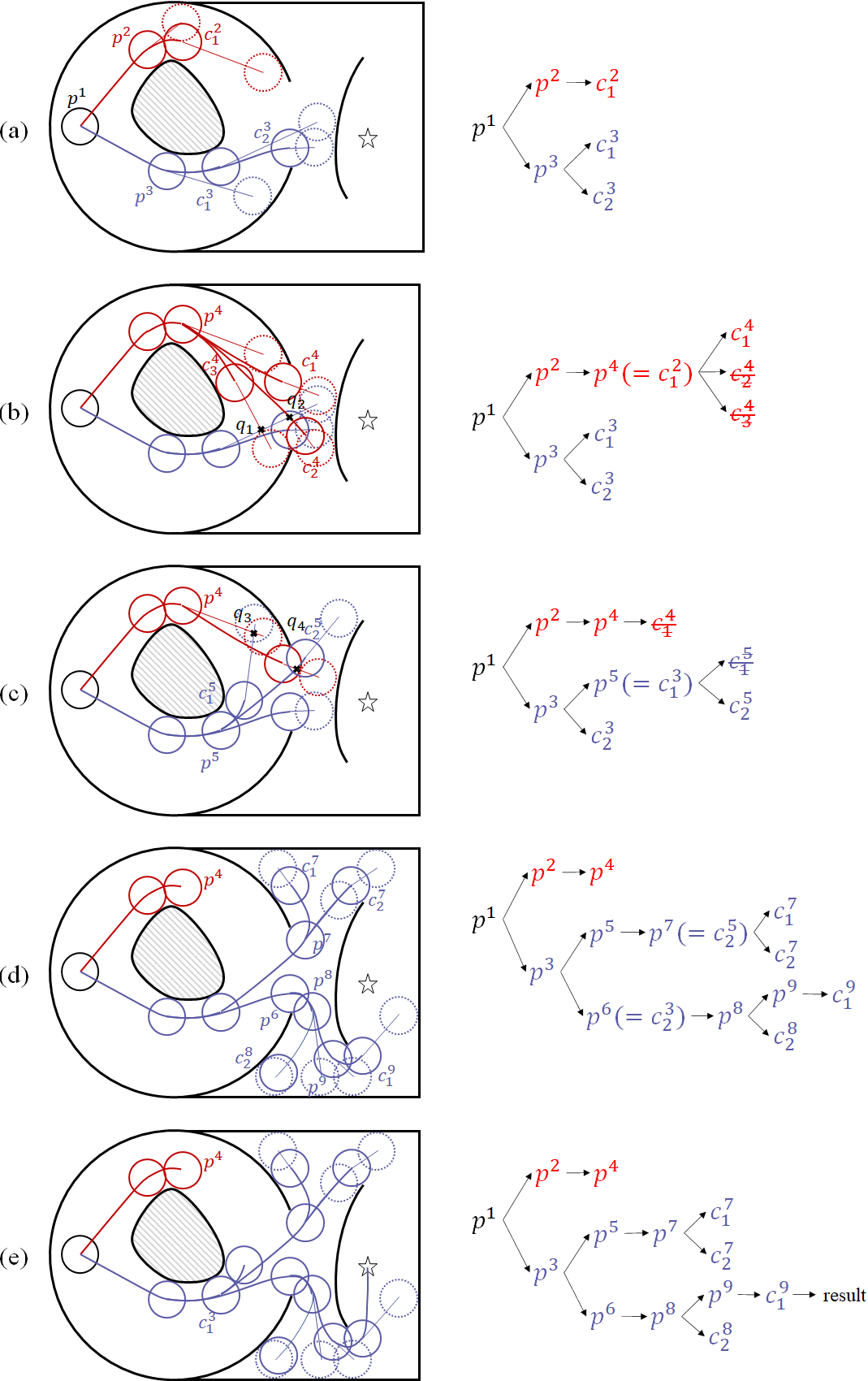}
\caption{Illustration of how pathfinding along unnecessary path homotopies are terminated in the case of Fig.~\ref{fig:saddle_curve}(c). }\label{fig:solution_fig3}
\end{figure}

\begin{figure*}
\begin{minipage}[b]{0.79\linewidth}
\subfigure[]{
\includegraphics[width = 0.17\textwidth]{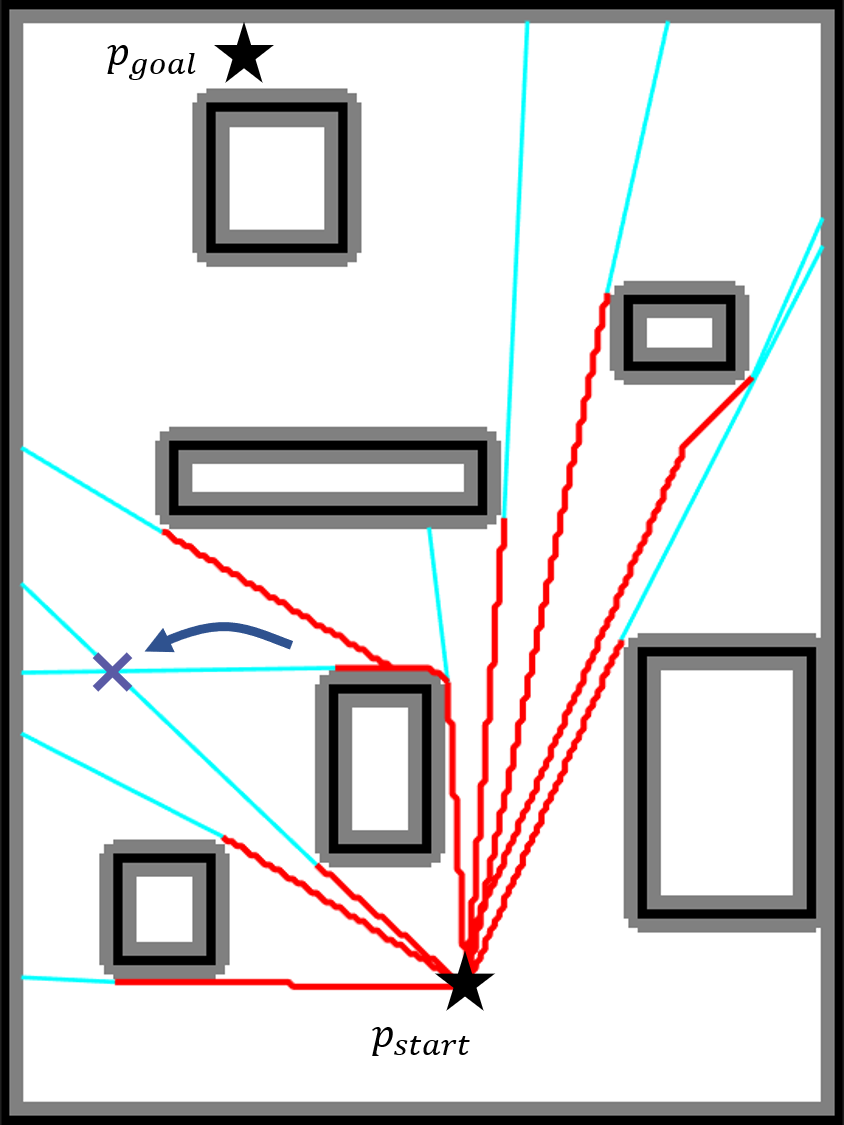}
}
\subfigure[]{
\includegraphics[width = 0.17\textwidth]{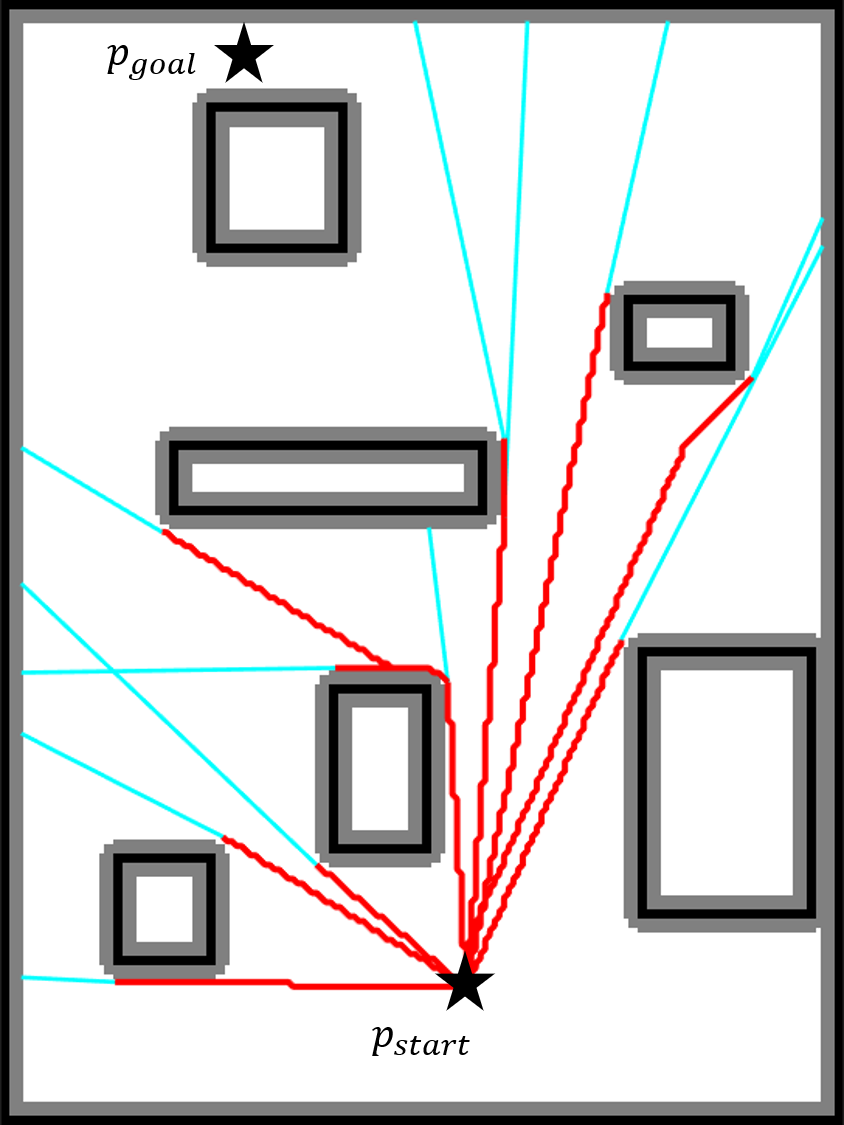}
}
\subfigure[]{
\includegraphics[width = 0.17\textwidth]{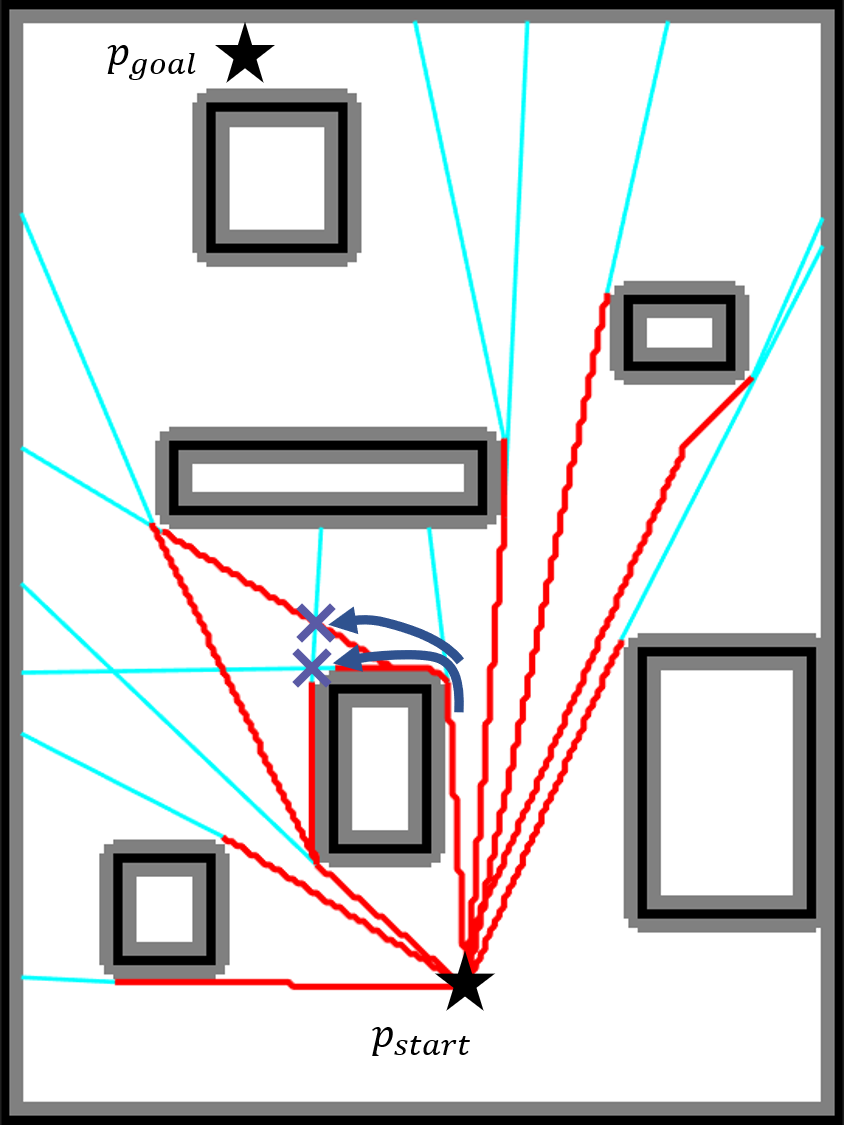}
}
\subfigure[]{
\includegraphics[width = 0.17\textwidth]{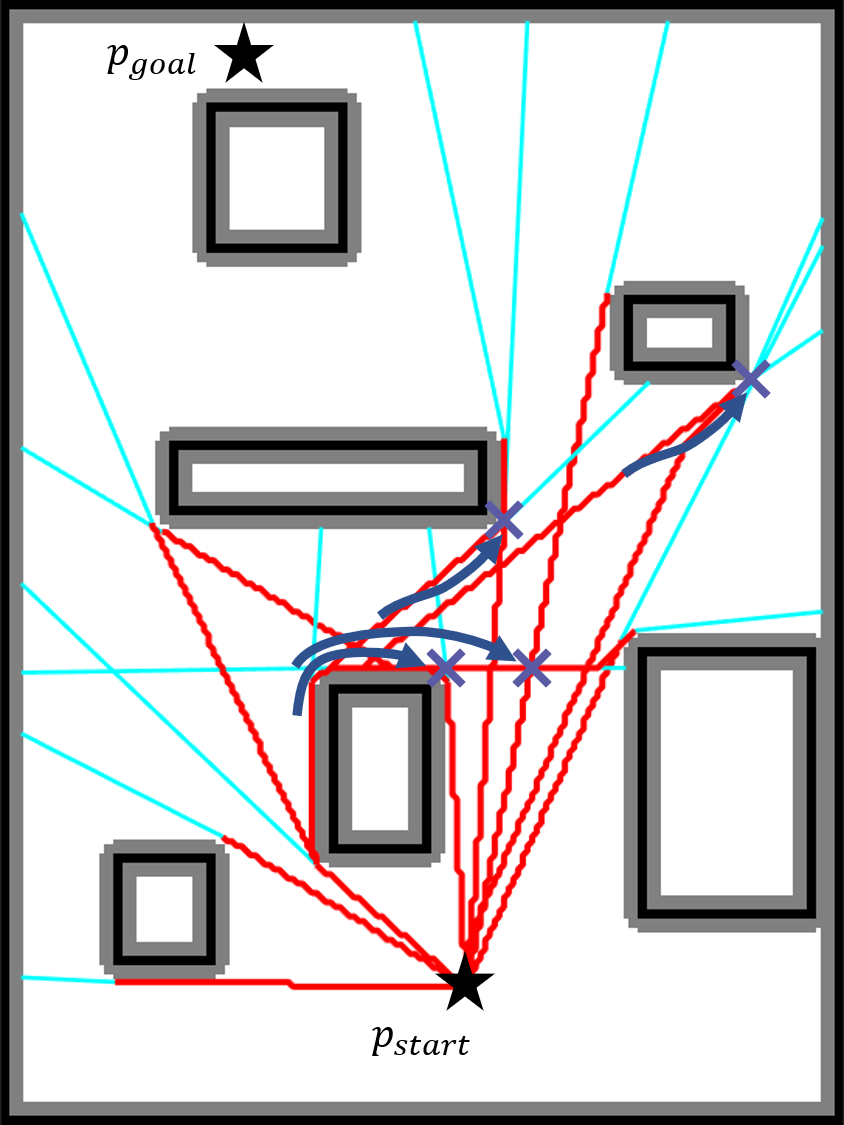}
}
\subfigure[]{
\includegraphics[width = 0.17\textwidth]{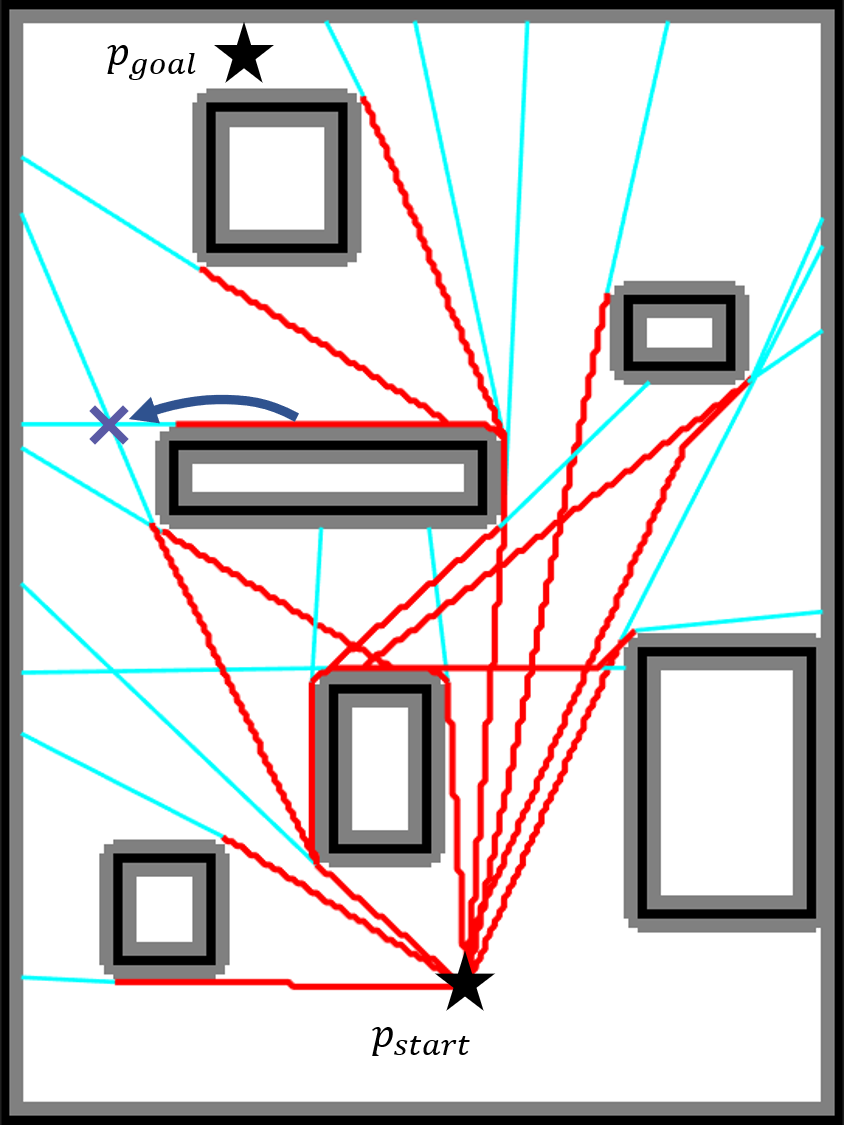}
}\\
\subfigure[]{
\includegraphics[width = 0.17\textwidth]{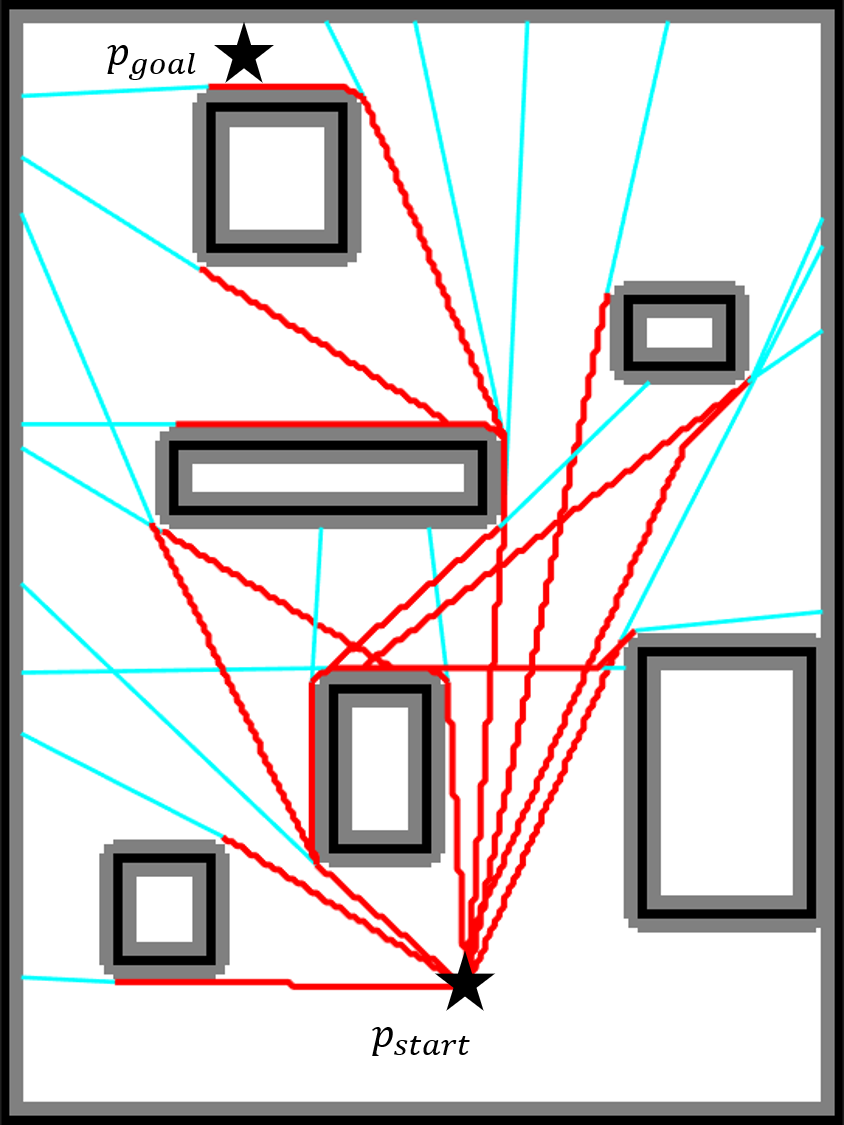}
}
\subfigure[]{
\includegraphics[width = 0.17\textwidth]{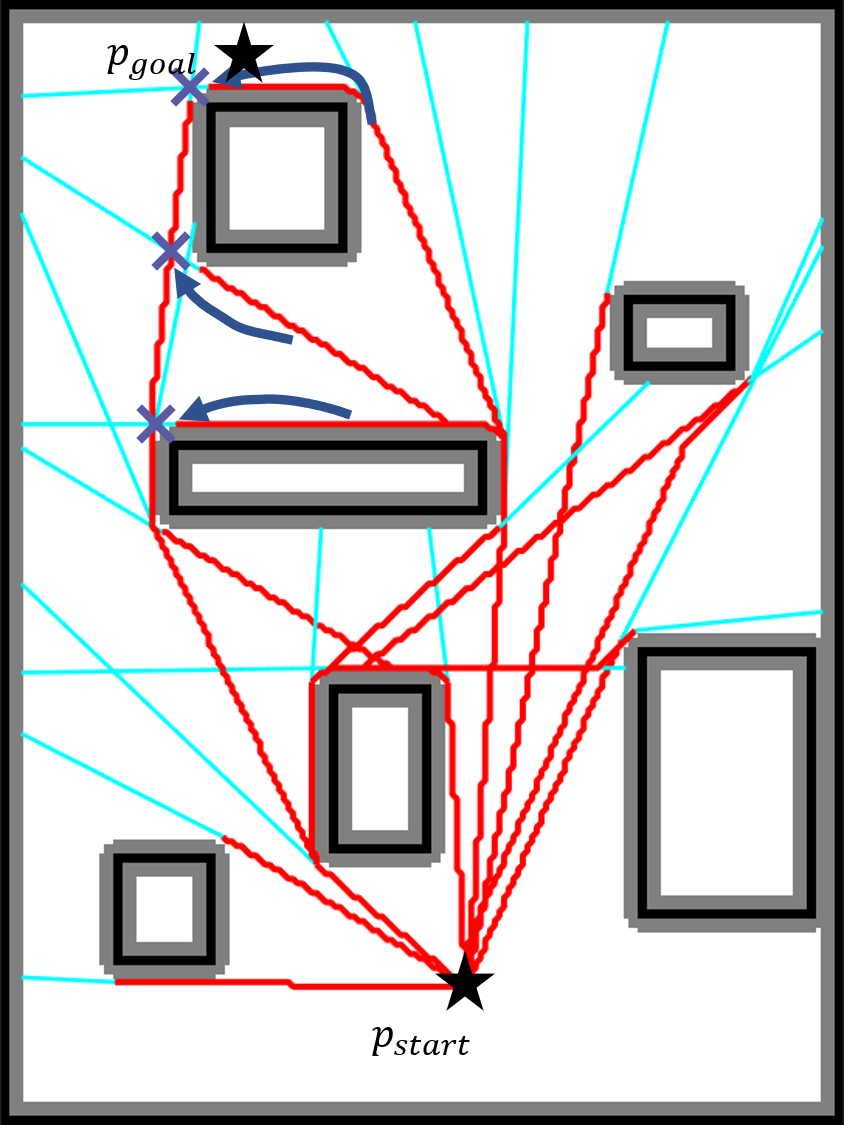}
}
\subfigure[]{
\includegraphics[width = 0.17\textwidth]{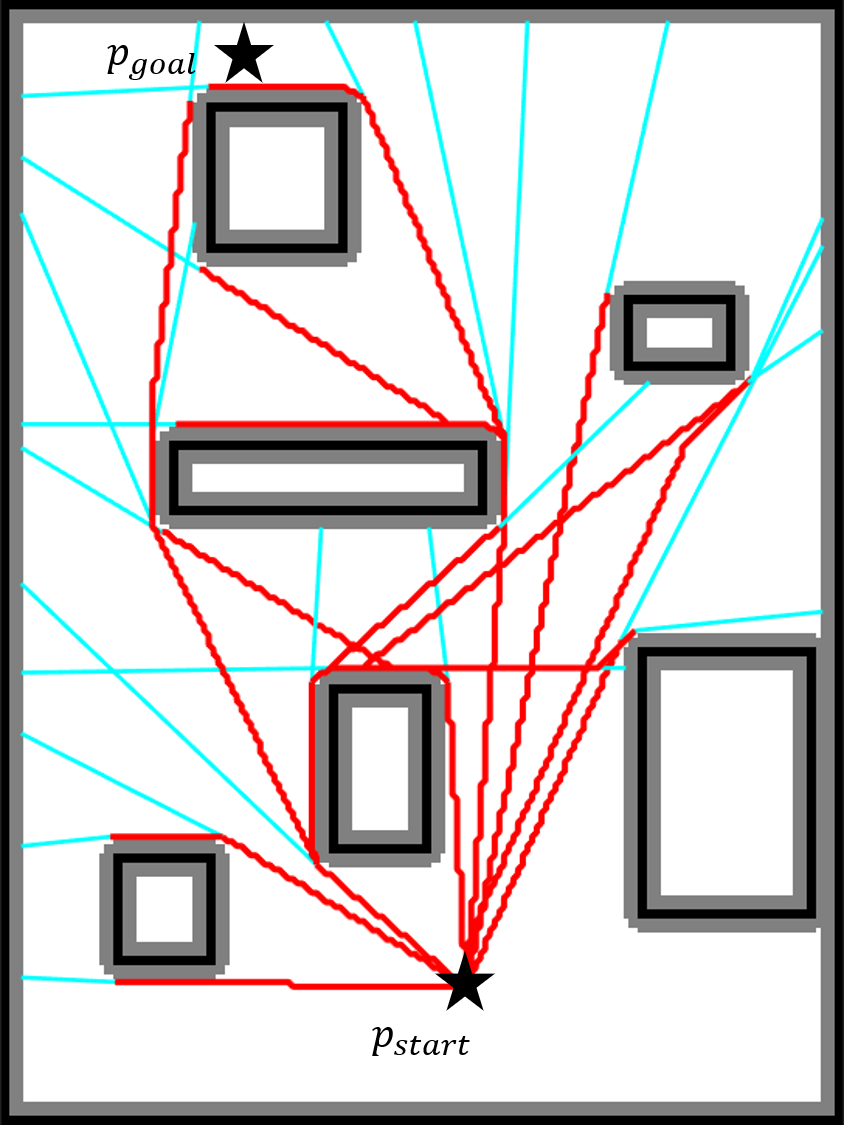}
}
\subfigure[]{
\includegraphics[width = 0.17\textwidth]{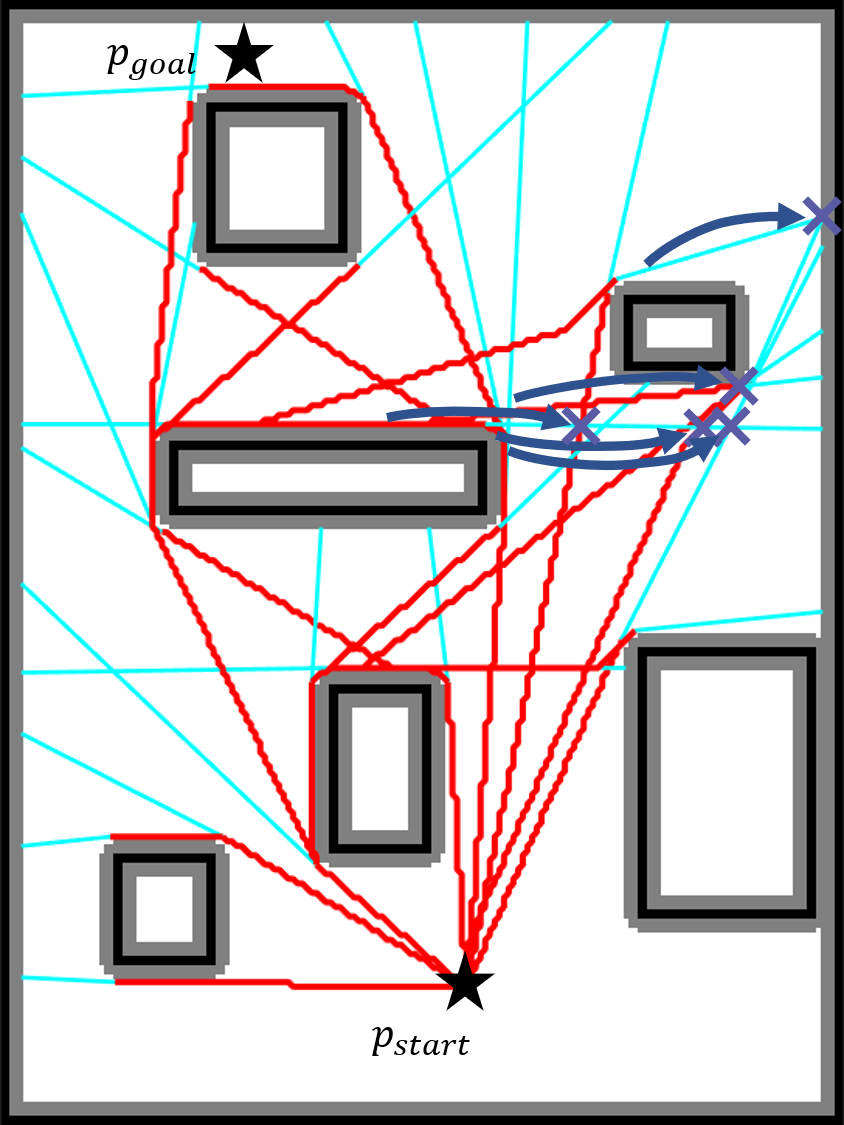}
}
\subfigure[]{
\includegraphics[width = 0.17\textwidth]{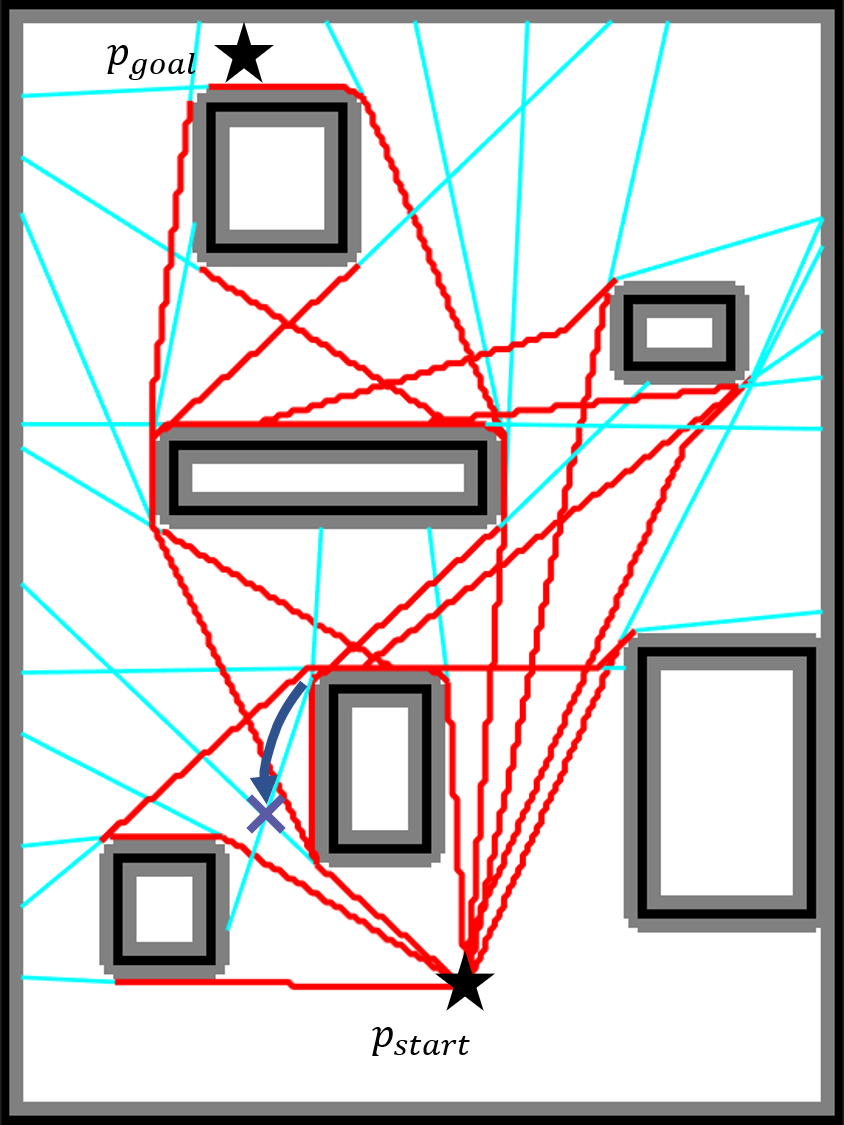}
}\\
\subfigure[]{
\includegraphics[width = 0.17\textwidth]{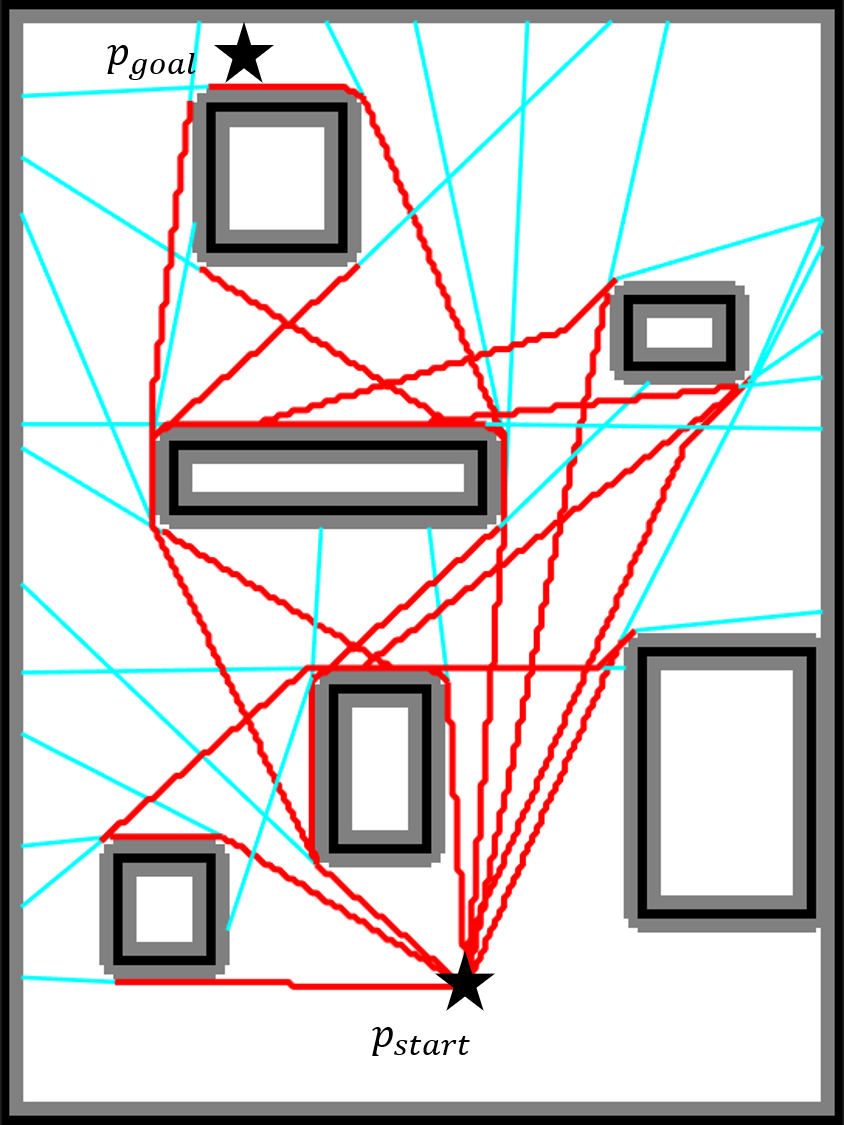}
}
\subfigure[]{
\includegraphics[width = 0.17\textwidth]{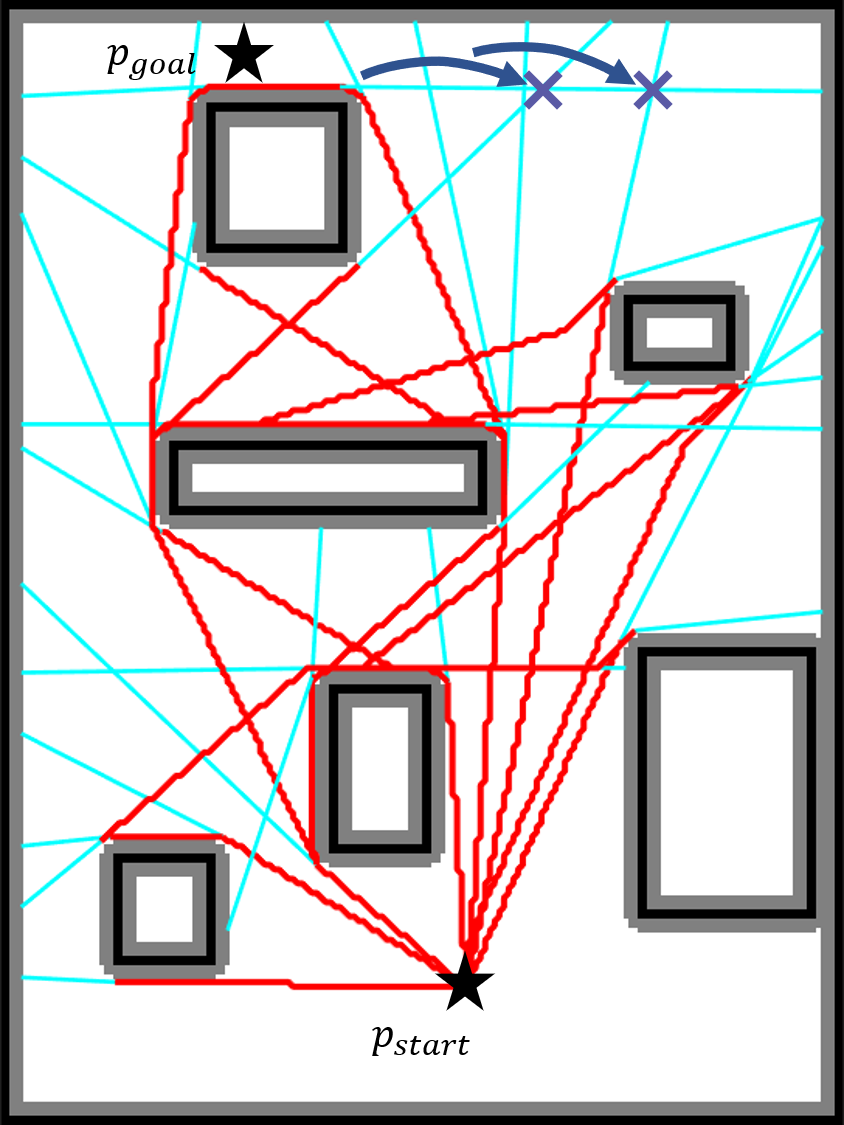}
}
\subfigure[]{
\includegraphics[width = 0.17\textwidth]{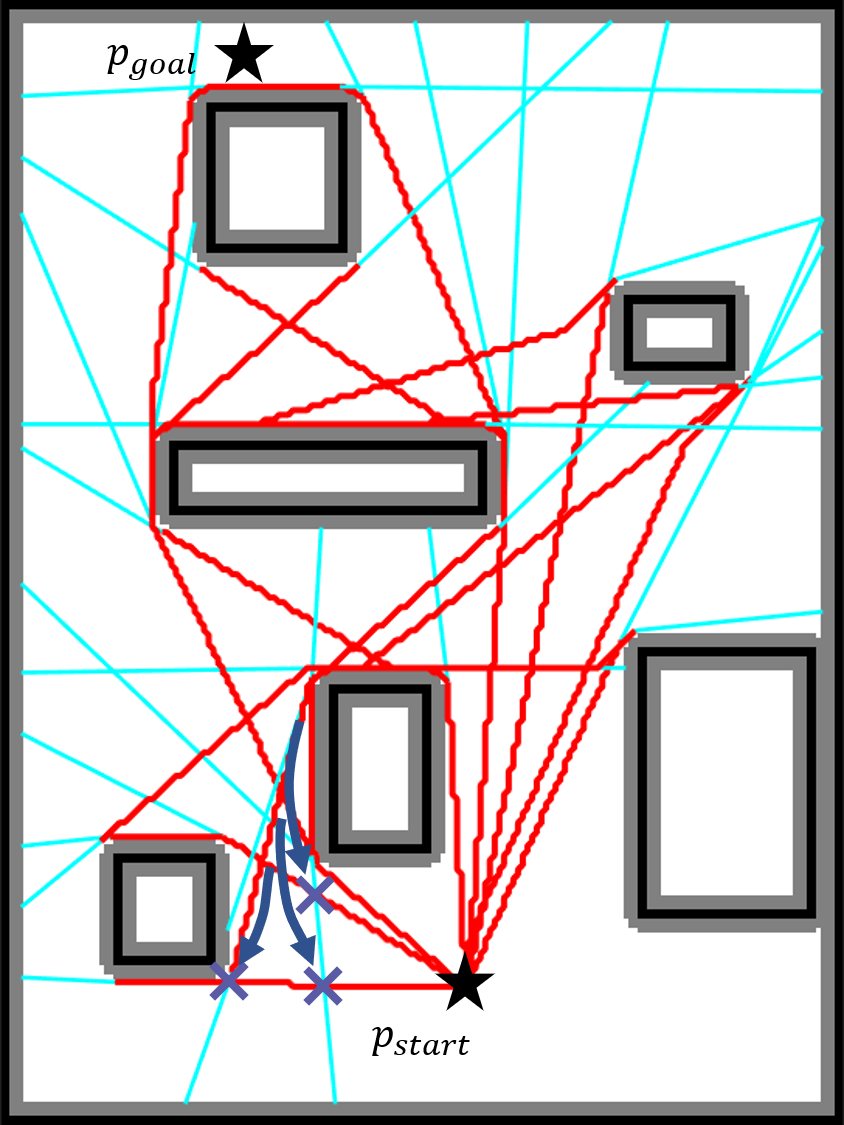}
}
\subfigure[]{
\includegraphics[width = 0.17\textwidth]{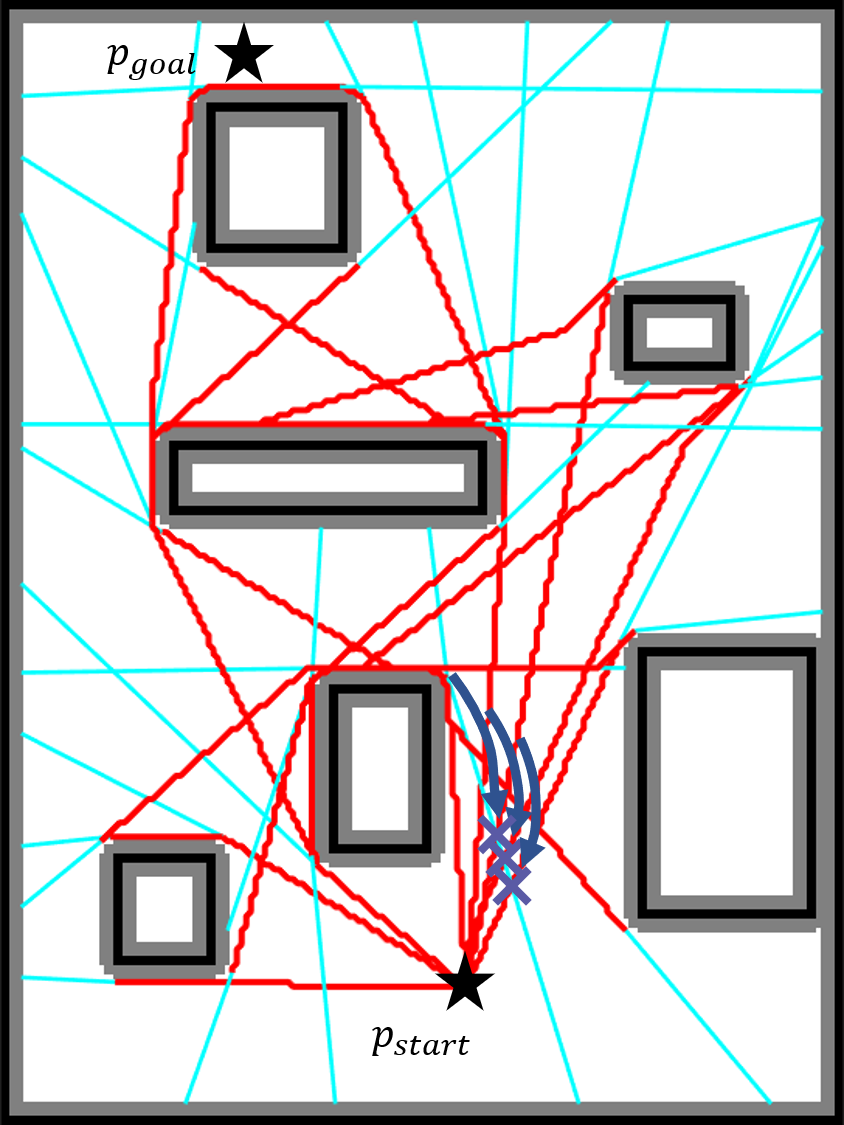}
}
\subfigure[]{
\includegraphics[width = 0.17\textwidth]{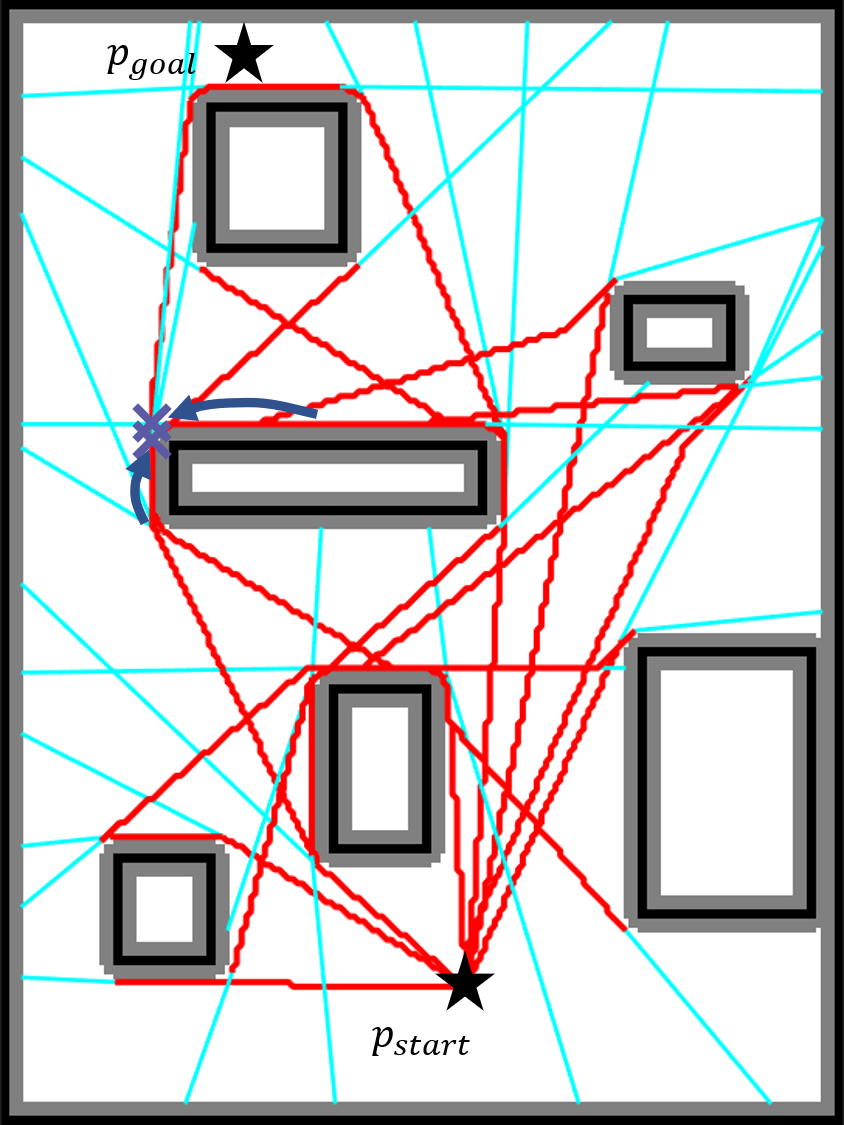}
}
\end{minipage}
\begin{minipage}[b]{0.19\linewidth}
\centering
\subfigure[]{
\includegraphics[width=0.71\textwidth]{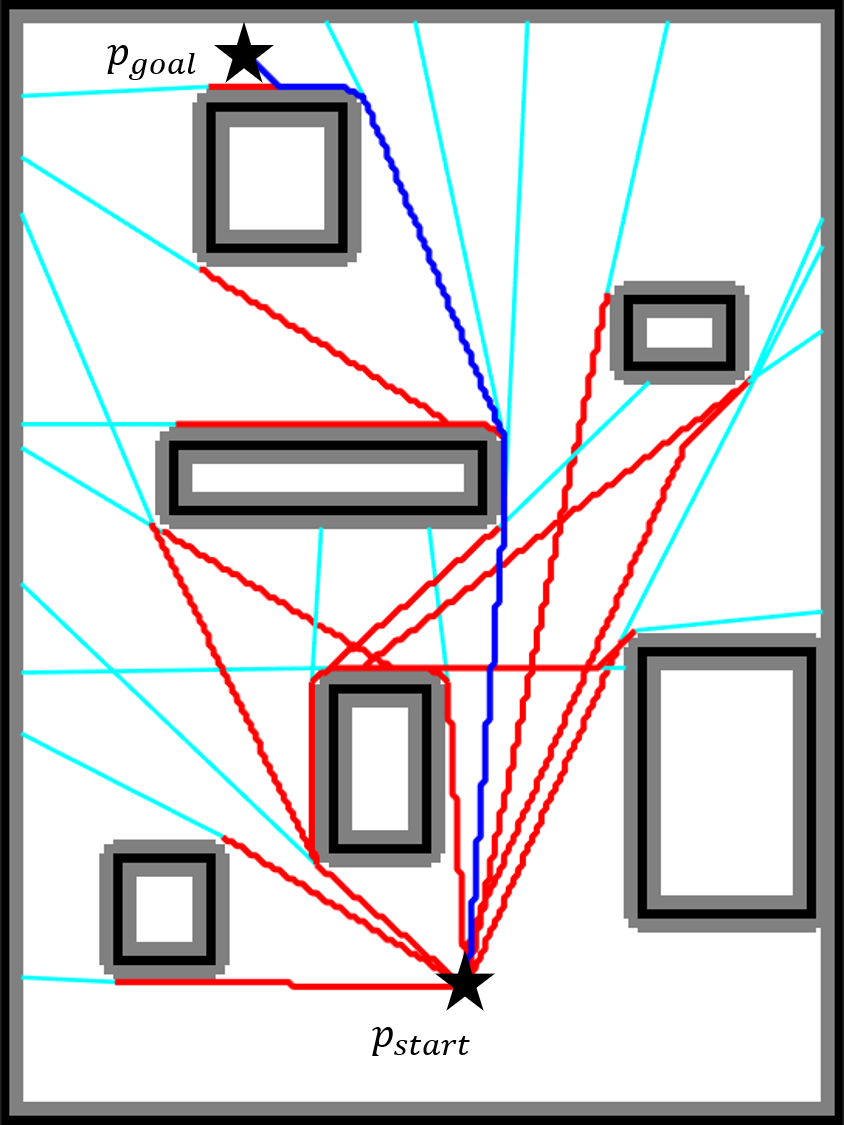}
}\\
\subfigure[]{
\includegraphics[width=0.71\textwidth]{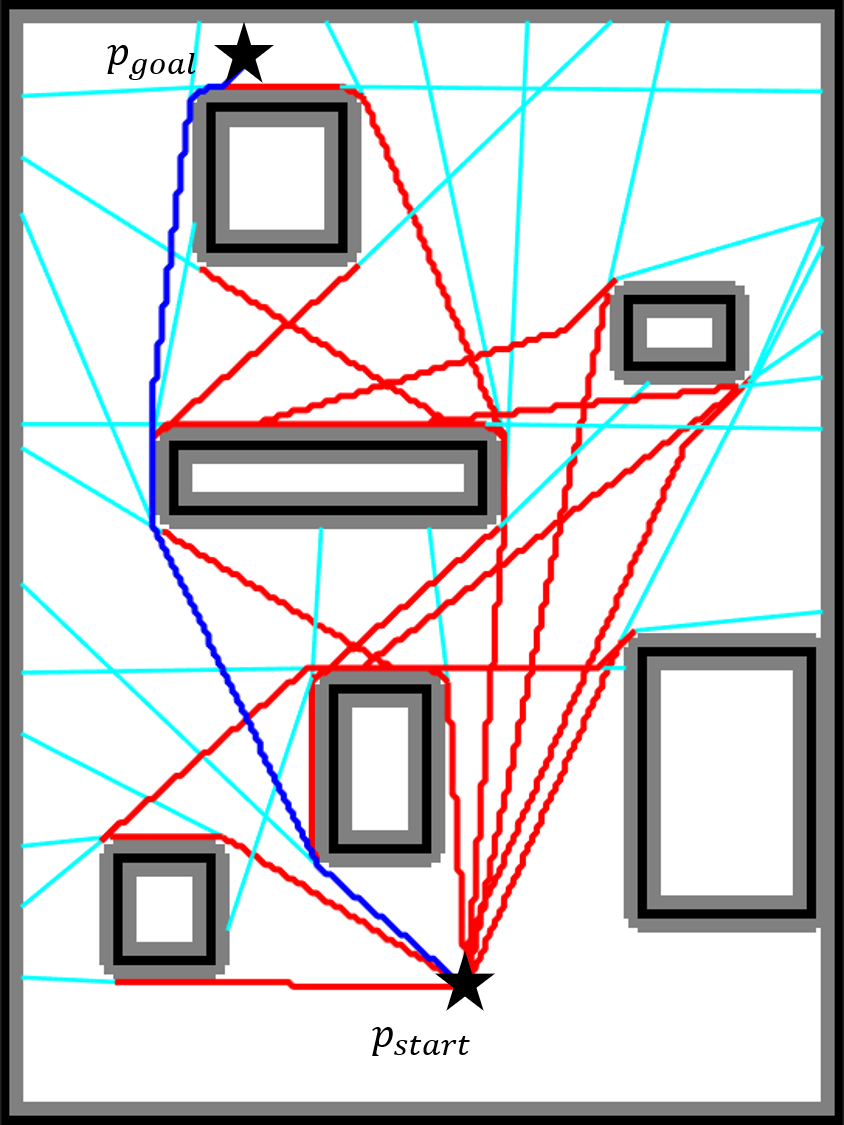}
}\\
\subfigure[]{
\includegraphics[width=0.71\textwidth]{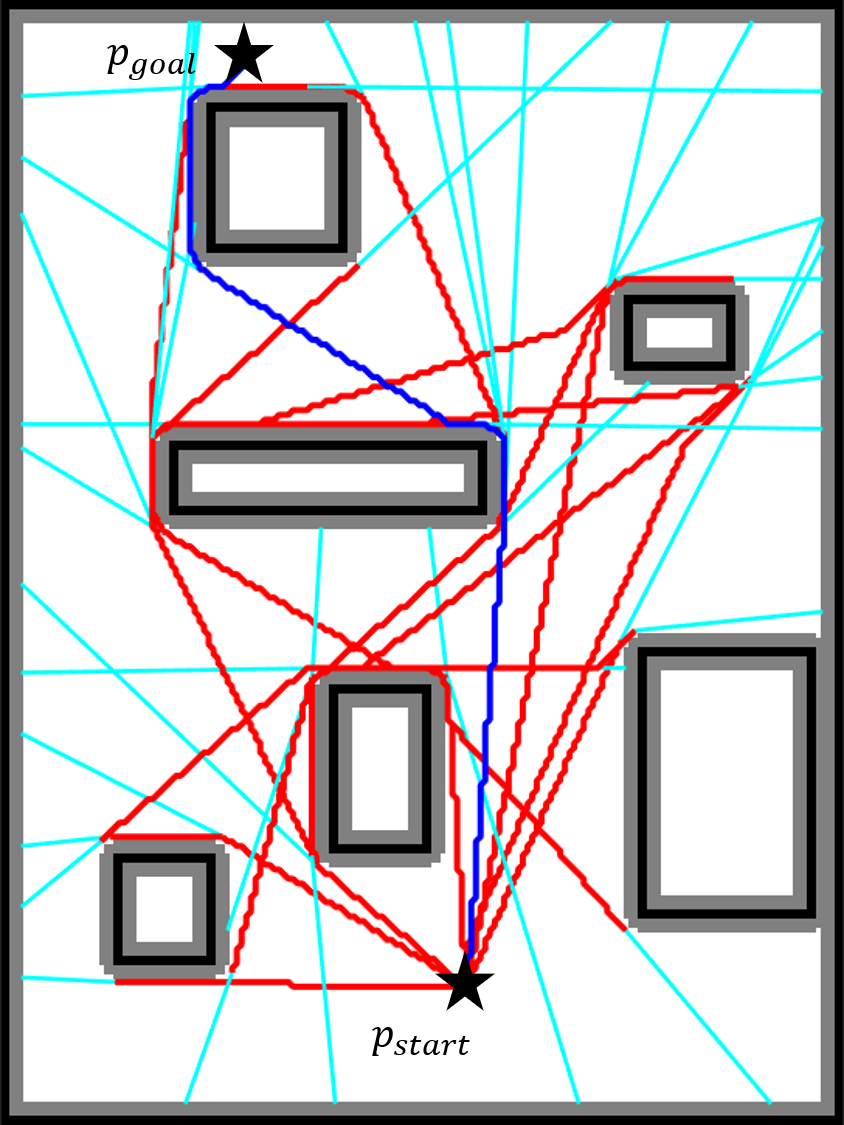}
}
\end{minipage}
\caption{Concrete steps of the proposed hierarchical topological tree and the homotopy simplification mechanism running in a multiply-connected environment for solving a $3$-SNPP problem. 
Black points and grey points are physical obstacles and C-space obstacles. 
The raycasting process is omitted in this figure. 
Red curves are the path segments found in narrow corridors, and cyan straight line segments are the gap sweepers. 
In (a)$\sim$(o), the $\times$ are the intersected points where a partial order relation between homotopy classes of paths is ascertained, and the purple arrows represent the relatively non-optimal path homotopies. 
The result paths of the $3$-SNPP problem are visualised in blue, in (p)(q)(r). }\label{fig:matlab}
\end{figure*}

\begin{figure}[t]
\centering
\subfigure[]{
\includegraphics[width=0.22\textwidth]{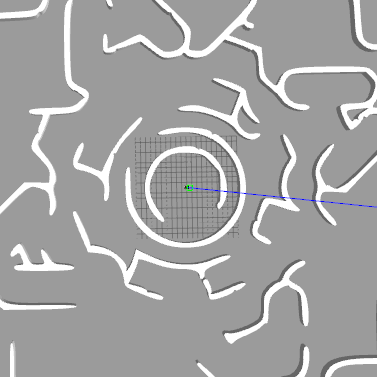}
}
\subfigure[]{
\includegraphics[width=0.22\textwidth]{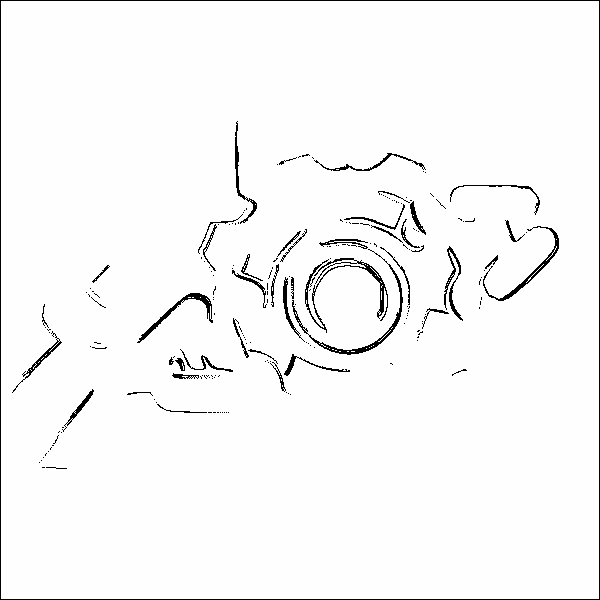}\label{fig:rvrl:map}
}
\caption{Illustration of (a) the simulated environment given by RoboCup 2019 RVRL, (b) the pre-constructed map for tests. 
The environment has been re-scaled to fit into a $600\times600$ grid-map. 
}\label{fig:map_and_robot}
\end{figure}

\begin{figure}[t]
\centering
\subfigure[Ours]{
\includegraphics[width = 0.22\textwidth]{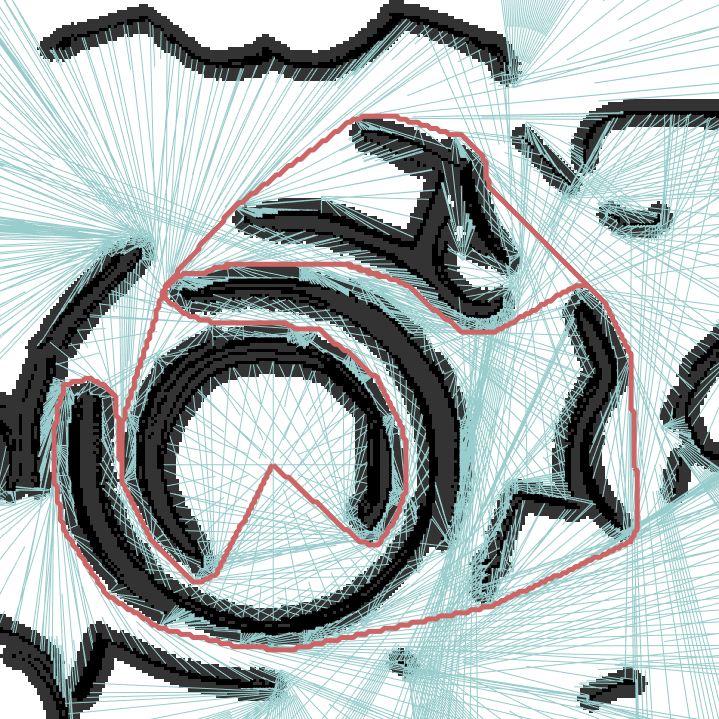}
}
\subfigure[\cite{Bhattacharya2012Topological}]{
\includegraphics[width = 0.22\textwidth]{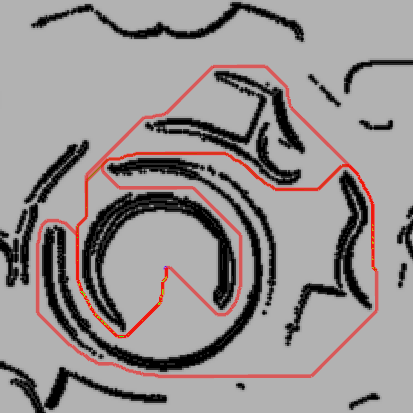}
}
\caption{The solution of a $4$-SNPP demo from $(350, 300)$ to $(467, 299)$. (a) The solution given by the proposed implementation. Only the rays and the resultant paths are depicted. (b) The solution given by~\cite{Bhattacharya2012Topological}.
Both solutions are correct $k$-SNPP solutions since the distance optimality is measured in grid-map. 
}\label{fig:compare}
\end{figure}

\section{Experimental Results}\label{section_experiment}
The proposed $k$-SNPP algorithm is a hierarchical topological tree equipped with an efficient topology simplification mechanism to eliminate the pathfinding along non-$k$-optimal path homotopies in a multiply-connected environment from a distance-based perspective. 
To the best of the author's knowledge, there does not exist such a topology simplification mechanism before. 
So in the first experiment in Section~\ref{subsection_case_study} the whole construction of the hierarchical topological tree is decomposed into concrete steps. 
The second experiment in Section~\ref{subsection_ros} compares the proposed algorithm to existing works~\cite{Bhattacharya2012Topological} for solving $k$-SNPP. 
The experiments show that the proposed algorithm has significantly simplified the complexity of $k$-SNPP path searching to transform it into a computationally affordable task.

\begin{figure}[t]
\centering
\subfigure[]{
\includegraphics[width=0.22\textwidth]{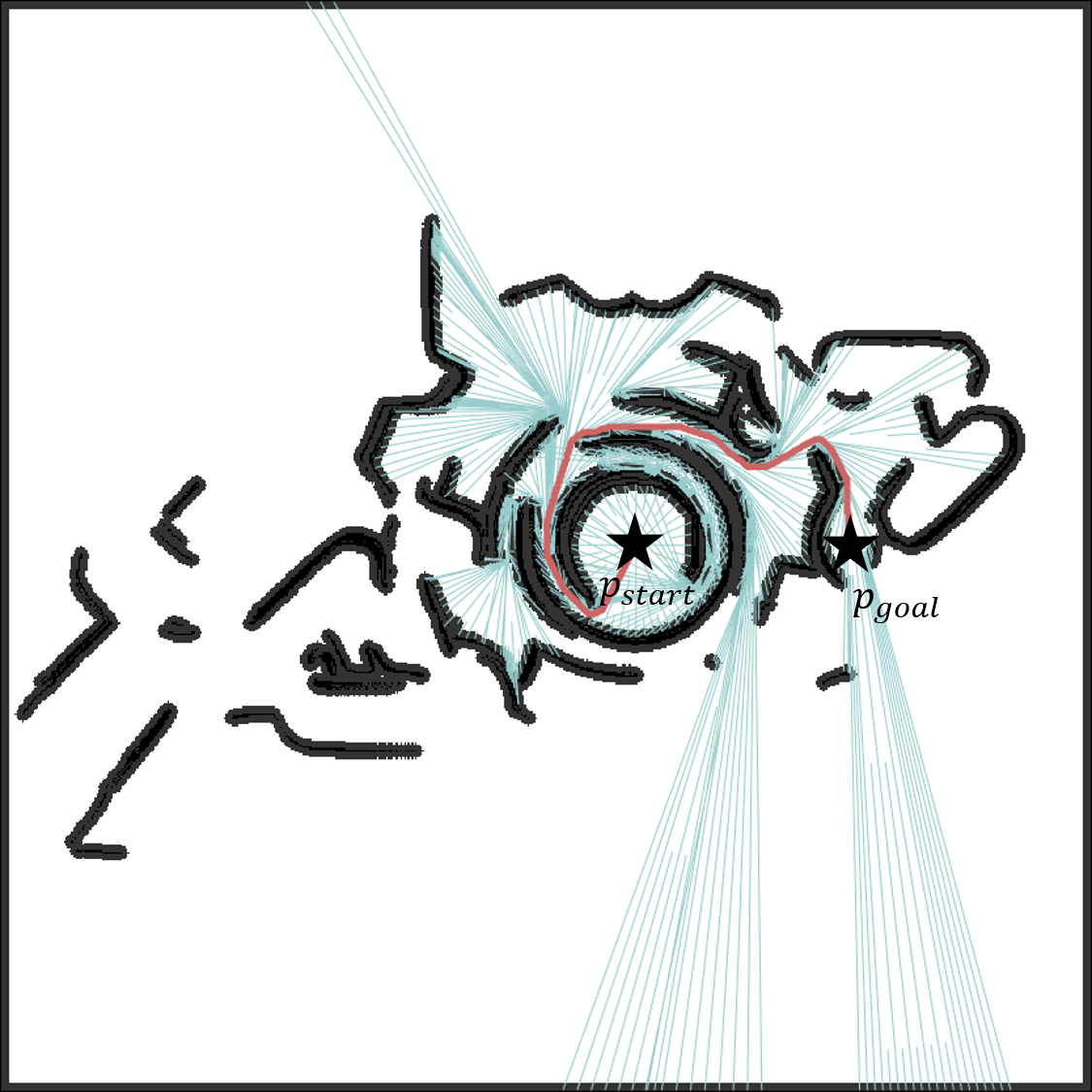}
}
\subfigure[]{
\includegraphics[width=0.22\textwidth]{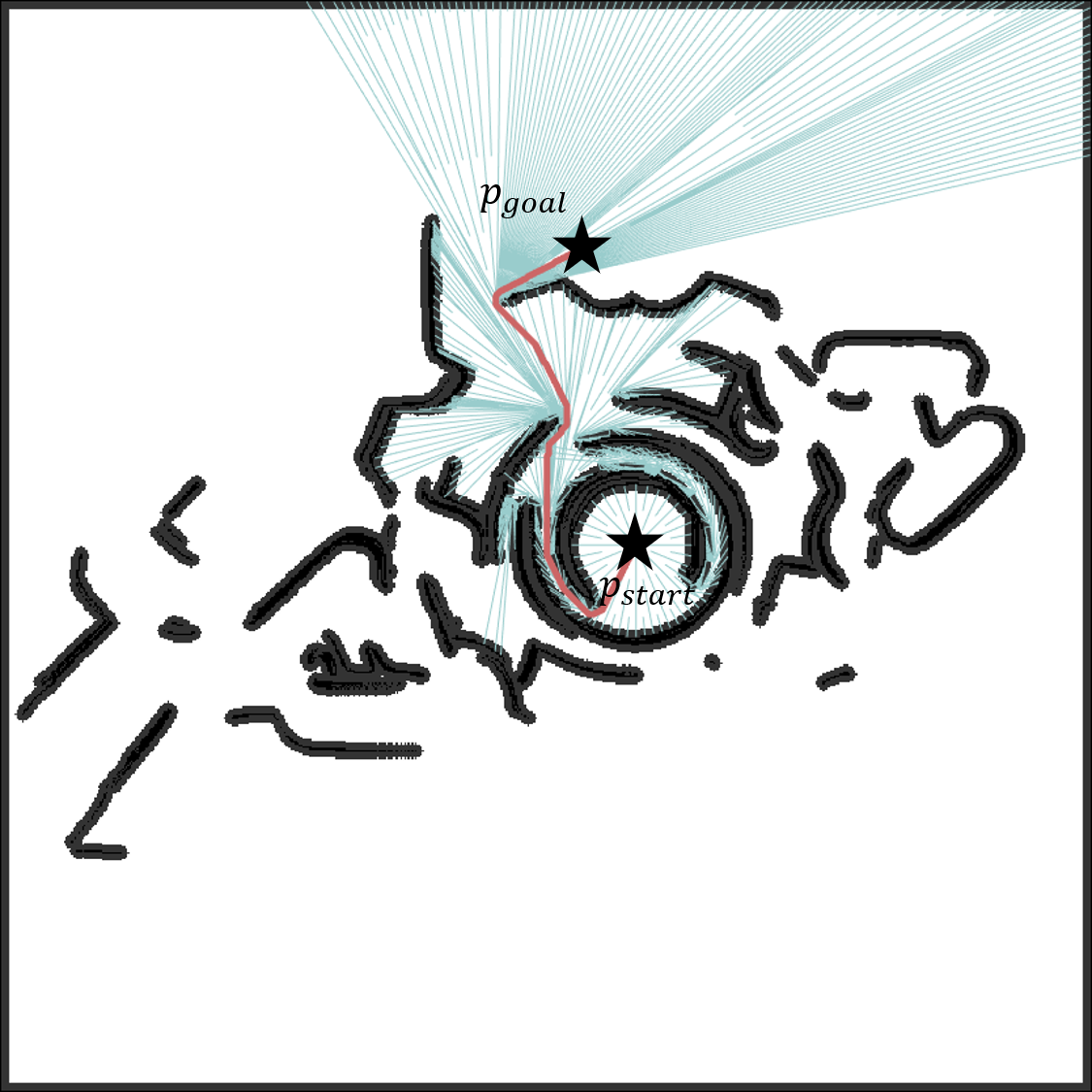}
}
\subfigure[]{
\includegraphics[width=0.22\textwidth]{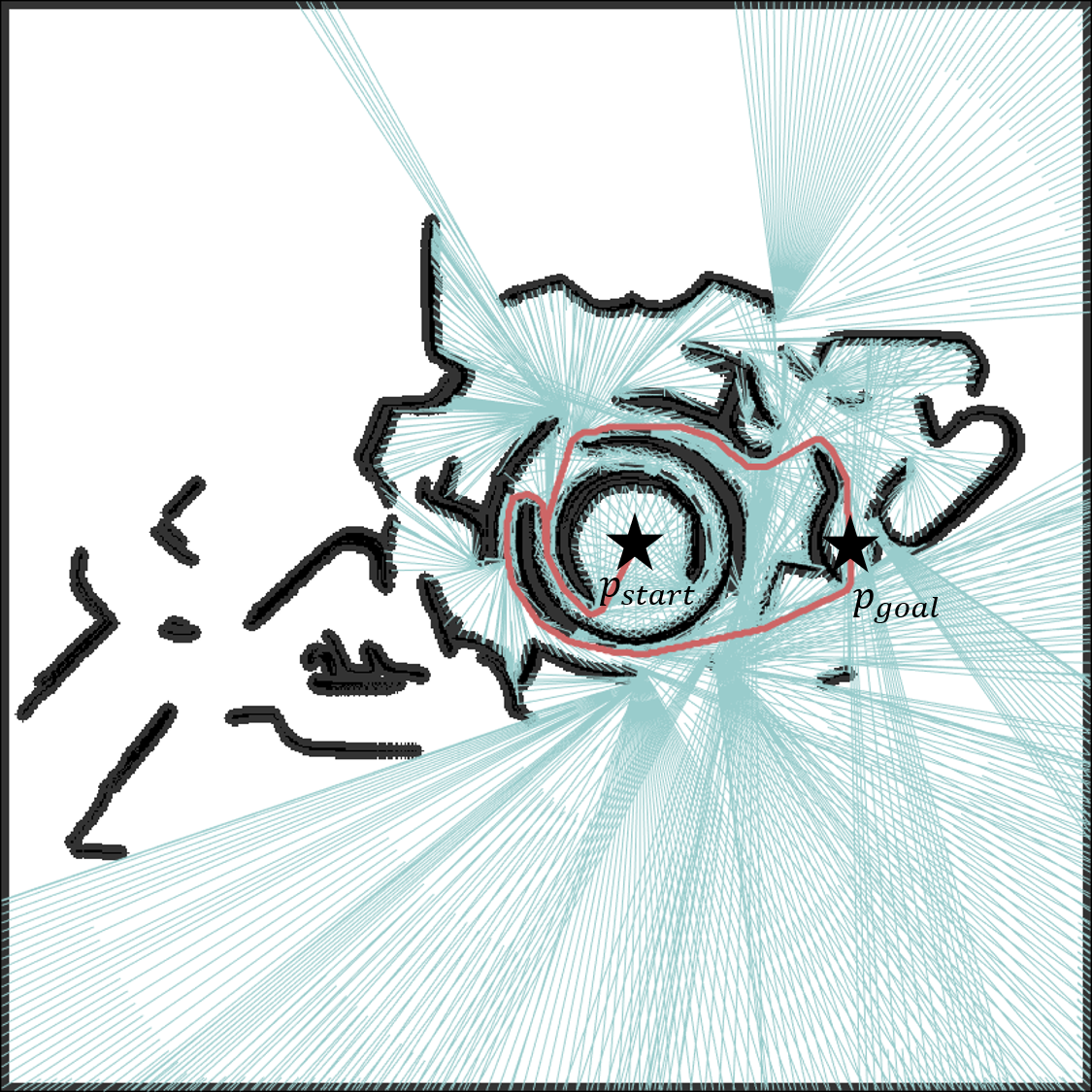}
}
\subfigure[]{
\includegraphics[width=0.22\textwidth]{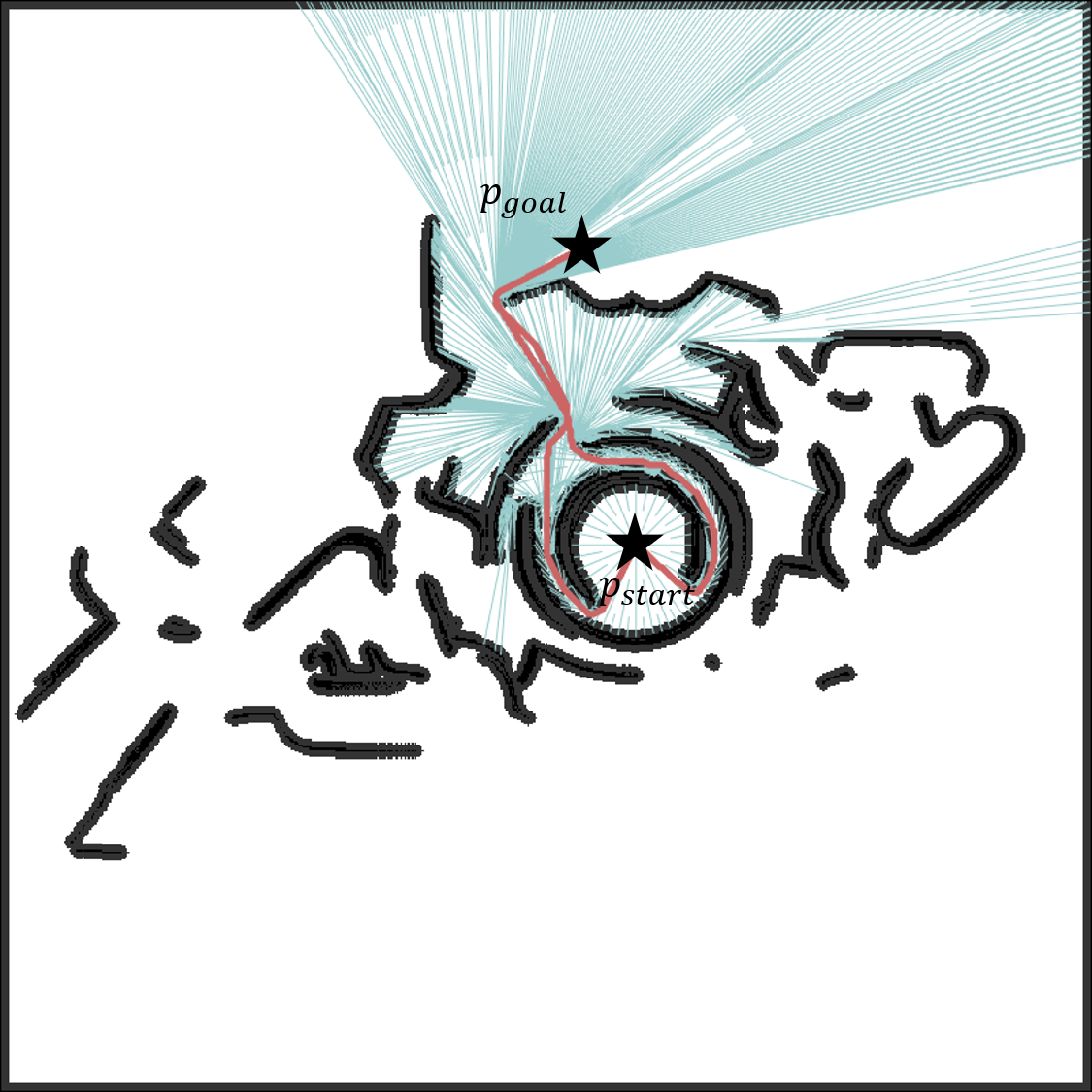}
}
\subfigure[]{
\includegraphics[width=0.22\textwidth]{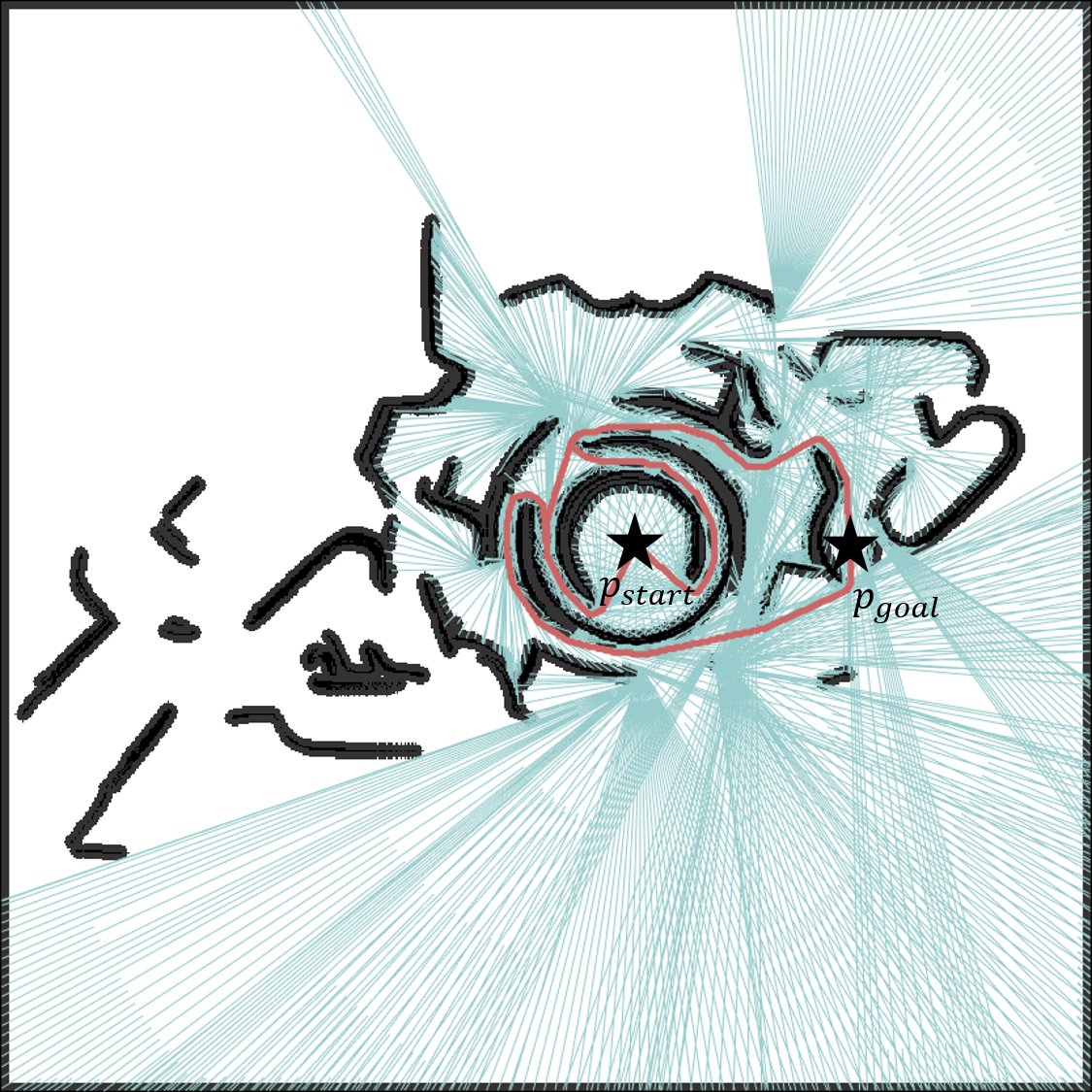}
}
\subfigure[]{
\includegraphics[width=0.22\textwidth]{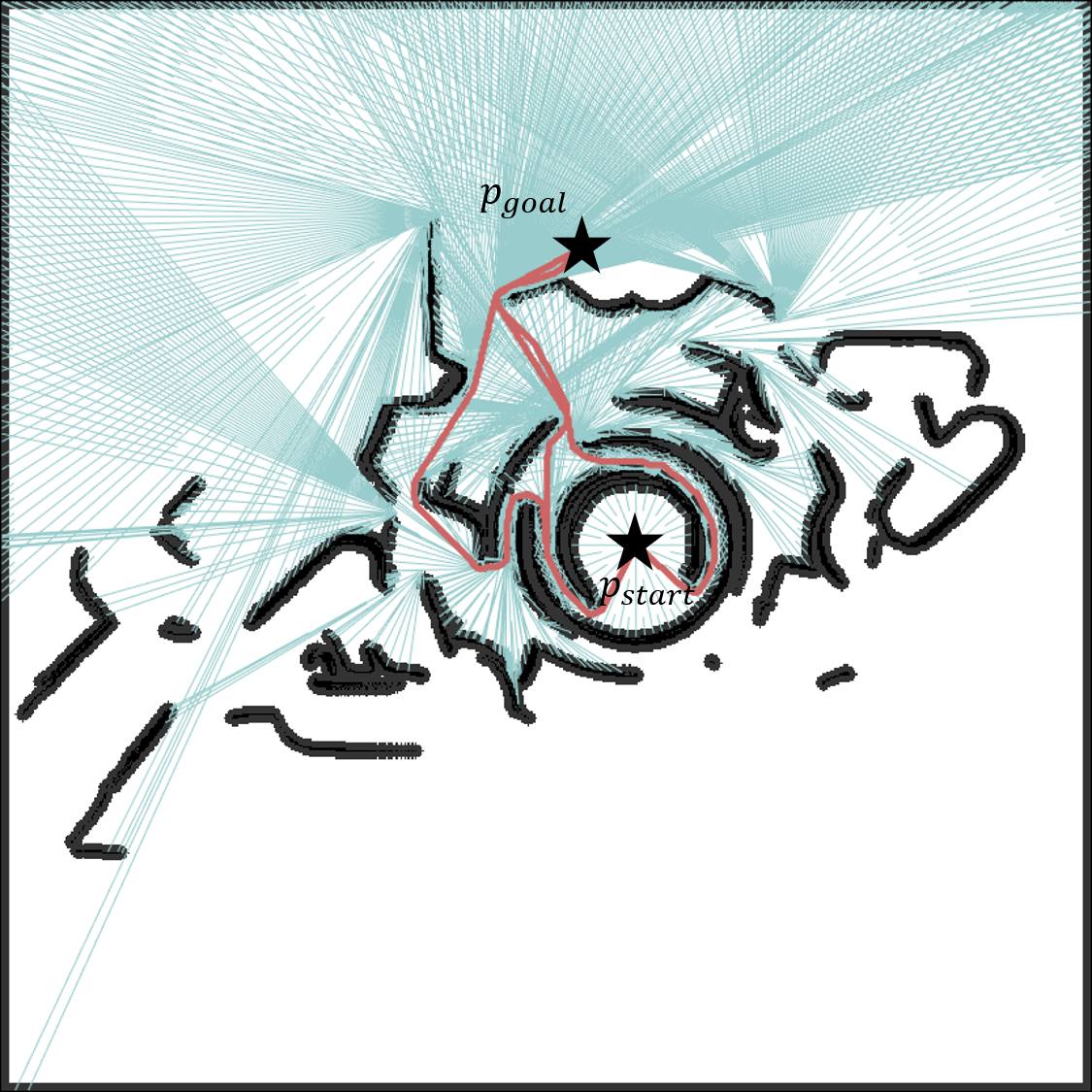}
}
\caption{
Illustration of the algorithm execution. 
Left figures are the (a) $1$-SNPP, (c) $2$-SNPP, and (e) $3$-SNPP solution from $(350, 300)$ to $(467, 299)$. 
Right figures are the (b) $1$-SNPP, (d) $2$-SNPP, and (f) $3$-SNPP solution from $(350, 300)$ to $(320, 463)$. 
The obstacles have been inflated based on the robot's radius, shown in black. 
The rays are visualised in grey, with the resulting paths being shown in red. 
}\label{fig:617_499}\label{fig:470_663}
\end{figure}

\subsection{Case Study}\label{subsection_case_study}
In this experiment, we illustrate the iterative expansion of the proposed hierarchical topological tree by a case study of solving $3$-SNPP implemented in MATLAB. 
See Fig.~\ref{fig:matlab} for illustration. 
The robot's initial location is at the bottom-central of the map, and the goal is set at the top-central. 

Fig.~\ref{fig:matlab}(a)$\sim$(o) are screenshots after node $1\sim 15$ are constructed. 
Black and grey grids represent the physical obstacles and C-space obstacles respectively. 
Tree edges are drawn in red curves, and gap sweepers are drawn in cyan. 
The topological structure intersections that indicate valid relative optimality between distinguished homotopies are marked by $\times$. 
The relatively non-optimal distinguished homotopy will be marked by blue arrows. 
Fig.~\ref{fig:matlab}(p)$\sim$(r) are screenshots when the $1$-st, $2$-nd, and $3$-rd shortest non-homotopic paths are obtained. 

It can be seen that the proposed algorithm can figure out relatively non-optimal topologies even at a very early stage of the planning process. 
Most notably, see Fig.~\ref{fig:matlab}(a), where only two nodes have been constructed, two gap sweepers intersect and a child of the second node is inspected to be relatively non-optimal, and is thus paused because the algorithm is currently looking for the globally shortest path. 
This makes the $1$-st shortest path obtained with very high efficiency (after only $6$ nodes are expanded). 
For the subsequent processes, since the shape of nodes is independent of the sequencing order of nodes, we only need to continue expanding the leaf nodes, no re-expansion is required. 
Finally, the $3$-SNPP problem is solved after $19$ nodes are expanded, as shown in Fig.~\ref{fig:matlab}(r). 

\subsection{Computational Time}\label{subsection_ros}
The algorithm is also implemented in C++~\footnote{Experiments are run on a computer with I7-8700 CPU and 32GB RAM}. 
The radius of the robot is set as $0.8m$. 
The map for testing is a challenging environment given by the RoboCup 2019 Rescue Virtual Robot League (RVRL)~\footnote{\url{https://github.com/RoboCupRescueVirtualRobotLeague/RoboCup2019RVRL\_Demo/wiki}}, as shown in Fig.~\ref{fig:map_and_robot}. 
The size of the map is $120m$, and the grid resolution is $0.2m$. 
So there are $600\times 600$ grids in the map. 
With the existence of concave obstacles, narrow passages, and large obstacle-free regions, the evaluation is persuasive. 
Given the fact that the inflation of obstacles has been effectively implemented and widely utilised in the cost-map module, the collision-checking process in all algorithms is simplified to just reading the value from the cost-map. 
The robot's location is set at the centre, $(350, 300)$ (grid index). 
We compare the computational time between the algorithm proposed in~\cite{Bhattacharya2012Topological} and ours. 
The locally optimal paths of a single topology in a grid-map are not unique, so the solutions of the two algorithms are not precisely the same, as shown in Fig.~\ref{fig:compare}, where however both of them are correct $k$-SNPP solutions. 
The screenshots of the proposed algorithm finding $1$-SNPP, $2$-SNPP and $3$-SNPP towards $(467, 299)$ and $(320, 463)$ have been visualised in Fig.~\ref{fig:617_499}. 
The statistics of the computational time have been summarised in Table.~\ref{tab:time}. 
Note that the time is not incrementally recorded. 
For example, the time for finding two and three shortest non-homotopic paths from $(350, 300)$ to $(445, 428)$ are 211.70ms and 302.02ms, respectively, which means that only 88.19ms has been addictively used to find the third shortest path that is non-homotopic to the shortest two paths. 
Statistics in Table.~\ref{tab:time} show that the proposed mechanism reduces the computational time of solving the $k$-SNPP problem with a proportion of $\geq$93\%. 
This eventually makes $k$-SNPP a computationally reasonable task. 

\begin{table*}[t]\footnotesize
\centering
\caption{Computational Time (Average of 10 Trials)}\label{tab:time}
\begin{tabular}{|c|c|c|c|c|c|c|c|c|c|c|c|c|}
\hline
\multirow{2}{*}{} & \multicolumn{3}{c|}{$(445, 428)$} & \multicolumn{3}{c|}{$(320, 463)$} & \multicolumn{3}{c|}{$(467, 299)$} & \multicolumn{3}{c|}{$(111, 206)$}\\
\cline{2-13}
& Ours & Baseline$^1$ & \%$^2$ & Ours & Baseline$^1$ & \%$^2$ & Ours & Baseline$^1$ & \%$^2$ & Ours & Baseline$^1$ & \%$^2$ \\
\hline
\hline
$1$-SNPP & \textbf{211.70ms} & 8.31s& 2.55\% & \textbf{189.29ms} & 7.17s& 2.64\% & \textbf{272.03ms} & 27.98s& 0.97\% & \textbf{82.38ms} & 27.11s& 0.3\%\\
\hline
$2$-SNPP & \textbf{302.02ms} & 12.45s& 2.43\% & \textbf{773.93ms} & 11.27s& 6.87\% & \textbf{640.77ms} & 53.11s& 1.21\% & \textbf{812.74ms} & 98.44s& 0.83\%\\
\hline
$3$-SNPP & \textbf{320.21ms} & 24.43s& 1.31\% & \textbf{983.92ms} & 134.13s& 0.73\% & \textbf{656.06ms} & 90.00s& 0.73\% & \textbf{866.45ms} & 99.38s& 0.87\%\\
\hline 
$4$-SNPP & \textbf{335.77ms} & 26.18s& 1.28\% & \textbf{998.81ms} & 153.85s& 0.65\% & \textbf{835.78ms} & 123.18s& 0.68\% & \textbf{877.72ms} & 223.93s& 0.39\%\\
\hline
\end{tabular}
\begin{tablenotes}
\item $^1$ Baseline: \cite{Bhattacharya2012Topological}
\item $^2$ \%: The proportion of computational time, Ours /\cite{Bhattacharya2012Topological}$\times 100$\%
\end{tablenotes}
\end{table*}

\section{Related Works}~\label{section_relatedworks}
\subsection{Shortest Path Planning}
The years have witnessed many contributions to the shortest path planning (SPP) tasks in the mobile robot planning area~\cite{Lavalle2006Planning}. 
Early works such as the Visibility graph~\cite{Lozano1979Algorithm} and the Tangent graph~\cite{Liu1992Path} are theoretically sound for shortest path planning in structural environments. 
The Visibility graph~\cite{Lozano1979Algorithm} claimed that locally shortest paths for particle robots always consist of the straight paths connecting the vertices of polygonal obstacles. 
The Tangent graph~\cite{Liu1992Path} generalised the locally shortest path to the ``tangent" of polygonal or curved obstacles. 
Besides the polygonal and the curved map, grid-based maps were also adopted which establish a more flexible representation of the obstacle. 
Regarding the grids in the map as nodes, graph searching-based algorithms were developed for optimal planning, such as the Dijkstra algorithm~\cite{Dijkstra1959Note} and the A* algorithm~\cite{hart1968formal}. 
The SPP problem was also studied by sampling-based approaches~\cite{Kavraki1996Probabilistic}, which do not need a pre-discretisation of the environment, and thus are more flexible.
However, the random selection of collision-free positions as waypoints with distance-based connectivities between waypoints (PRM~\cite{Kavraki1996Probabilistic}, RRT~\cite{Lavalle2001Randomized}) has been shown with zero probability to be optimal~\cite{Karaman2011Sampling}. 
The milestone of the probabilistically optimal sampling-based planners, the RRT* algorithm~\cite{Karaman2011Sampling}, together with its subsequent progresses~\cite{Naderi2015rtrrtstar}~\cite{otte2016rrtx}, proposed a novel mechanism to rewire the node to the best parent in a small neighbourhood if lower movement cost can be found. 

It is observed that a \textit{shortcut mechanism}~\cite{Efrat2006Computing} is embedded in all existing SPP solutions, which is a process of using a new, shorter path to replace an old path with the same endpoints. 
Taking the A* algorithm as an example, once a lower-cost path towards an already-expanded node is constructed, the parent of the node is switched to the lowest-cost neighbour to ensure optimality, which is essentially a shortcut mechanism. 
Also, in sampling-based algorithms such as the RRT*~\cite{Karaman2011Sampling}, the key process to acquiring a distance-optimal path to a node is to look for its best parent node (within a neighbourhood, for an algorithmic complexity consideration). 
Although the graph searching-based planners~\cite{stentz1995focussed}~\cite{Koenig2004Lifelong} and the sampling-based planners~\cite{Naderi2015rtrrtstar}~\cite{otte2016rrtx} have been significantly improved in recent years, the time-costly shortcut mechanism has to be preserved for the distance optimality. 

\subsection{Topological Path Planning}

Compared to the solutions to SPP, solutions to the topological path planning (TPP) tasks are far less mature. 
In a simply-connected environment, all closed paths (the path which starts and terminates at the same point) can continuously shrink to a single point, i.e., homotopic to the single-point path. 
As such it is proven that the locally shortest path must be also the globally shortest path, which however is not true in a multiply-connected environment. 
In other words, the non-trivialness of homotopy classes of paths is intrinsic to the multiply-connected environment.

Early reports on the TPP problem are heuristic. 
Constructing a graph with randomly sampled waypoints as vertices~\cite{Kavraki1996Probabilistic} and generating multiple paths~\cite{Schmitzberger2002Capture} may be valid strategies for obtaining non-homotopic resulting paths, but there is no guarantee for the pairwise non-homotopy of paths. 
The \textit{visibility condition}~\cite{Leonard2008Path} can be utilised for the removal of homotopic resulting paths~\cite{Zhou2020Robust}. 
However, it cannot instruct the planner to find non-homotopic paths. 
So there is still no guarantee for acquiring non-homotopic path results. 
Within a limited computational time, the algorithms might not find the desired number ($k$) of non-homotopic paths. 
In the worst case, all founded paths might be homotopic. 
So early sampling-based strategies~\cite{Zhou2020Robust} were essentially looking for multiple resulting paths, expecting that some of them are non-homotopic. 

It has been noted that path segments in the Voronoi Graph~\cite{Aurenhammer1991Voronoi}~\cite{Pettre2005Navigation}~\cite{Banerjee2006Framework}~\cite{Wang2019Optimal} are unique amidst two obstacles, then different paths found in the Voronoi Diagram guarantee non-homotopy. 
Hence Voronoi-based methods are valid algorithms for TPP. 
However, when an environment is transformed into a Voronoi graph, the $k$ shortest Voronoi paths are not pairwise homotopic to the $k$ shortest non-homotopic paths in the grid-map. 
So locally optimising the $k$-shortest Voronoi paths will not be the correct $k$-SNPP solution. 
This is why the topology-awareness of the Voronoi graph cannot be generalised to solve $k$-SNPP. 
Another disadvantage for Voronoi-based methods is that generating Voronoi Graph is a costly exercise, but the $k$-shortest non-homotopic paths need not be Voronoi. 
This makes generating Voronoi Graph for path planning to be a last-sorted approach for TPP. 

There indeed exist algorithms~\cite{Schmitzberger2002Capture}~\cite{Hernandez2011Path} that adopt topological invariants to distinguish different homotopy classes of paths. 
Typical topological invariants~\cite{Pokorny2016Topological} are the $H$-signature~\cite{Bhattacharya2012Topological} and the winding number~\cite{Pokorny2016High}. 
Such algorithms are essentially carrying out the path planning in not the 2D environment but its \textit{universal covering space} (UCS), or sometimes called \textit{homotopy-augmented graph}. 
Since distance-based optimal path searching in UCS is a valid strategy for solving $k$-SNPP, we discuss it in detail in the next subsection. 

\subsection{Topology-Aware Distance-Optimal Path Planning}

Two kinds of homotopy-aware shortest path planners have been proposed~\cite{Hernandez2015Comparison}: 
One is to find the shortest path in the given path homotopy~\cite{Hernandez2011A}, and the other is to look for $k$ shortest non-homotopic paths among multiple path homotopies~\cite{Bhattacharya2012Topological}. 

For algorithms belonging to the first category~\cite{Hershberger1994Computing}~\cite{Bespamyatnikh2003Computing}~\cite{Efrat2006Computing}, the input data is a non-optimal path or the character of a homotopy class, and the resultant path is the locally shortest path in the given topology. 
The requirement of input topology restricts the algorithm applicability, because there may be no initial path available. 
And such algorithms bypass the problem of selecting the optimal path homotopies among a mammoth set of path homotopies. 
So we omit to survey them further.   

A systematic solution to the $k$-SNPP problem has been proposed by path searching in the \textit{universal covering space}~\cite{Bhattacharya2012Topological} (UCS). 
The UCS of a 2D multiply-connected environment is a manifold such that, the image of non-homotopic paths in the original environment will have different endpoints in the UCS:  
The image of the starting point is still a single configuration in the UCS, but the goal point will have multiple (infinite) image configurations. 
Then the $k$-SNPP problem in the 2D environment is transformed into a multi-goal SPP problem in the UCS, where the desired $k$ paths are those towards the nearest $k$ images of the goal point. 
Although the idea is novel in its topological representation, the pathfinding strategies~\cite{Bhattacharya2012Topological}~\cite{Yi2016Homotopy} and the distance-optimal planners adopted are independent. 
For example, A* may be adopted, whose homotopy awareness is to augment a node from a 2D location to a combination of the 2D location and the topological invariants of its currently shortest paths (e.g., the $H$-signature). 
Only the path-level shortcut mechanism is adopted, i.e., an old path will be replaced only when a newly constructed path has the same $H$-signature and is shorter than it. 
As for the paths that visit the same location but have different $H$-signatures, the algorithms have to preserve them all for topology completeness. 
Similar algorithms have been also proposed for tethered robot applications~\cite{Wang2018Topological}~\cite{Kim2015Path}, whose underlying ideas are essentially equivalent as above. 
This is in contrast to the distance-based topology simplification mechanism proposed in this paper, where our main concentration has been paid to comparing the locally shortest paths belonging to different topologies. 

Please also note that the classic $k$ shortest path planning problem in graph theory is different from the problem discussed in this paper because in graph theory the environment has been modelled into the concatenation of edges, and different paths rendezvous at node (a single point), thus different paths must be non-homotopic paths, and the travelling distance comparison of non-homotopic paths can be easily obtained by the path-level shortcut mechanism. 
In robotics, when the environment is transformed to an abstract graph~\cite{Simeon2000Visibility}~\cite{Schmitzberger2002Capture}~\cite{Blochliger2018Topomap}, the optimality metric changes jointly. 
Short-cutting the $k$ shortest paths obtained from an abstract graph~\cite{Daneshpajouh2011Heuristic} cannot yield the desired $k$-SNPP solution. 
In contrast, the problem tackled in this work appears exactly because of the non-existence of a single point that all paths in two homotopy classes of paths will visit.

\section{Conclusion}\label{section_conclusion}
The main contribution of this work is a systematic mechanism for distance-based topology simplification, which reduces the algorithmic complexity of finding the $k$-shortest non-homotopic paths. 
In a 2D environment with $n$ internal obstacles, $2^n$ different non-self-crossing homotopy classes of paths can be characterised, wherein only $k$ ones are desired and all other $(2^n-k)$ ones are unnecessary. 
All existing $k$-SNPP algorithms are equivalent to an exhaustive exploration process in the configuration space which is mathematically the universal covering space of the 2D environment until $k$ resultant paths are collected. 
This is a time-consuming practice, and removing unnecessary path topologies naturally motivates a comparison between different homotopy classes of paths, where the non-$k$-optimal topologies should be discarded as early as possible whilst planning. 

The main difficulty in carrying out topology simplification is to compare the length of locally shortest paths before knowing their length because the paths have not been constructed at an intermediate state of the planning process. 
To formally solve the problem, we have introduced a novel representation of topologies in an intermediate state of the path planning process, \textit{distinguished homotopy}.
A goal location relaxation strategy has been proposed in Section~\ref{section:goal_relaxation} which we think is necessary for any possible distance-optimal topology simplification mechanism that might be proposed in the future. 
A hierarchical topological tree has been developed in Section~\ref{section_node} and Section~\ref{section_theoretical_analysis}. 
Finally, the distance-based topology simplification mechanism has been built upon the topological tree in Section~\ref{section_related_optimality}. 
A step-by-step illustration of the proposed algorithm has been illustrated in this paper. 
Extensive comparisons for the algorithmic efficiency in C++ have been carried out. 



\bibliographystyle{ieeetr} 
\bibliography{TRO21}

\newpage

\end{document}